\documentclass{article} 
    \PassOptionsToPackage{numbers, compress}{natbib}

    \usepackage[final]{neurips_2020}

\usepackage[utf8]{inputenc} 
\usepackage[T1]{fontenc}    
\usepackage{hyperref}       
\usepackage{url}            
\usepackage{booktabs}       
\usepackage{amsfonts}       
\usepackage{nicefrac}       
\usepackage{microtype}      

\usepackage{amssymb,amsmath,amsthm}
\usepackage{bbm}
\usepackage{xifthen}
\usepackage{todonotes}
\usepackage{xcolor}
\usepackage{caption}
\usepackage{subcaption}
\usepackage{macros} 
\usepackage{enumitem}
\usepackage{algorithm}
\usepackage{algorithmic}

\def\shownotes{1}  \ifnum\shownotes=1
\newcommand{\authnote}[2]{[ #2 --#1 ]}
\else
\newcommand{\authnote}[2]{}
\fi

\title{Robust Optimization for Fairness \\ with Noisy Protected Groups}

\newcommand*\samethanks[1][\value{footnote}]{\footnotemark[#1]}
\author{%
  Serena Wang
  \thanks{First two authors have equal contributions.} \\
  UC Berkeley\\
  Google Research \\
  \texttt{serenalwang@berkeley.edu} \\
  \And
  Wenshuo Guo 
  \samethanks  \\
  UC Berkeley\\
  \texttt{wsguo@berkeley.edu} \\
  \AND
  Harikrishna Narasimhan \\
  Google Research\\
  \texttt{hnarasimhan@google.com} \\
   \And
  Andrew Cotter \\
  Google Research\\
  \texttt{acotter@google.com} \\
   \AND
  Maya Gupta\\
  Google Research\\
  \texttt{mayagupta@google.com} \\
  \And
  Michael I. Jordan \\
  UC Berkeley\\
  \texttt{jordan@berkeley.edu} \\
}

\begin{document}

\maketitle

\begin{abstract}
Many existing fairness criteria for machine learning involve equalizing some metric across \textit{protected groups} such as race or gender. However, practitioners trying to audit or enforce such group-based criteria can easily face the problem of noisy or biased protected group information. First, we study the consequences of na{\"i}vely relying on noisy protected group labels: we provide an upper bound on the fairness violations on the true groups $G$ when the fairness criteria are satisfied on noisy groups $\hat{G}$. Second, we introduce two new approaches using robust optimization that, unlike the na{\"i}ve approach of only relying on $\hat{G}$, are guaranteed to satisfy fairness criteria on the true protected groups $G$ while minimizing a training objective. We provide theoretical guarantees that one such approach converges to an optimal feasible solution. Using two case studies, we show empirically that the robust approaches achieve better true group fairness guarantees than the na{\"i}ve approach. 

\end{abstract}

\section{Introduction}

As machine learning becomes increasingly pervasive in real-world decision making, the question of ensuring \textit{fairness} of ML models becomes increasingly important. The definition of what it means to be ``fair'' is highly context dependent. Much work has been done on developing mathematical fairness criteria according to various societal and ethical notions of fairness, as well as methods for building machine-learning models that satisfy those fairness criteria \citep[see, e.g.,][]{dwork:2012,Hardt:2016,Russell:2017, Kusner:2017,Zafar:2017,cotter2019optimization,Friedler:2019,WangGupta:2020}. 

Many of these mathematical fairness criteria are \textit{group-based}, where a target metric is equalized or enforced over subpopulations in the data, also known as \textit{protected groups}. For example, the \textit{equality of opportunity} criterion introduced by \citet{Hardt:2016} specifies that the true positive rates for a binary classifier are equalized across protected groups. The \emph{demographic parity}~\citep{dwork:2012} criterion requires that a classifier's positive prediction rates are equal for all protected groups. 

One important practical question is whether or not these fairness notions can be reliably measured or enforced if the protected group information is noisy, missing, or unreliable. For example, survey participants may be incentivized to obfuscate their responses for fear of disclosure or discrimination, or may be subject to other forms of response bias. Social desirability response bias may affect participants' answers regarding religion, political affiliation, or sexual orientation~\citep{Krumpal:2011}. The collected data may also be outdated: census data collected ten years ago may not an accurate representation for measuring fairness today. 

Another source of noise arises from estimating the labels of the protected groups.  For various image recognition tasks (e.g., face detection), one may want to measure fairness across protected groups such as gender or race. However, many large image corpora do not include protected group labels, and one might instead use a separately trained classifier to estimate group labels, which is likely to be noisy \cite{Joy:2018}. Similarly, zip codes can act as a noisy indicator for socioeconomic groups. 

In this paper, we focus on the problem of training binary classifiers with fairness constraints when only noisy labels, $\hat{G} \in \{1,...,\hat{m}\}$, are available for $m$ true protected groups, $G \in \{1,...,m\}$, of interest. We study two aspects: First, if one satisfies fairness constraints for noisy protected groups $\hat{G}$, what can one say with respect to those fairness constraints for the true groups $G$? Second, how can side information about the noise model between $\hat{G}$ and $G$ be leveraged to better enforce fairness with respect to the true groups $G$?

\textbf{Contributions:} Our contributions are three-fold:
\begin{enumerate}[noitemsep, topsep=0em]
  \item We provide a bound on the fairness violations with respect to the true groups $G$ when the fairness criteria are satisfied for the noisy groups $\hat{G}$. 
  \item We introduce two new robust-optimization methodologies that satisfy fairness criteria on the true protected groups $G$ while minimizing a training objective. These methodologies differ in convergence properties, conservatism, and noise model specification. 
  %
  \item We show empirically that unlike the na{\"i}ve approach, our two proposed approaches are able to satisfy fairness criteria with respect to the true groups $G$ on average. 
\end{enumerate}

The first approach we propose (Section \ref{sec:dro}) is based on distributionally robust optimization (DRO) \citep{Duchi:2018, ben2013robust}. Let $p$ denotes the full distribution of the data $X,Y \sim p$. Let $p_j$ be the distribution of the data conditioned on the true groups being $j$, so $X, Y | G = j \sim p_j$; and $\hat{p}_j$ be the distribution of $X, Y$ conditioned on the noisy groups. Given an upper bound on the total variation (TV) distance $\gamma_j \geq TV(p_j, \hat{p}_j)$ for each $j \in \{1,...,m\}$, we define $\tilde{p}_j$ such that the conditional distributions ($X,Y | \tilde{G}=j \sim \tilde{p}_j$) fall within the bounds $\gamma_i$ with respect to $\hat{G}$. Therefore, the set of all such $\tilde{p}_j$ is guaranteed to include the unknown true group distribution $p_j, \forall j\in\mathcal{G}$. Because it is based on the well-studied DRO setting, this approach has the advantage of being easy to analyze. However, the results may be overly conservative unless tight bounds $\{\gamma_j\}_{j=1}^m$ can be given.

 
Our second robust optimization strategy (Section \ref{sec:softweights}) uses a robust re-weighting of the data from soft protected group assignments, inspired by criteria proposed by \citet{Kallus:2020} for auditing the fairness of ML models given imperfect group information. Extending their work, we \emph{optimize} a constrained problem to achieve their robust fairness criteria, and provide a theoretically ideal algorithm that is guaranteed to converge to an optimal feasible point, as well as an alternative practical version that is more computationally tractable. 
Compared to DRO, this second approach uses a more precise noise model, $P(\hat{G} = k | G = j)$, between $\hat{G}$ and $G$ for all pairs of group labels $j,k$, that can be estimated from a small auxiliary dataset containing ground-truth labels for both $G$ and $\hat{G}$. An advantage of this more detailed noise model is that a practitioner can incorporate knowledge of any bias in the relationship between $G$ and $\hat{G}$ (for instance, survey respondents favoring one socially preferable response over others), which causes it to be less likely than DRO to result in an overly-conservative model. Notably, this approach does \emph{not} require that $\hat{G}$ be a direct approximation of $G$---in fact, $G$ and $\hat{G}$ can represent distinct (but related) groupings, or even groupings of different sizes, with the noise model tying them together. For example, if $G$ represents ``language spoken at home,'' then $\hat{G}$ could be a noisy estimate of ``country of residence.''

\section{Related work}


\textbf{Constrained optimization for group-based fairness metrics:}
The simplest techniques for enforcing group-based constraints apply a post-hoc correction of an existing classifier \cite{Hardt:2016,Woodworth:2017}. For example, one can enforce \textit{equality of opportunity} by choosing different decision thresholds for an existing binary classifier for each protected group~\citep{Hardt:2016}. However, the classifiers resulting from these post-processing techniques may not necessarily be optimal in terms of accuracy. Thus, constrained optimization techniques have emerged to train machine-learning models that can more optimally satisfy the fairness constraints while minimizing a training objective~\citep{Goh:2016,Cotter:ALT, cotter2019optimization, Zafar:2017, Agarwal:2018, Donini:2018, Narasimhan2019:3players}. 


\textbf{Fairness with noisy protected groups:}
Group-based fairness notions rely on the knowledge of \textit{protected group} labels. However, practitioners may only have access to noisy or unreliable protected group information.  One may na{\"i}vely try to enforce fairness constraints with respect to these noisy protected groups using the above constrained optimization techniques, but there is no guarantee that the resulting classifier will satisfy the fairness criteria with respect to the true protected groups \citep{gupta2018proxy}.

Under the conservative assumption that a practitioner has no information about the protected groups, \citet{Hashimoto:2018} applied DRO to enforce what \citet{lahoti2020fairness} refer to as \textit{Rawlsian Max-Min fairness}. In contrast, here we assume some knowledge of a noise model for the noisy protected groups, and are thus able to provide tighter results with DRO: we provide a practically meaningful maximum total variation distance bound to enforce in the DRO procedure.
We further extend \citet{Hashimoto:2018}'s work by applying DRO to problems equalizing fairness metrics over groups, which may be desired in some practical applications~\citep{Kolodny:2019}. 

Concurrently, \citet{lahoti2020fairness} proposed an adversarial reweighting approach to improve group fairness by assuming that non-protected features and task labels are correlated with unobserved groups. Like \citet{Hashimoto:2018}, \citet{lahoti2020fairness} also enforce \textit{Rawlsian Max-Min fairness} with unknown protected groups, whereas our setup includes constraints for parity based fairness notions.

\citet{Kallus:2020} considered the problem of \textit{auditing} fairness criteria given noisy groups. They propose a ``robust'' fairness criteria using soft group assignments and show that if a given model satisfies those fairness criteria with respect to the noisy groups, then the model will satisfy the fairness criteria with respect to the true groups.
Here, we build on that work by providing an algorithm for training a model that satisfies their robust fairness criteria while minimizing a training objective. 


\citet{Lamy:2019} showed that when there are only two protected groups, one need only tighten the ``unfairness tolerance'' when enforcing fairness with respect to the noisy groups. \citet{mozannar2020fair} showed that if the predictor is independent of the protected attribute, then fairness with respect to the noisy groups is the same as fairness with respect to the true groups. When there are more than two groups, and when the noisy groups are included as an input to the classifier, other robust optimization approaches may be necessary. 
When using post-processing instead of constrained optimization, \citet{Awasthi:2020} showed that under certain conditional independence assumptions, post-processing using the noisy groups will not be worse in terms of fairness violations than not post-processing at all. In our work, we consider the problem of training the model subject to fairness constraints, rather than taking a trained model as given and only allowing post-processing, and we do not rely on conditional independence assumptions. Indeed, the model may include the noisy protected attribute as a feature.

\textbf{Robust optimization:} We use a minimax set-up of a two-player game where the uncertainty is adversarial, and one minimizes a worst-case objective over a feasible set~\citep{BEN:09, bertsimas2011theory}; e.g., the noise is contained in a unit-norm ball around the input data. As one such approach, we apply a recent line of work on DRO which assumes that the uncertain distributions of the data are constrained to belong to a certain set~\citep{Namkoong:2016, Duchi:2018, Li:2019}. 


\section{Optimization problem setup}
We begin with the training problem for incorporating group-based fairness criteria in a learning setting \citep{Goh:2016,Hardt:2016,Donini:2018,Agarwal:2018,cotter2019optimization}. Let $X \in \mathcal{X} \subseteq \mathbb{R}^D$ be a random variable representing a feature vector, with a random binary label $Y \in \mathcal{Y} = \{0,1\}$ and random protected group membership $G \in \mathcal{G} = \{1,...,m\}$. In addition, let $\hat{G} \in \hat{\mathcal{G}} = \{1,...,\hat{m}\}$ be a random variable representing the noisy protected group label for each $(X, Y)$, which we assume we have access to during training. For simplicity, assume that $\hat{\mathcal{G}} = \mathcal{G}$ (and $\hat{m} = m$). Let $\phi(X;\theta)$ represent a binary classifier with parameters $\theta \in \Theta$ where $\phi(X;\theta) > 0$ indicates a positive classification.

Then, training with fairness constraints \citep{Goh:2016,Hardt:2016,Donini:2018,Agarwal:2018,cotter2019optimization} is:
\begin{equation} \label{eq:orig_short}
\min_{\theta} \quad f(\theta)\quad \textrm{ s.t. } \quad g_j(\theta) \leq 0, \forall j\in\mathcal{G},
\end{equation}
The objective function $f(\theta) = \E[l(\theta, X, Y)]$, where $l(\theta, X, Y)$ is any standard binary classifier training loss. The constraint functions $g_j(\theta) = \E[h(\theta, X, Y) | G = j]$ for $j \in \mathcal{G}$, where $h(\theta, X, Y)$
is the target fairness metric, e.g. $h(\theta,X,Y) = \Ind\big(\phi(X;\theta) > 0\big) - \E[\Ind\big(\phi(X;\theta) > 0\big)]$ when equalizing positive rates for the \textit{demographic parity} \citep{dwork:2012} criterion (see \cite{cotter2019optimization} for more examples). Algorithms have been studied for problem (\ref{eq:orig_short}) when the true protected group labels $G$ are given~\citep[see, e.g.,][]{Eban:2017,Agarwal:2018,cotter2019optimization}.


\section{Bounds for the na{\"i}ve approach}\label{sec:naive}
When only given the noisy groups $\hat{G}$, one na{\"i}ve approach to solving problem (\ref{eq:orig_short}) is to simply re-define the constraints using the noisy groups~\citep{gupta2018proxy}: 
\begin{align}
\min_{\theta} \quad f(\theta) \quad
\textrm{s.t. } \quad  \hat{g}_j(\theta) \leq 0, \;\; \forall j \in \mathcal{G},
\label{eq:naiveproxy}
\end{align}
where $\hat{g}_j(\theta) = \E[h(\theta, X, Y) | \hat{G} = j], \;\;j \in \mathcal{G} $. 

This introduces a practical question: if a model was constrained to satisfy fairness criteria on the noisy groups, how far would that model be from satisfying the constraints on the true groups? We show that the fairness violations on the true groups $G$ can at least be bounded when the fairness criteria are satisfied on the noisy groups $\hat{G}$, provided that $\hat{G}$ does not deviate too much from $G$.


\subsection{Bounding fairness constraints using TV distance}
\label{sec:bound_TV}
Recall that $X,Y | G = j \sim p_j$ and $X,Y | \hat{G} = j \sim \hat{p}_j$. We use the TV distance $TV(p_j, \hat{p}_j)$ to measure the distance between the probability distributions $p_j$ and $\hat{p}_j$ (see Appendix \ref{app:proofs-TV} and \citet{optimal_transport}). Given a bound on $TV(p_j, \hat{p}_j)$, we obtain a bound on fairness violations for the true groups when na{\"i}vely solving the optimization problem (\ref{eq:naiveproxy}) using only the noisy groups:

\begin{theorem}\label{thm:general} (proof in Appendix \ref{app:proofs-TV}.) Suppose a model with parameters $\theta$ satisfies fairness criteria with respect to the noisy groups $\hat{G}$:
$\hat{g}_j(\theta) \leq 0, \;\;\forall j \in \mathcal{G}. $
Suppose $|h(\theta,x_1,y_1) - h(\theta, x_2, y_2)| \leq 1$ for any $(x_1, y_1) \neq (x_2, y_2)$. If $TV(p_j, \hat{p}_j) \leq \gamma_j$ for all $j \in \mathcal{G}$, then the fairness criteria with respect to the true groups $G$ will be satisfied within slacks $\gamma_j$ for each group:
$g_j(\theta) \leq \gamma_j, \;\; \forall j \in \mathcal{G}.$
\end{theorem}

Theorem \ref{thm:general} is tight for the family of functions $h$ that satisfy $|h(\theta,x_1,y_1) - h(\theta, x_2, y_2)| \leq 1$ for any $(x_1, y_1) \neq (x_2, y_2)$. This condition holds for any fairness metrics based on rates such as demographic parity, where $h$ is simply some scaled combination of indicator functions. \citet{cotter2019optimization} list many such rate-based fairness metrics. Theorem \ref{thm:general} can be generalized to functions $h$ whose differences are not bounded by 1 by looking beyond the TV distance to more general Wasserstein distances between $p_j$ and $\hat{p}_j$. We show this in Appendix \ref{app:wass}, but for all fairness metrics referenced in this work, formulating Theorem \ref{thm:general} with the TV distance is sufficient. 

\subsection{Estimating the TV distance bound in practice}

Theorem \ref{thm:general} bounds the fairness violations of the na{\"i}ve approach in terms of the TV distance between the conditional distributions $p_j$ and $\hat{p}_j$, which assumes knowledge of $p_j$ and is not always possible to estimate. Instead, we can estimate an upper bound on $TV(p_j, \hat{p}_j)$ from metrics that are easier to obtain in practice. Specifically, the following lemma shows that shows that if the prior on class $j$ is unaffected by the noise, $P(G \neq \hat{G} | G = j)$ directly translates into an upper bound on $TV(p_j, \hat{p}_j)$. 

\begin{lemma}\label{lem:tv_bound} (proof in Appendix \ref{app:proofs-TV}.) Suppose $P(G = j) = P(\hat{G} = j)$ for a given $j \in \mathcal{G}$. Then $TV(p_j, \hat{p}_j) \leq P(G \neq \hat{G} | G = j) $.
\end{lemma}

In practice, an estimate of $P(G \neq \hat{G} | G = j)$ may come from a variety of sources. As assumed by \citet{Kallus:2020}, a practitioner may have access to an \textit{auxiliary} dataset containing $G$ and $\hat{G}$, but not $X$ or $Y$. Or, practitioners may have some prior estimate of $P(G \neq \hat{G} | G = j)$: if $\hat{G}$ is estimated by mapping zip codes to the most common socioeconomic group for that zip code, then census data provides a prior for how often $\hat{G}$ produces an incorrect socioeconomic group. 

By relating Theorem \ref{thm:general} to realistic noise models, Lemma \ref{lem:tv_bound} allows us to bound the fairness violations of the na{\"i}ve approach using quantities that can be estimated empirically. In the next section we show that Lemma \ref{lem:tv_bound} can also be used to produce a \textit{robust} approach that will actually guarantee full satisfaction of the fairness violations on the true groups $G$.

\section{Robust Approach 1: Distributionally robust optimization (DRO)}\label{sec:dro}

While Theorem \ref{thm:general} provides an upper bound on the performance of the na{\"i}ve approach, it fails to provide a guarantee that the constraints on the true groups are satisfied, i.e. $g_j(\theta) \leq 0$. Thus, it is important to find other ways to do better than the na{\"i}ve optimization problem (\ref{eq:naiveproxy}) in terms of satisfying the constraints on the true groups. In particular, suppose in practice we are able to assert that $P(G \neq \hat{G} | G = j) \leq \gamma_j$ for all groups $j \in \mathcal{G}$. Then Lemma \ref{lem:tv_bound} implies a bound on TV distance between the conditional distributions on the true groups and the noisy groups: $TV(p_j, \hat{p}_j) \leq \gamma_j$. 
Therefore, any feasible solution to the following constrained optimization problem is guaranteed to satisfy the fairness constraints on the true groups:
\begin{align}
\min_{\theta \in \Theta} \quad f(\theta) \quad
\text{s.t.} \quad \max_{\substack{\tilde{p}_j: TV(\tilde{p}_j, \hat{p}_j) \leq \gamma_j \\ \tilde{p}_j \ll p}}  \tilde{g}_j(\theta) \leq 0, \;\; \forall j\in \mathcal{G},
\label{eq:dro}
\end{align}
where $\tilde{g}_j(\theta)= \E_{X,Y \sim \tilde{p}_j}[h(\theta, X, Y)]$, and $\tilde{p}_j \ll p$ denotes absolute continuity.

\subsection{General DRO formulation}
A DRO problem is a minimax optimization~\citep{Duchi:2018}: 
\begin{equation}\label{eq:general_dro}
    \min_{\theta \in \Theta} \max_{q: D(q,p) \leq \gamma} \; \E_{X, Y \sim q}[l(\theta, X,Y)],
\end{equation}
where $D$ is some divergence metric between the distributions $p$ and $q$, and $l:\Theta \times \mathcal{X} \times \mathcal{Y} \to \mathbb{R}$. Much existing work on DRO focuses on how to solve the DRO problem for different divergence metrics $D$. \citet{Namkoong:2016} provide methods for efficiently and optimally solving the DRO problem for $f$-divergences, and other work has provided methods for solving the DRO problem for Wasserstein distances~\citep{Li:2019, Esfahani:2018}. \citet{Duchi:2018} further provide finite-sample convergence rates for the empirical version of the DRO problem.

\subsection{Solving the DRO problem}
An important and often difficult aspect of using DRO is specifying a divergence $D$ and bound $\gamma$ that are meaningful. In this case, Lemma \ref{lem:tv_bound} gives us the key to formulating a DRO problem that is guaranteed to satisfy the fairness criteria with respect to the true groups $G$. 

The optimization problem (\ref{eq:dro}) can be written in the form of a DRO problem (\ref{eq:general_dro}) with TV distance by using the Lagrangian formulation. Adapting a simplified version of a gradient-based algorithm provided by \citet{Namkoong:2016}, we are able to solve the empirical formulation of problem (\ref{eq:general_dro}) efficiently. Details of our empirical Lagrangian formulation and pseudocode are in Appendix \ref{app:dro}.

\section{Robust Approach 2: Soft group assignments}\label{sec:softweights}

While any feasible solution to the distributionally robust constrained optimization problem (\ref{eq:dro}) is guaranteed to satisfy the constraints on the true groups $G$, 
choosing each $\gamma_j = P(G \neq \hat{G} | G = j)$ as an upper bound on $TV(p_j, \hat{p}_j)$ may be rather conservative. Therefore, as an alternative to the DRO constraints in (\ref{eq:dro}), in this section we show how to optimize using the robust fairness criteria proposed by \citet{Kallus:2020}. 

\subsection{Constraints with soft group assignments}

Given a trained binary predictor, $\hat{Y}(\theta) = \Ind(\phi(\theta ;X) > 0)$, \citet{Kallus:2020} proposed a set of robust fairness criteria that can be used to audit the fairness of the given trained model with respect to the true groups $G \in \mathcal{G}$ using the noisy groups $\hat{G} \in \hat{\mathcal{G}}$, where $\mathcal{G} = \hat{\mathcal{G}}$ is not required in general.
They assume access to a \textit{main dataset} with the noisy groups $\hat{G}$, true labels $Y$, and features $X$, as well an \textit{auxiliary dataset} containing both the noisy groups $\hat{G}$ and the true groups $G$. From the main dataset, one obtains estimates of the joint distributions $(\hat{Y}(\theta), Y, \hat{G})$; from the auxiliary dataset, one obtains estimates of the joint distributions $(\hat{G}, G)$ and a noise model $P(G = j | \hat{G} = k)$ $ \forall j \in \mathcal{G}, k \in \hat{\mathcal{G}}$. 


These estimates are used to associate each example with a vector of weights, where each weight is an estimated probability that the example belongs to the true group $j$. Specifically, suppose that we have a function $w: \mathcal{G} \times \{0,1\} \times \{ 0,1\} \times \hat{\mathcal{G}} \to [0,1]$, where $w(j \mid \hat{y}, y, k)$ estimates $P(G = j| \hat{Y}(\theta) = \hat{y}, Y = y, \hat{G} = k)$. We rewrite the fairness constraint $E[h(\theta, X, Y) | G = j] = \frac{E[h(\theta, X, Y) P( G = j| \hat{Y}(\theta), Y, \hat{G})]}{P(G=j)}$ (derivation in Appendix \ref{app:tower_property}), and estimate this using $w$. We also show how $h$ can be adapted to the \textit{equality of opportunity} setting in Appendix \ref{app:tpr_sa}.

Given the main dataset and auxiliary dataset, we limit the possible values of the function $w(j \mid \hat{y}, y, k)$ using the law of total probability (as in \cite{Kallus:2020}).
The set of possible functions $w$ is given by:
\begin{align}
\begin{split}
    \mathcal{W}(\theta) = 
    \left\{w: \substack{ \sum_{\hat{y}, y \in \{0,1\}} w(j \mid \hat{y}, y, k) P(\hat{Y}(\theta) = \hat{y}, Y = y| \hat{G} = k) = P(G = j | \hat{G} = k), \\ 
    \sum_{j=1}^m w(j \mid \hat{y}, y, k) = 1, w(j \mid \hat{y}, y, k) \geq 0 \quad \forall \hat{y}, y \in \{0,1\}, j \in \mathcal{G}, k \in \hat{\mathcal{G}}} \right\}. 
\end{split}\label{eq:W}
\end{align}

The robust fairness criteria can now be written in terms of $\mathcal{W}(\theta)$ as:
\begin{equation}\label{eq:w_constraint}
    \max_{w \in \mathcal{W}(\theta)} g_j(\theta, w) \leq 0, \;\; \forall j \in \mathcal{G} \quad \text{ where }\quad g_j(\theta, w) = \frac{\E[h(\theta, X, Y) w(j |\hat{Y}(\theta), Y, \hat{G})]}{P(G=j)}.
\end{equation}

\subsection{Robust optimization with soft group assignments}
We extend \citet{Kallus:2020}'s work by formulating a robust optimization problem using soft group assignments. Combining the robust fairness criteria above with the training objective, we propose:
\begin{align}
\min_{\theta \in \Theta} \quad f(\theta) \quad 
\textrm{s.t.} \max_{w \in \mathcal{W}(\theta)} g_j(\theta, w)  \leq 0, \;\; \forall j \in \mathcal{G},
\label{eq:softweights}
\end{align}
where $\Theta$ denotes the space of model parameters.
Any feasible solution is guaranteed to satisfy the original fairness criteria with respect to the true groups. Using a Lagrangian, problem (\ref{eq:softweights}) can be rewritten as:
%
\begin{equation}
  \label{eq:softweights_lagrangian}
  \min_{\theta\in\Theta} \max_{\lambda \in \Lambda} \cL(\theta, \lambda)
\end{equation}
where the Lagrangian $\cL(\theta, \lambda) = f(\theta) + \sum_{j=1}^m \lambda_j \max_{w \in \mathcal{W}(\theta)} g_j(\theta, w)$, and $\Lambda \subseteq \R_{+}^{m}$.
 
When solving this optimization problem, we use the empirical finite-sample versions of each expectation. As described in Proposition 9 of \citet{Kallus:2020}, the inner maximization (\ref{eq:w_constraint}) over $w \in \mathcal{W}(\theta)$ can be solved as a linear program for a given fixed $\theta$. However, the Lagrangian problem (\ref{eq:softweights_lagrangian}) is not as straightforward to optimize, since the feasible set $\mathcal{W}(\theta)$ depends on $\theta$ through $\hat{Y}$. While in general the pointwise maximum of convex functions is convex, the dependence of $\mathcal{W}(\theta)$ on $\theta$ means that even if $g_j(\theta, w)$ were convex, $\max_{w \in \mathcal{W}(\theta)} g_j(\theta, w)$ is not necessarily convex. We first introduce a theoretically ideal algorithm that we prove converges to an optimal, feasible solution. This ideal algorithm relies on a 
minimization oracle, which is not always computationally tractable. Therefore, we further provide a practical algorithm using gradient methods that mimics the ideal algorithm in structure and computationally tractable, but does not share the same convergence guarantees. 

\subsection{Ideal algorithm}
The minimax problem in \eqref{eq:softweights_lagrangian}
can be interpreted as a zero-sum game between the $\theta$-player and $\lambda$-player. 
In Algorithm \ref{algo:ideal}, we provide an iterative procedure for solving \eqref{eq:softweights_lagrangian}, where at each step, the $\theta$-player performs a full optimization, i.e., a \textit{best response} over $\theta$, and the $\lambda$-player responds with a gradient ascent update on $\lambda$. 

For a fixed $\theta$, the gradient of the Lagrangian $\cL$ with respect to $\lambda$  is given by
$
\partial \cL(\theta, \lambda) / \partial \lambda_j \,=\, \max_{w \in \cW(\theta)} g_j(\theta, w)
$,
which is a linear program in $w$. 
The challenging part, however, is the best response over $\theta$; that is, finding a solution $\min_{\theta}\,\cL(\theta, \lambda)$ for a given $\lambda$, as this involves a max over constraints $\cW(\theta)$ which  depend  on $\theta$. To implement this best response, we formulate a nested minimax problem that decouples this intricate dependence on $\theta$, by introducing Lagrange multipliers for the constraints in $\cW(\theta)$. We then solve this problem with an  oracle that jointly minimizes over both $\theta$ and the newly introduced Lagrange multipliers (details in Algorithm \ref{algo:best-response-theta} in Appendix \ref{app:ideal-alg-opt-feas}). 

The output of the best-response step is a stochastic classifier  with a distribution $\hat{\theta}^{(t)}$ over a finite set of $\theta$s. Algorithm \ref{algo:ideal} then returns the average of these distributions, $\bar{\theta} = \frac{1}{T} \sum_{t=1}^{T} \hat{\theta}^{t}$, over $T$ iterations.  By extending recent results on constrained optimization~\citep{Cotter:ALT}, we show in Appendix \ref{app:ideal-alg-opt-feas}  that the  output $\bar{\theta}$ is near-optimal and near-feasible for the  robust optimization problem in \eqref{eq:softweights}. 
That is, for a given $\epsilon > 0$, by picking $T$ to be large enough, we have that the objective $\E_{\theta \sim \bar{\theta}}\left[f(\theta)\right] \leq f(\theta^{*}) + \epsilon,$ for any $\theta^{*}$ that is feasible, and the expected violations in the robust  constraints are also no more than $\epsilon$. 

\begin{figure}[!ht]
\vspace{-10pt}
\begin{algorithm}[H]
\caption{\textit{Ideal} Algorithm}
\label{algo:ideal}
\begin{algorithmic}[1]
\REQUIRE{learning rate $\eta_\lambda > 0$, estimates of $P( G = j | \hat{G} = k)$ to specify $\mathcal{W}(\theta)$, $\rho$, $\rho'$}
\FOR{$t = 1, \ldots, T$}
\STATE \textit{Best response on $\theta$}: run the oracle-based Algorithm \ref{algo:best-response-theta}
to find a distribution $\hat{\theta}^{(t)}$ over $\Theta$
s.t.
       $\E_{\theta \sim \hat{\theta}^{(t)}}\left[\cL(\theta, \lambda^{(t)})\right] \,\leq\, \min_{\theta \in \Theta}\, \cL(\theta, \lambda^{(t)}) \,+\, \rho$.
\STATE \textit{Estimate gradient $\nabla_\lambda \cL(\hat{\theta}^{(t)}, \lambda^{(t)})$}: for each $j \in \mathcal{G}$, choose $\delta^{(t)}_j$ s.t.

$\delta^{(t)}_j \,\leq\, \E_{\theta \sim \hat{\theta}^{(t)}}\left[\max_{w \in \cW(\theta)} g_j(\theta, w)\right] \,\leq\, \delta^{(t)}_j \,+\, \rho'$
    \STATE \textit{Ascent step on $\lambda$}: $\tilde{\lambda}_j^{(t+1)}  \gets \lambda_j^{(t)} + \eta_\lambda\,\delta^{(t)}_j, \;\; \forall j \in \mathcal{G}; \quad \lambda^{(t+1)}  \gets \Pi_{\Lambda}(\tilde{\lambda}^{(t+1)})$\\
\ENDFOR
\STATE \textbf{return}{ $\bar{\theta} = \frac{1}{T} \sum_{t=1}^{T} \hat{\theta}^{(t)}$}
\end{algorithmic}
\end{algorithm}
\vspace{-15pt}
\end{figure}

\subsection{Practical algorithm}
Algorithm \ref{algo:ideal} is guaranteed to converge to a near-optimal, near-feasible solution, but 
may be computationally intractable and impractical  for the following reasons. First, the algorithm needs a nonconvex minimization oracle to compute a best response over $\theta$. Second, there are multiple levels of nesting, making it difficult to scale the algorithm with mini-batch or stochastic updates. Third,  the output is 
a distribution over multiple models, which can be be difficult to use in practice \cite{Narasimhan:2019}.
%

Therefore, we supplement Algorithm \ref{algo:ideal} with a practical algorithm, Algorithm \ref{algo:heuristic} (see Appendix \ref{app:practical}) that is similar in structure, but
approximates the inner best response routine with two simple steps: a maximization over $w \in \cW(\theta^{(t)})$ using a linear program for the current iterate $\theta^{(t)}$, and a gradient step on $\theta$ at the maximizer $w^{(t)}$. 
Algorithm \ref{algo:heuristic} leaves room for other practical modifications such as using stochastic gradients. We provide further discussion in Appendix \ref{app:practical}.

\section{Experiments}\label{sec:experiments}
We compare the performance of the na{\"i}ve approach and the two robust optimization approaches (DRO and soft group assignments) empirically using two datasets from UCI \cite{Dua:2019} with different constraints. For both datasets, we stress-test the performance of the different algorithms under different amounts of noise between the true groups $G$ and the noisy groups $\hat{G}$. We take $l$ to be the hinge loss. The specific constraint violations measured and additional training details can be found in Appendix \ref{app:experiment_details}. All experiment code is available on GitHub at \url{https://github.com/wenshuoguo/robust-fairness-code}.

\textbf{Generating noisy protected groups:}
Given the true protected groups, we synthetically generate noisy protected groups by selecting a fraction $\gamma$ of data uniformly at random. For each selected example, we perturb the group membership to a different group also selected uniformly at random from the remaining groups. 
This way, for a given $\gamma$, $P(\hat{G} \neq G) \approx P(\hat{G} \neq G | G =j) \approx \gamma$ for all groups $j,k \in \mathcal{G}$. We evaluate the performance of the different algorithms ranging from small to large amounts of noise: $\gamma \in \{0.1, 0.2, 0.3, 0.4, 0.5\}$.

\subsection{Case study 1 (Adult): equality of opportunity}
We use the Adult dataset from UCI~\citep{Dua:2019} collected from 1994 US Census, which has 48,842 examples and 14 features (details in Appendix \ref{app:experiments}). The classification task is to determine whether an individual makes over \$50K per year. For the true groups, we use $m=3$ race groups of ``white,'' ``black,'' and ``other.'' 
As done by \cite{cotter2019optimization,Friedler:2019,Zafar:2019}, we enforce \textit{equality of opportunity} by equalizing true positive rates (TPRs). Specifically, we enforce that the TPR conditioned on each group is greater than or equal to the overall TPR on the full dataset with some slack $\alpha$, which produces $m$ true group fairness criteria, $\{g_j^{\text{TPR}}(\theta) \leq 0 \} \;\; \forall j \in \mathcal{G}$ (details on the constraint function $h$ in Appendix \ref{app:tpr_dro} and \ref{app:tpr_sa}).





\subsection{Case study 2 (Credit): equalized odds}



We consider another application of group-based fairness constraints to credit default prediction.
\citet{fourcade2013classification} provide an in depth study of the effect of credit scoring techniques on the credit market, showing that this scoring system can perpetuate inequity.
Enforcing group-based fairness with credit default predictions has been considered in a variety of prior works \citep{Hardt:2016, berk2017convex, WangGupta:2020, aghaei2019learning, bera2019fair, grariachieving, Friedler:2019,barocas-hardt-narayanan}. Following \citet{Hardt:2016} and \citet{grariachieving}, we enforce \textit{equalized odds} \citep{Hardt:2016} by equalizing both true positive rates (TPRs) and false positive rates (FPRs) across groups. 

We use the ``default of credit card clients'' dataset from UCI~\cite{Dua:2019} collected by a company in Taiwan~\cite{Yeh:2009}, which contains 30,000 examples and 24 features (details in Appendix \ref{app:experiments}). The classification task is to determine whether an individual defaulted on a loan. 
We use $m=3$ groups based on education levels: ``graduate school,'' ``university,'' and ``high school/other'' (the use of education in credit lending has previously been studied in the algorithmic fairness and economics literature~\citep{gillis2020false, bera2019fair, lazar2020resolution}). We constrain the TPR conditioned on each group to be greater than or equal to the overall TPR on the full dataset with a slack $\alpha$, and the FPR conditioned on each group to be less than or equal to the overall FPR on the full dataset. This produces $2m$ true group-fairness criteria, $\{g_j^{\text{TPR}}(\theta) \leq 0, g_j^{\text{FPR}}(\theta) \leq 0 \} \;\; \forall j \in \mathcal{G}$ (details on constraint functions $h$ in Appendix \ref{app:tpr_dro} and \ref{app:tpr_sa}).

\begin{figure}[!ht]
\begin{center}
\centerline{\begin{tabular}{ccc}
 \includegraphics[width=0.31\textwidth]{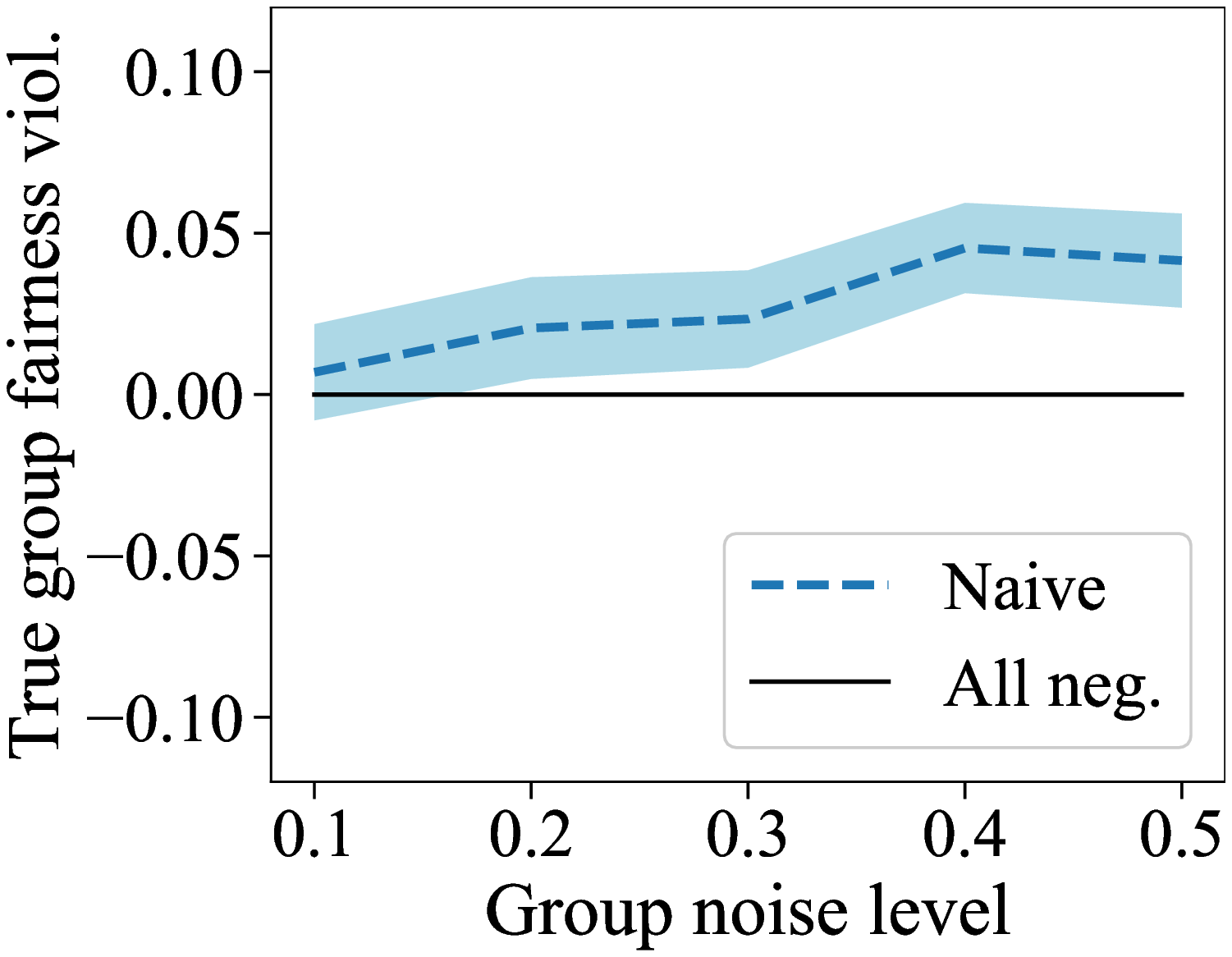} &
  \includegraphics[width=0.31\textwidth]{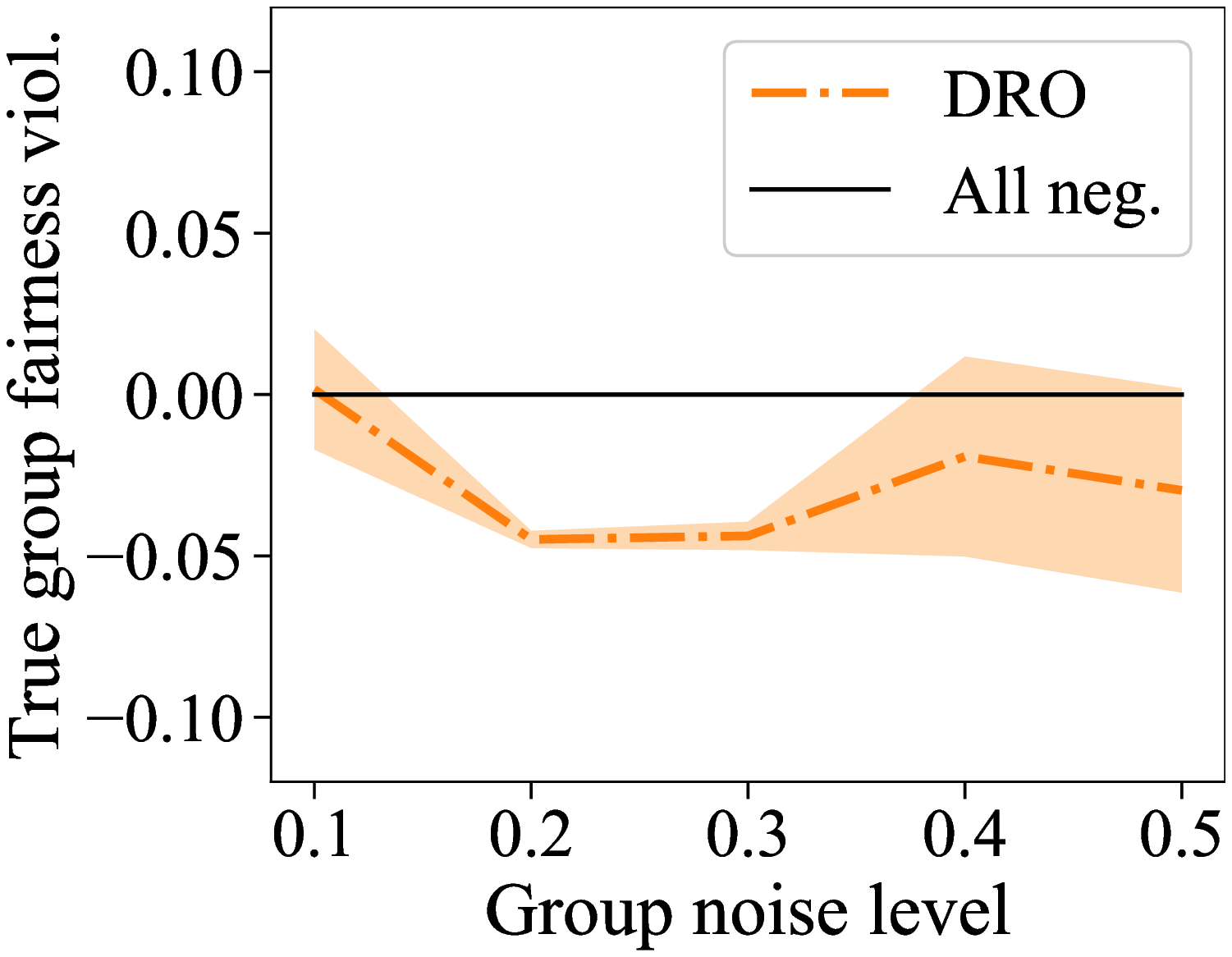} & \includegraphics[width=0.31\textwidth]{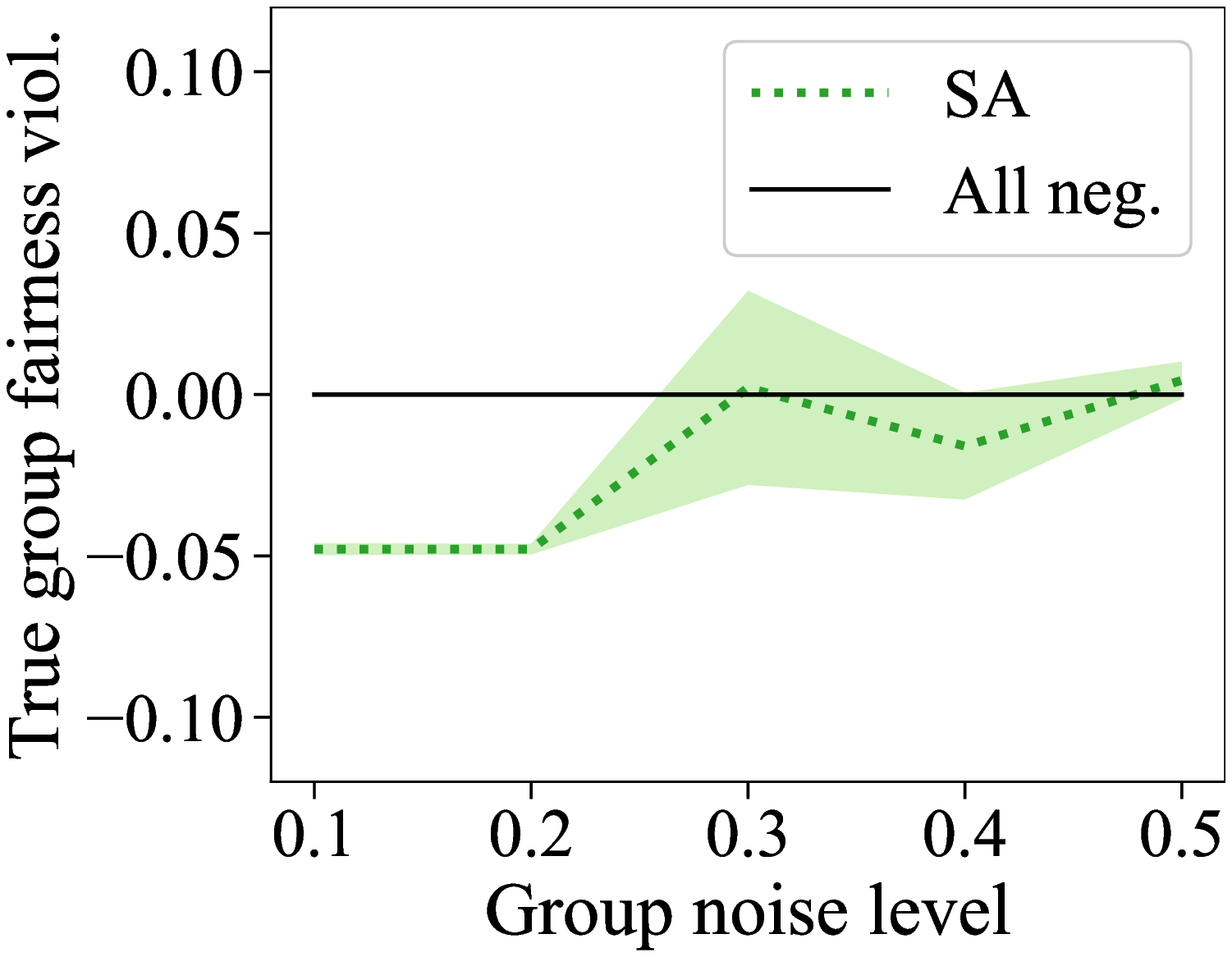} \\
\end{tabular}}
\caption{Case study 1 (Adult): maximum true group constraint violations on test set for the Naive, DRO, and soft assignments (SA) approaches for different group noise levels $\gamma$ on the Adult dataset (mean and standard error over 10 train/val/test splits). The black solid line represents the performance of the trivial ``all negatives'' classifier, which has constraint violations of 0. A negative violation indicates satisfaction of the fairness constraints on the true groups.}
\label{fig:adult_constraints} 
\end{center}
\vskip -0.2in
\end{figure}


\begin{figure}[!ht]
\begin{center}
\centerline{\begin{tabular}{ccc}
 \includegraphics[width=0.31\textwidth]{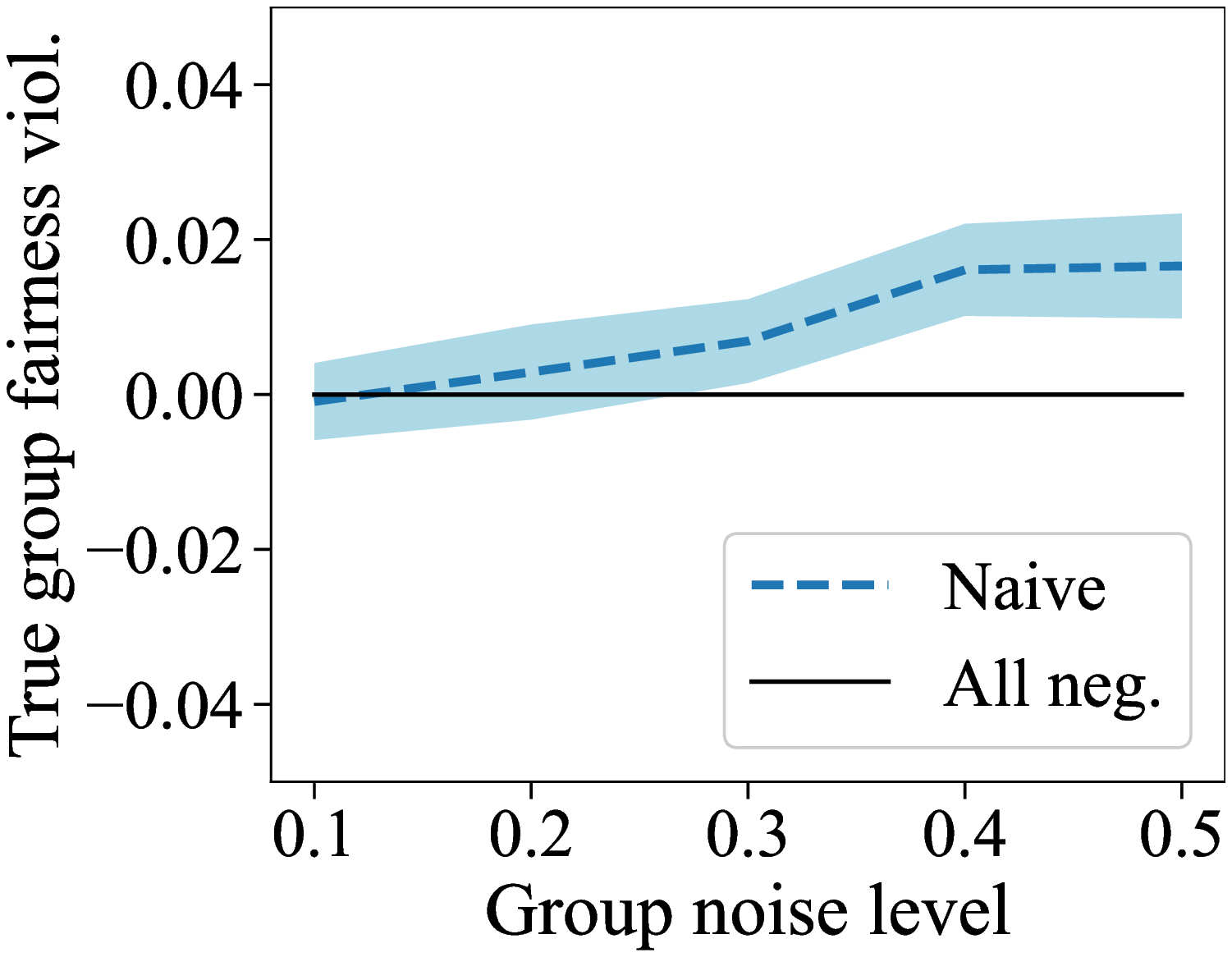} &
  \includegraphics[width=0.31\textwidth]{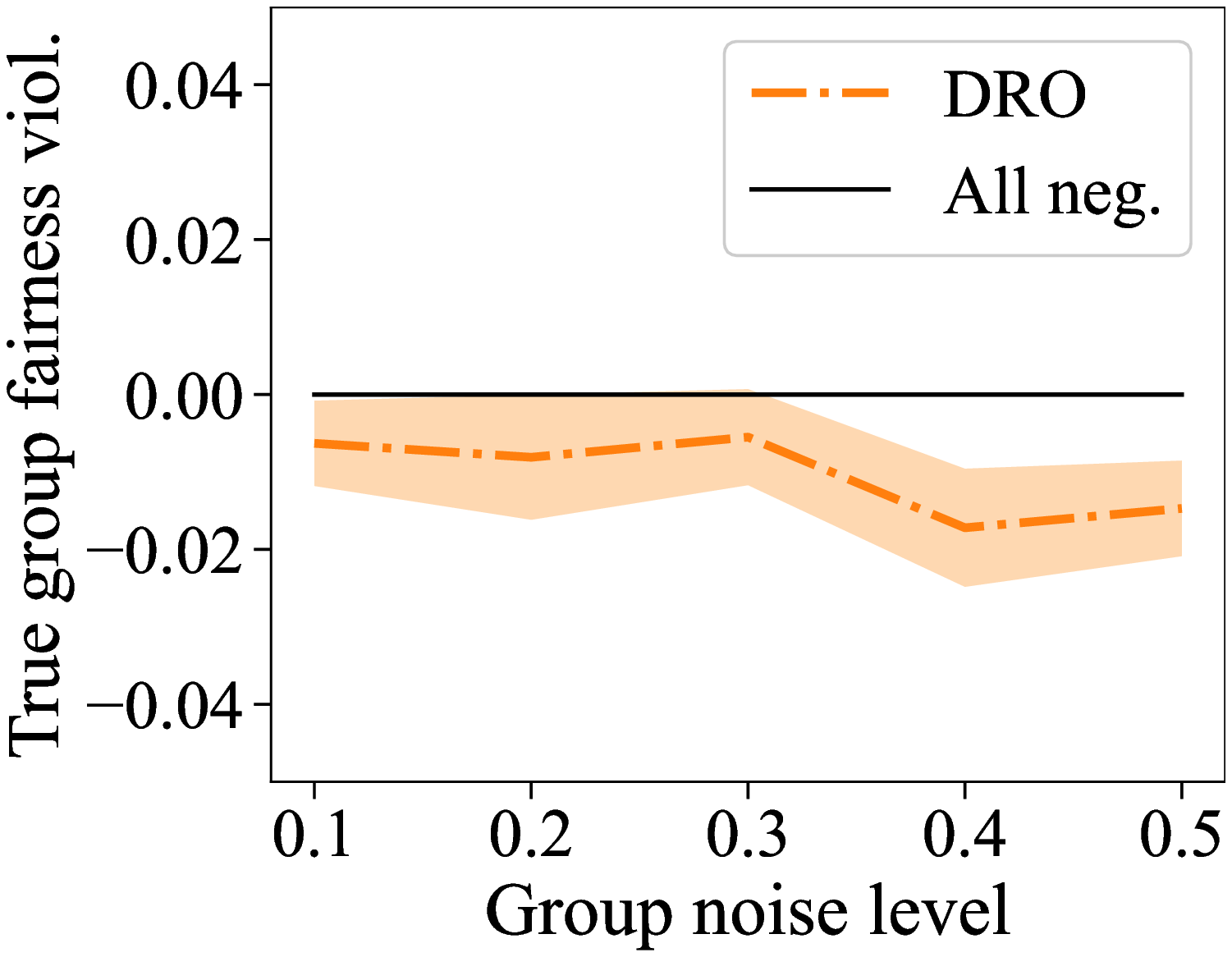} & \includegraphics[width=0.31\textwidth]{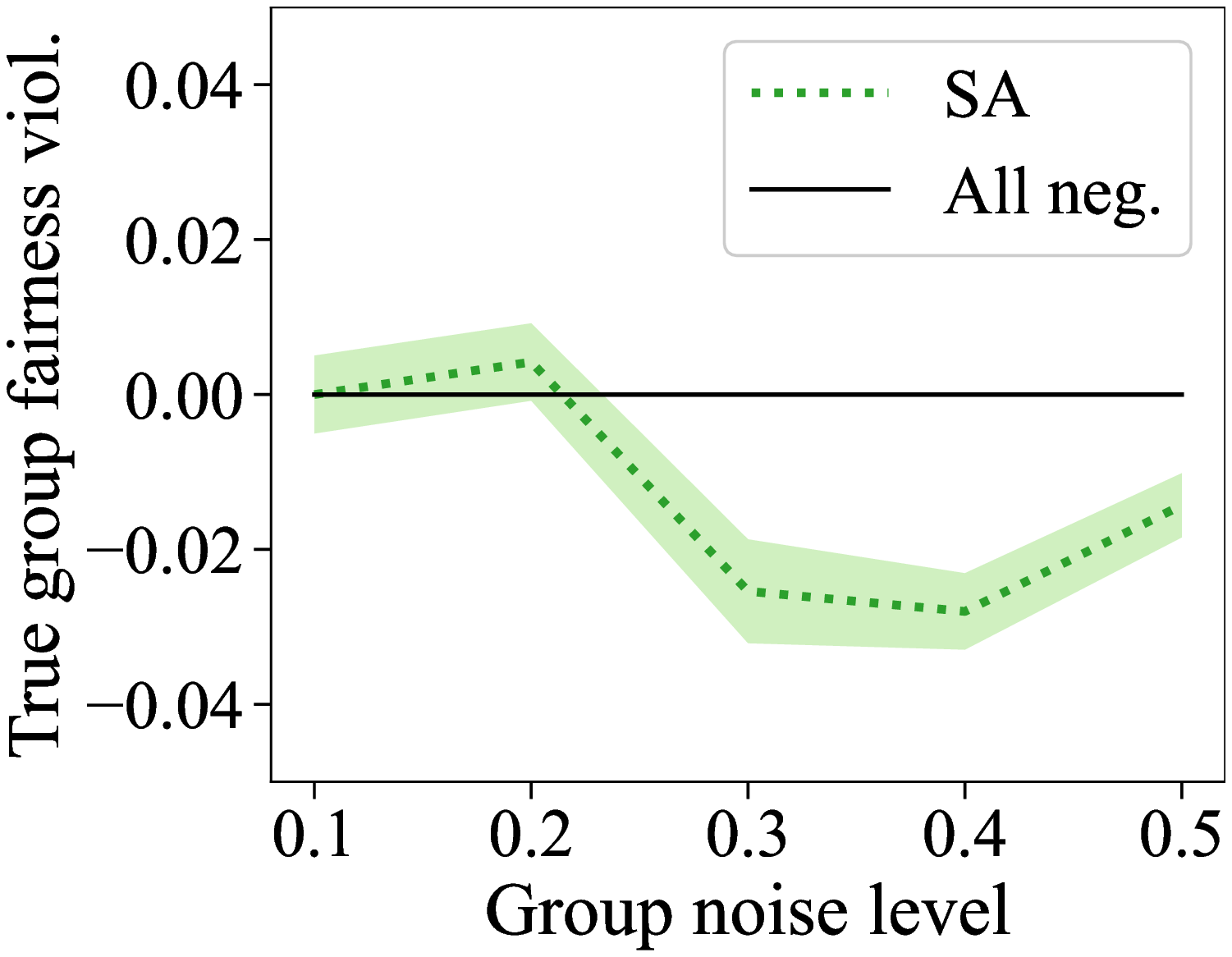} \\
\end{tabular}}
\caption{Case study 2 (Credit): maximum true group constraint violations on test set for the Naive, DRO, and soft assignments (SA) approaches for different group noise levels $\gamma$ on the Credit dataset (mean and standard error over 10 train/val/test splits). 
This figure shows the max constraint violation over all TPR and FPR constraints, and Figure \ref{fig:credit_constraints_tpr_fpr} in Appendix \ref{app:experiment_results} shows the breakdown of these constraint violations into the max TPR and the max FPR constraint violations. }
\label{fig:credit_constraints} 
\end{center}
\vskip -0.2in
\end{figure}


\begin{figure}[!ht]
\vskip -0.1in
\begin{center}
\centerline{\begin{tabular}{cc}
Adult error rates & Credit error rates \\
 \includegraphics[width=0.31\textwidth]{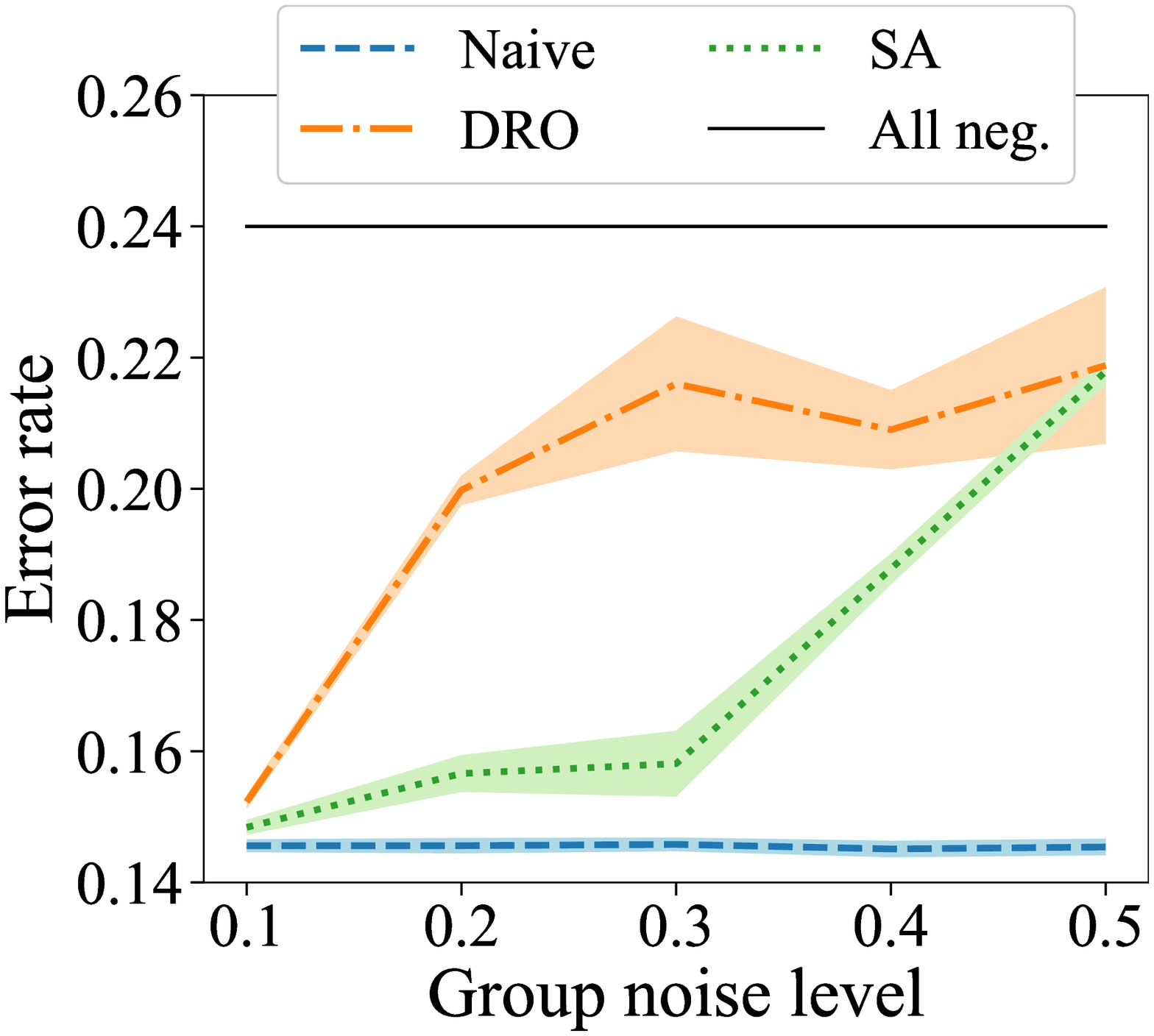} &
  \includegraphics[width=0.31\textwidth]{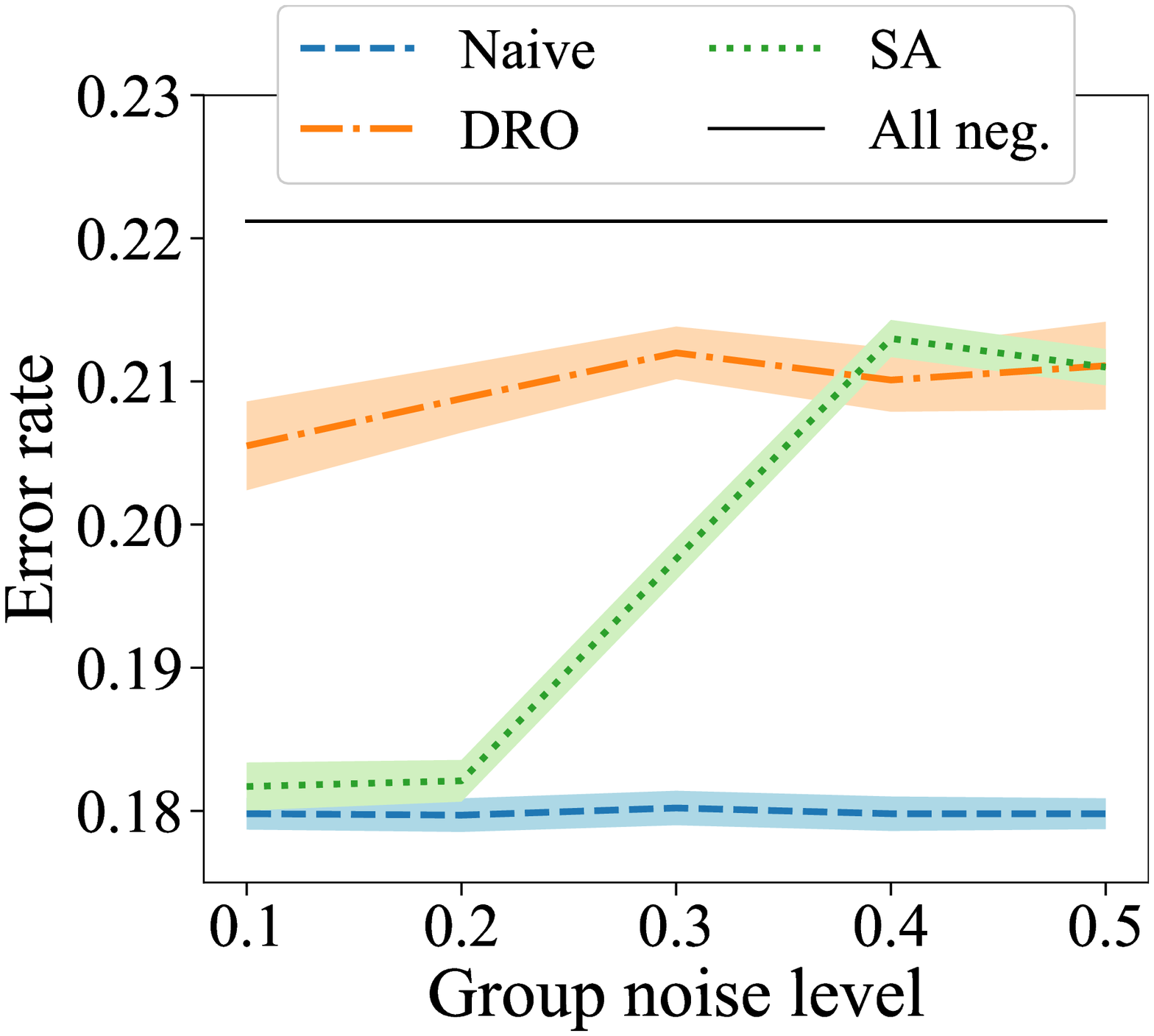} \\
\end{tabular}}
\caption{Error rates on test set for different group noise levels $\gamma$ on the Adult dataset (\textit{left}) and the Credit dataset (\textit{right}) (mean and standard error over 10 train/val/test splits). The black solid line represents the performance of the trivial ``all negatives'' classifier. The soft assignments (SA) approach achieves lower error rates than DRO, and as the noise level increases, the gap in error rate between the naive approach and each robust approach increases.}
\label{fig:objectives} 
\end{center}
\vskip -0.3in
\end{figure}

\subsection{Results}

In case study 1 (Adult), the unconstrained model achieves an error rate of $0.1447 \pm 0.0012$ (mean and standard error over 10 splits) and a maximum constraint violation of $0.0234 \pm 0.0164$ on test set with respect to the true groups. The model that assumes knowledge of the true groups achieves an error rate of $0.1459 \pm 0.0012$ and a maximum constraint violation of $-0.0469 \pm 0.0068$ on test set with respect to the true groups. As a sanity check, this demonstrates that when given access to the true groups, it is possible to satisfy the constraints on the test set with a reasonably low error rate.

In case study 2 (Credit), the unconstrained model achieves an error rate of $0.1797 \pm 0.0013$ (mean and standard error over 10 splits) and a maximum constraint violation of $0.0264 \pm 0.0071$ on the test set with respect to the true groups. The constrained model that assumes knowledge of the true groups achieves an error rate of $0.1796 \pm  0.0011$ and a maximum constraint violation of $-0.0105 \pm 0.0070$ on the test set with respect to the true groups. For this dataset, it was possible to satisfy the constraints with approximately the same error rate on test as the unconstrained model. Note that the unconstrained model achieved a lower error rate on the train set than the constrained model ($0.1792 \pm 0.0015$ unconstrained vs. $0.1798 \pm 0.0024$ constrained).  

For both case studies, Figures \ref{fig:adult_constraints} and \ref{fig:credit_constraints} show that the robust approaches DRO (\textit{center}) and soft group assignments (SA) (\textit{right}) satisfy the constraints on average for all noise levels. As the noise level increases, the na{\"i}ve approach (\textit{left}) has increasingly higher true group constraint violations.  The DRO and SA approaches come at a cost of a higher error rate than the  na{\"i}ve approach (Figure \ref{fig:objectives}).  The error rate of the na{\"i}ve approach is close to the model optimized with constraints on the true groups $G$, regardless of the noise level $\gamma$. However, as the noise increases, the na{\"i}ve approach no longer controls the fairness violations on the true groups $G$, even though it does satisfy the constraints on the noisy groups $\hat{G}$ (Figures \ref{fig:adult_proxy_constraint_volations} and \ref{fig:credit_proxy_constraint_volations} in Appendix \ref{app:experiment_results}). DRO generally suffers from a higher error rate compared to SA (Figure \ref{fig:objectives}), illustrating the conservatism of the DRO approach.

\section{Conclusion}
We explore the practical problem of enforcing group-based fairness for binary classification given noisy protected group information. In addition to providing new theoretical analysis of the na{\"i}ve approach of only enforcing fairness on the noisy groups, we also propose two new robust approaches that guarantee satisfaction of the fairness criteria on the true groups. For the DRO approach, Lemma \ref{lem:tv_bound} gives a theoretical bound on the TV distance to use in the optimization problem. For the soft group assignments approach, we provide a theoretically ideal algorithm and a practical alternative algorithm for satisfying the robust fairness criteria proposed by \citet{Kallus:2020} while minimizing a training objective. We empirically show that both of these approaches managed to satisfy the constraints with respect to the true groups, even under difficult noise models. 

In follow-up work, \citet{Narasimhan2020manyconstraints} provide a general method for enforcing a large number of constraints at once, and 
enforce constraints concurrently on many possible realizations of noisy protected groups under a given noise model. This can be seen as an extension of the Soft Group Assignments approach that we propose in Section \ref{sec:softweights}, which \citet{Narasimhan2020manyconstraints} describe in their Appendix.

One additional avenue of future work is to empirically compare the robust approaches when the noisy groups have different dimensionality from the true groups (Appendix \ref{app:dro-multi}). Second, the looseness of the bound in Lemma \ref{lem:tv_bound} can lead to over-conservatism of the DRO approach, suggesting a need to better calibrate the DRO neighborhood.
Finally, it would be valuable to study the impact of distribution mismatch between the main dataset and the auxiliary dataset. 

\newpage
\clearpage

\section*{Broader Impact}
\vskip -0.05in

As machine learning is increasingly employed
in high stakes environments,
any potential application has to be scrutinized
to ensure that it will not perpetuate, exacerbate, or create new injustices.
Aiming to make machine learning algorithms themselves intrinsically
fairer, more inclusive, and more equitable plays an important role in
achieving that goal. Group-based fairness \citep{Hardt:2016, Friedler:2019} is a popular approach that the machine learning community has used to define and evaluate fair machine learning algorithms. Until recently, such work has generally assumed access to clean, correct protected group labels in the data.
Our work addresses the technical challenge of enforcing group-based fairness criteria under noisy, unreliable, or outdated group information.
However, we emphasize that this technical improvement alone does not necessarily lead to an algorithm having positive societal impact, for reasons that we now delineate. 




\textbf{Choice of fairness criteria} 

First, our work does not address the choice of the group-based fairness criteria. Many different algorithmic fairness criteria have been proposed, with varying connections to prior sociopolitical framing \cite{Narayanan:2018,Hutchinson:2019}. From an algorithmic standpoint, these different choices of fairness criteria have been shown to lead to very different prediction outcomes and tradeoffs \cite{Friedler:2019}.  
Furthermore, even if a mathematical criterion may seem reasonable (e.g., equalizing positive prediction rates with \textit{demographic parity}), \citet{liu2018} show that the long-term impacts may not always be desirable, and the choice of criteria should be heavily influenced by domain experts, along with awareness of tradeoffs.

\textbf{Choice of protected groups}

In addition to the specification of fairness criteria, our work also assumes that the true protected group labels have been pre-defined by the practitioner. However, in real applications, the selection of appropriate true protected group labels is itself a nontrivial issue.

First, the measurement and delineation of these protected groups should not be overlooked, as ``the process of drawing boundaries around distinct social groups for fairness research is fraught; the construction of categories has a
long history of political struggle and legal argumentation''~\cite{hanna2020}.
Important considerations include the context in which the group labels were collected, who chose and collected them, and what implicit assumptions are being made by choosing these group labels. 
One example is the operationalization of race in the context of algorithmic fairness. \citet{hanna2020} critiques ``treating race as an attribute,
rather than a structural, institutional, and relational phenomenon.'' 
The choice of categories surrounding gender identity and sexual orientation have strong implications and consequences as well \cite{williams2014}, with entire fields dedicated to critiquing these constructs. \citet{jacobs2019} provide a general framework for understanding measurement issues for these sensitive attributes in the machine-learning setting, building on foundational work from the social sciences \cite{bandalos}.

Another key consideration when defining protected groups is problems of \textit{intersectionality} \cite{Crenshaw:1990, Hooks:1992}. Group-based fairness criteria inherently do not consider within-group inequality \cite{Kasy:2020}. Even if we are able to enforce fairness criteria robustly for a given set of groups, the intersections of groups may still suffer~\cite{Joy:2018}. 

\textbf{Domain specific considerations}

Finally,
we emphasize that group-based fairness criteria simply may not be sufficient to mitigate problems of significant background injustice in certain domains. 
\citet{abebe2020roles} argue that computational methods have mixed roles in addressing social problems, where they can serve as \textit{diagnostics}, \textit{formalizers}, and \textit{rebuttals}, and also that ``computing acts as synecdoche when it makes long-standing social problems newly salient in the public eye.'' Moreover, the use of the algorithm itself may perpetuate inequity, and in the case of credit scoring, create stratifying effects of economic classifications that shape life-chances \cite{fourcade2013classification}. We emphasize the importance of domain specific considerations ahead of time before applying any algorithmic solutions (even ``fair'' ones) in sensitive and impactful settings. 


\begin{ack}
We thank Rediet Abebe, Timnit Gebru, and Margaret Mitchell for many useful pointers into the broader socio-technical literature that provides essential context for this work. We thank Jacob Steinhardt for helpful technical discussions. We also thank Stefania Albanesi and Domonkos Vámossy for an inspiring early discussion of practical scenarios when noisy protected groups occur. Finally, we thank Kush Bhatia, Collin Burns, Mihaela Curmei, Frances Ding, Sara Fridovich-Keil, Preetum Nakkiran, Adam Sealfon, and Alex Zhao for their helpful feedback on an early version of the paper.
This material is based upon work supported by the National Science Foundation Graduate Research Fellowship Program under Grant No. DGE 1752814. 



\end{ack}

\bibliography{neurips-ref}
\bibliographystyle{plainnat}

\newpage
\clearpage

\appendix

\section{Proofs for Section \ref{sec:naive}}\label{app:proofs}
This section provides proofs and definitions details for the theorems and lemmas presented in Section \ref{sec:naive}.
\subsection{Proofs for TV distance}\label{app:proofs-TV}
\begin{definition}(TV distance)
Let $c(x,y) = \Ind(x \neq y)$ be a metric, and let $\pi$ be a coupling between probability distributions $p$ and $q$. Define the total variation (TV) distance between two distributions $p,q$ as
\begin{align*}
    TV(p, q) &= \inf_{\pi} \E_{X,Y\sim \pi} [c(X, Y)] \\
&\text{s.t.} \int \pi(x, y)dy = p(x), \int\pi(x, y)dx = q(y).
\end{align*}
\end{definition}

\begin{reptheorem}{thm:general}
Suppose a model with parameters $\theta$ satisfies fairness criteria with respect to the noisy groups $\hat{G}$:
$$\hat{g}_j(\theta) \leq 0 \;\;\forall j \in \mathcal{G}. $$
Suppose $|h(\theta,x_1,y_1) - h(\theta, x_2, y_2)| \leq 1$ for any $(x_1, y_1) \neq (x_2, y_2)$. If $TV(p_j, \hat{p}_j) \leq \gamma_j$ for all $j \in \mathcal{G}$, then the fairness criteria with respect to the true groups $G$ will be satisfied within slacks $\gamma_j$ for each group:
$$ g_j(\theta)  \leq \gamma_j \quad \forall j \in \mathcal{G}.$$
\end{reptheorem}
\begin{proof}
For any group label $j$,
\begin{align*}
    g_j(\theta) = g_j(\theta) - \hat{g}_j(\theta) + \hat{g}_j(\theta) \leq |g_j(\theta) - \hat{g}_j(\theta)| + \hat{g}_j(\theta).
\end{align*}
By Kantorovich-Rubenstein theorem (provided here as Theorem \ref{thm:KR}), we also have
\begin{align*}
    |\hat{g}_j(\theta) - g_j(\theta)| &=|\E_{X,Y \sim \hat{p}_j}[h(\theta,X,Y)] - \E_{X,Y \sim p_j}[h(\theta,X,Y)]| \leq TV(p_j, \hat{p}_j).
\end{align*}
By assumption that $\theta$ satisifes fairness constraints with respect to the noisy groups $\hat{G}$, $\hat{g}_j(\theta) \leq 0$. Thus, we have the desired result that $g_j(\theta) \leq TV(p_j, \hat{p}_j) \leq \gamma_j$.

Note that if $p_j$ and $\hat{p}_j$ are discrete, then the TV distance $TV(p_j, \hat{p_j})$ could be very large. In that case, the bound would still hold, but would be loose. 
\end{proof}

\begin{theorem}\label{thm:KR} (Kantorovich-Rubinstein).\footnote{Edwards, D.A. On the Kantorovich–Rubinstein theorem. \textit{Expositiones Mathematicae}, 20(4):387-398, 2011.} Call a function $f$ Lipschitz in $c$ if $|f(x) - f(y)| \leq c(x, y)$ for all
$x, y$, and let $\mathcal{L}(c)$ denote the space of such functions. If $c$ is a metric, then we have
$$W_c(p, q) = \sup_{f\in \mathcal{L}(c)} \E_{X \sim p}[f(X)] - \E_{X \sim q}[f(X)].$$

As a special case, take $c(x, y) = \mathbb{I}(x \neq y)$ (corresponding to TV distance). Then $f \in \mathcal{L}(c)$ if and only if
$|f(x) - f(y)| \leq 1$ for all $x \neq y$. By translating $f$, we can equivalently take the supremum over all $f$ mapping
to $[0, 1]$. This says that
$$TV(p, q) = \sup_{f:\mathcal{X} \to [0,1]} \E_{X \sim p}[f(X)] - \E_{X \sim q}[f(X)]$$
\end{theorem}

\begin{replemma}{lem:tv_bound} Suppose $P(G = i) = P(\hat{G} = i)$ for a given $i \in \{1,2,...,m\}$. Then $TV(p_i, \hat{p}_i) \leq P(G \neq \hat{G} | G = i) $.
\end{replemma}
\begin{proof}
For probability measures $p_i$ and $\hat{p}_i$, the TV distance is given by
$$TV(p_i, \hat{p}_i) = \sup\{|p_i(A) - \hat{p}_i(A)| : A \text{ is a measurable event}\}.$$
Fix $A$ to be any measurable event for both $p_i$ and $\hat{p}_i$. This means that $A$ is also a measurable event for $p$, the distribution of the random variables $X,Y$. By definition of $p_i$, $p_i(A) = P(A | G = i)$. Then
\begin{align*}
    |p_i(A) - \hat{p}_i(A)| &= |P(A | G = i) - P(A | \hat{G} = i)| \\
    &= |P(A | G = i, \hat{G} = i)P(\hat{G} = i| G = i) \\
    & \quad + P(A | G = i, \hat{G} \neq i)P(\hat{G} \neq i| G = i) \\
    &\quad - P(A | \hat{G} = i, G = i) P(G = i | \hat{G} = i) \\ 
    &\quad - P(A | \hat{G} = i, G \neq i)P(G \neq i | \hat{G} = i)| \\
    &= |P(A | G = i, \hat{G} = i)\left(P(\hat{G} = i| G = i) - P(G = i | \hat{G} = i)\right) \\
    &\quad - P(\hat{G} \neq G | G = i)\left(P(A | G = i, \hat{G} \neq i) - P(A | \hat{G} = i, G \neq i) \right)| \\
    &= |0 - P(\hat{G} \neq G | G = i)\left(P(A | G = i, \hat{G} \neq i) - P(A | \hat{G} = i, G \neq i) \right)| \\
    &\leq P(\hat{G} \neq G | G = i)
\end{align*}
The second equality follows from the law of total probability. The third and the fourth equalities follow from the assumption that $P(G = i) = P(\hat{G} = i)$, which implies that $P(\hat{G} = G | G = i) = P(G = \hat{G}| \hat{G} = i)$ since $$P(G = \hat{G} | G = i) = \frac{P(G = \hat{G}, G = i)}{P(G = i)} = \frac{P(G = \hat{G}, \hat{G} = i)}{P(\hat{G} = i)} = P(G = \hat{G} | \hat{G} = i).$$
This further implies that $P(\hat{G} \neq i | G = i) = P(G \neq i| \hat{G} = i)$.

Since $|p_i(A) - \hat{p}_i(A)| \leq P(\hat{G} \neq G | G = i)$ for any measurable event $A$, the supremum over all events $A$ is also bounded by $P(\hat{G} \neq G | G = i)$. This gives the desired bound on the TV distance.
\end{proof}




\subsection{Generalization to Wasserstein distances}
\label{app:wass}

Theorem \ref{thm:general} can be directly extended to loss functions that are Lipschitz in other metrics. To do so, we first provide a more general definition of Wasserstein distances:

\begin{definition}(Wasserstein distance)
Let $c(x,y)$ be a metric, and let $\pi$ be a coupling between $p$ and $q$. Define the Wasserstein distance between two distributions $p,q$ as
\begin{align*}
    W_c(p, q) = \inf_{\pi}\;\; &\E_{X,Y\sim \pi} [c(X, Y)] \\
\text{s.t.} &\int \pi(x, y)dy = p(x), \int\pi(x, y)dx = q(y).
\end{align*}
\end{definition}

As a familiar example, if $c(x,y) = ||x-y||_2$, then $W_c$ is the earth-mover distance, and $\mathcal{L}(c)$ is the class of 1-Lipschitz functions. Using the Wasserstein distance $W_c$ under different metrics $c$, we can bound the fairness violations for constraint functions $h$ beyond those specified for the TV distance in Theorem \ref{thm:general}. 

\begin{theorem}\label{thm:wasserstein}

Suppose a model with parameters $\theta$ satisfies fairness criteria with respect to the noisy groups $\hat{G}$:
$$\hat{g}_j(\theta) \leq 0 \;\;\forall j \in \mathcal{G}. $$
Suppose the function $h$ satisfies $|h(\theta,x_1,y_1) - h(\theta, x_2, y_2)| \leq c((x_1,y_1), (x_2,y_2))$ for any $(x_1, y_1) \neq (x_2, y_2)$ w.r.t a metric $c$. 
If $W_c(p_j, \hat{p}_j) \leq \gamma_j$ for all $j \in \mathcal{G}$, then the fairness criteria with respect to the true groups $G$ will be satisfied within slacks $\gamma_j$ for each group:
$$  g_j(\theta)  \leq  \gamma_j \quad \forall j \in \mathcal{G}.$$
\end{theorem}
\begin{proof}

By the triangle inequality, for any group label $j$,
\begin{align*}
    |g_j(\theta)  -  g(\theta)| \leq |g_j(\theta) - \hat{g}_j(\theta)| + \hat{g}_j(\theta)
\end{align*}
By Kantorovich-Rubenstein theorem (provided here as Theorem \ref{thm:KR}), we also have
\begin{align*}
    |\hat{g}_j(\theta) - g_j(\theta)| &=|\E_{X,Y \sim \hat{p}_j}[h(\theta,X,Y)] - \E_{X,Y \sim p_j}[h(\theta,X,Y)]|\\
    &\leq W_c(p_j, \hat{p}_j).
\end{align*}
By the assumption that $\theta$ satisifes fairness constraints with respect to the noisy groups $\hat{G}$, $ \hat{g}_j(\theta) \leq 0 $.
Therefore, combining these with the triangle inequality, we get the desired result.
\end{proof}

\section{Details on DRO formulation for TV distance}\label{app:dro}
Here we describe the details on solving the DRO problem (\ref{eq:dro}) with TV distance using the empirical Lagrangian formulation. We also provide the pseudocode we used for the projected gradient-based algorithm to solve it. 

\subsection{Empirical Lagrangian Formulation}\label{app:dro-emp-lag-form}
We rewrite the constrained optimization problem (\ref{eq:dro}) as a minimax problem using the Lagrangian formulation. We also convert all expectations into expectations over empirical distributions given a dataset of $n$ samples $(X_1,Y_1,G_1),...,(X_n,Y_n,G_n)$.

Let $n_j$ denote the number of samples that belong to a true group $G = j$. Let the empirical distribution $\hat{p}_j \in \mathbb{R}^n$ be a vector with $i$-th entry $\hat{p}_j^i = \frac{1}{n_j}$ if the $i$-th example has a noisy group membership $\hat{G}_i = j$, and $0$ otherwise. Replacing all expectations with expectations over the appropriate empirical distributions, the empirical form of (\ref{eq:dro}) can be written as:
\begin{align}
\begin{split}
\min_{\theta} \quad &\frac{1}{n}\sum_{i=1}^n l(\theta, X_i, Y_i) \\
\text{s.t.} \quad & \max_{\tilde{p}_j \in \mathbb{B}_{\gamma_j}(\hat{p}_j)}  \sum_{i=1}^n\tilde{p}_j^i h(\theta, X_i, Y_i)  \leq 0 \quad \forall j \in \mathcal{G}
\end{split}\label{eq:dro_emp}
\end{align}
where $\mathbb{B}_{\gamma_j}(\hat{p}_j) = \{\tilde{p}_j \in \mathbb{R}^n : \frac{1}{2}\sum_{i=1}^n |\tilde{p}_j^i - \hat{p}_j^i| \leq \gamma_j, \sum_{i=1}^n \tilde{p}_j^i = 1, \tilde{p}_j^i \geq 0 \quad \forall i = 1,...,n\}$. \\

For ease of notation, for $j \in \{1, 2, ..., m\}$, let 
$$f(\theta) = \frac{1}{n}\sum_{i=1}^n l(\theta, X_i, Y_i)$$
$$f_j(\theta, \tilde{p}_j) = \sum_{i=1}^n\tilde{p}_j^i h(\theta, X_i, Y_i). $$
Then the Lagrangian of the empirical formulation (\ref{eq:dro_emp}) is
$$\mathcal{L}(\theta, \lambda) = f(\theta) + \sum_{j=1}^m \lambda_j \max_{\tilde{p}_j \in \mathbb{B}_{\gamma}(\hat{p}_j)} f_j(\theta, \tilde{p}_j)$$
and problem (\ref{eq:dro_emp}) can be rewritten as
$$\min_{\theta} \max_{\lambda \geq 0} f(\theta) + \sum_{j=1}^m \lambda_j \max_{\tilde{p}_j \in \mathbb{B}_{\gamma}(\hat{p}_j)} f_j(\theta, \tilde{p}_j) $$
Moving the inner max out of the sum and rewriting the constraints as $\ell_1$-norm constraints:
\begin{align}
    \begin{split}
        \min_{\theta} \max_{\lambda \geq 0}\max_{\substack{\tilde{p}_j \in \mathbb{R}^n, \tilde{p}_j \geq 0,\\ j = 1,...,m}} &f(\theta) + \sum_{j=1}^m \lambda_j f_j(\theta, \tilde{p}_j) \\
        \text{s.t. } & ||\tilde{p}_j - \hat{p}_j||_1 \leq 2\gamma_j,\;\; || \tilde{p}_j ||_1 = 1 \quad \forall j \in \{1,...,m\}
    \end{split}\label{eq:dro-emp-app}
\end{align}
Since projections onto the $\ell_1$-ball can be done efficiently~\citep{duchi2008efficient}, we can solve problem (\ref{eq:dro-emp-app}) using a projected gradient descent ascent (GDA) algorithm. This is a simplified version of the algorithm introduced by \citet{Namkoong:2016} for solving general classes of DRO problems. We provide pseudocode in Algorithm \ref{alg:dro}, as well as an actual implementation in the attached code.

\subsection{Projected GDA Algorithm for DRO}
\label{app:dro-alg}
\begin{algorithm}[H] 
\caption{Project GDA Algorithm}
\label{alg:dro}
\begin{algorithmic}[1]
\REQUIRE learning rates $\eta_\theta > 0$, $\eta_\lambda > 0$, $\eta_p > 0$, estimates of $P( G \neq \hat{G} | \hat{G} = j)$ to specify $\gamma_j$.
\vskip 0.05in
\FOR{$t = 1, \ldots, T$}
\vskip 0.05in

\STATE \textit{Descent step on $\theta$:} \\
 $\theta^{(t+1)} \gets \theta^{(t)} - \eta_\theta\nabla_\theta f(\theta^{(t)}) - \eta_\theta \sum_{j=1}^m \lambda_j^{(t)} \nabla_\theta f_j(\theta^{(t)}, \tilde{p}_j^{(t)})$
\vskip 0.05in
\STATE \textit{Ascent step on $\lambda$}:\\

$\lambda^{(t+1)}_j \gets \lambda^{(t)}_j + \eta_\lambda f_j(\theta, \tilde{p}_j^{(t)})$\\
\vskip 0.05in
\FOR{$j = 1,...,m$}
        \STATE \textit{Ascent step on} $\tilde{p}_j$: $\tilde{p}_j^{(t+1)} \gets \tilde{p}_j^{(t)} + \eta_p \lambda_j^{(t)}\nabla_{\tilde{p}_j} f_j(\theta^{(t)}, \tilde{p}_j^{(t)})$
        \STATE \textit{Project $\tilde{p}_j^{(t+1)}$ onto $\ell_1$-norm constraints}: $||\tilde{p}_j^{(t+1)} - \hat{p}_j||_1 \leq 2\gamma_j, || \tilde{p}_j^{(t+1)} ||_1 = 1$
\ENDFOR   
\ENDFOR
\STATE \textbf{return} { $\theta^{(t^{*})}$ where $t^{*}$ denotes the \textit{best} iterate that satisfies the constraints in (\ref{eq:dro}) with the lowest objective.}
\end{algorithmic}
\end{algorithm}

\subsection{Equalizing TPRs and FPRs using DRO}\label{app:tpr_dro} 
In the two case studies in Section \ref{sec:experiments}, we enforce \textit{equality of opportunity} and \textit{equalized odds} \cite{Hardt:2016} by equalizing true positive rates (TPRs) and/or false positive rates (FPRs) within some slack $\alpha$. In this section, we describe in detail the implementation of the constraints for equalizing TPRs and FPRs under the DRO approach.

To equalize TPRs with slack $\alpha$ under the DRO approach, we set 
\begin{equation}\label{eq:dro_tpr}
    \tilde{g}_j^{\text{TPR}}(\theta) = \frac{\E_{X,Y \sim p}[\Ind(Y=1) \Ind(\hat{Y}=1)]}{\E_{X,Y \sim p}[\Ind(Y=1)]} - \frac{ \E_{X,Y \sim \tilde{p}_j}[ \Ind(Y=1) \Ind(\hat{Y}=1)] }{ \E_{X,Y \sim \tilde{p}_j}[\Ind(Y=1)] } - \alpha.
\end{equation}
The first term corresponds to the TPR for the full population.
The second term estimates the TPR for group $j$. Setting $\alpha = 0$ exactly equalizes true positive rates.

To equalize FPRs with slack $\alpha$ under the DRO approach, we set 
\begin{equation}\label{eq:dro_fpr}
    \tilde{g}_j^{\text{FPR}}(\theta) = \frac{ \E_{X,Y \sim \tilde{p}_j}[ \Ind(Y=0) \Ind(\hat{Y}=1)] }{ \E_{X,Y \sim \tilde{p}_j}[\Ind(Y=0)] } - \frac{\E_{X,Y \sim p}[\Ind(Y=0) \Ind(\hat{Y}=1)]}{\E_{X,Y \sim p}[\Ind(Y=0)]} - \alpha.
\end{equation}
The first term estimates the FPR for group $j$. The second term corresponds to the FPR for the full population. Setting $\alpha = 0$ exactly equalizes false positive rates.

To equalize TPRs for Case Study 1, we apply $m$ constraints, \\$\left\{\max_{\tilde{p}_j: TV(\tilde{p}_j, \hat{p}_j) \leq \gamma_j, \tilde{p}_j \ll p} \tilde{g}_j^{\text{TPR}}(\theta) \leq 0\right\} \; \forall j \in \mathcal{G}$. 

To equalize both TPRs and FPRs simultaneously for Case Study 2, we apply $2m$ constraints, $\left\{\max_{\tilde{p}_j: TV(\tilde{p}_j, \hat{p}_j) \leq \gamma_j, \tilde{p}_j \ll p} \tilde{g}_j^{\text{TPR}}(\theta) \leq 0, \max_{\tilde{p}_j: TV(\tilde{p}_j, \hat{p}_j) \leq \gamma_j, \tilde{p}_j \ll p} \tilde{g}_j^{\text{FPR}}(\theta) \leq 0\right\} \; \forall j \in \mathcal{G}$.

\subsubsection{$h(\theta, X, Y)$ for equalizing TPRs and FPRs}
 
Since the notation in Section \ref{sec:dro} and in the rest of the paper uses generic functions $h$ to express the group-specific constraints, we show in Lemma \ref{lem:dro_tpr_l1} that the constraint using $\tilde{g}_j^{\text{TPR}}(\theta)$ in Equation (\ref{eq:dro_tpr}) can also be written as an equivalent constraint in the form of Equation (\ref{eq:dro}), as
$$\tilde{g}_j^{\text{TPR}}(\theta) = \E_{X,Y\sim \tilde{p}_j}[h^{\text{TPR}}(\theta, X, Y)]$$
for some function $h^{\text{TPR}}: \Theta \times \mathcal{X} \times \mathcal{Y} \to \mathbb{R}$.

\begin{lemma}\label{lem:dro_tpr_l1}
 Denote $\hat{Y}$ as $\Ind(\phi(X;\theta) > 0)$. Let $h^{\text{TPR}}(\theta,X,Y)$ be given by

$$h^{\text{TPR}}(\theta,X,Y) = \frac{1}{2} \left(-\Ind(\hat{Y}=1, Y = 1) -  \Ind(Y=1) \left(\alpha - \frac{\E_{X,Y \sim p}[\Ind(Y=1,\hat{Y}=1)]}{\E_{X,Y \sim p}[\Ind(Y=1)]}  \right)\right).$$

Then \begin{align*}
    \frac{\E_{X,Y \sim p}[\Ind(Y=1) \Ind(\hat{Y}=1)]}{\E_{X,Y \sim p}[\Ind(Y=1)]} - \frac{ \E_{X,Y \sim \tilde{p}_j}[ \Ind(Y=1) \Ind(\hat{Y}=1)] }{ \E_{X,Y \sim \tilde{p}_j}[\Ind(Y=1)]}  - \alpha \leq 0 \\
        \iff \E_{X,Y \sim \tilde{p}_j}[h^{\text{TPR}}(\theta, X, Y)] \leq 0.
\end{align*}
\end{lemma}
\begin{proof}
Substituting the given function $h^{\text{TPR}}(\theta, X, Y)$, and using the fact that \\ $\E_{X,Y \sim \tilde{p}_j}[\Ind(Y=1)] \geq 0$:
\begin{align*}
    &\E_{X,Y \sim \tilde{p}_j}[h^{\text{TPR}}(\theta, X, Y)] \leq 0 \\
    &\iff \E_{X,Y \sim \tilde{p}_j}\left[\frac{1}{2} \left(-\Ind(\hat{Y}=1, Y = 1) -  \Ind(Y=1) \left(\alpha - \frac{\E_{X,Y \sim p}[\Ind(Y=1,\hat{Y}=1)]}{\E_{X,Y \sim p}[\Ind(Y=1)]}  \right)\right)\right] \leq 0 \\
    &\iff -\E_{X,Y \sim \tilde{p}_j}[\Ind(\hat{Y}=1, Y = 1)] -\E_{X,Y \sim \tilde{p}_j}\left[\Ind(Y=1) \left(\alpha - \frac{\E_{X,Y \sim p}[\Ind(Y=1,\hat{Y}=1)]}{\E_{X,Y \sim p}[\Ind(Y=1)]}  \right)\right] \leq 0 \\
    &\iff -\E_{X,Y \sim \tilde{p}_j}[\Ind(\hat{Y}=1, Y = 1)] -\alpha\E_{X,Y \sim \tilde{p}_j}[\Ind(Y=1)] \\
    & \quad\quad\quad + \frac{\E_{X,Y \sim p}[\Ind(Y=1,\hat{Y}=1)]}{\E_{X,Y \sim p}[\Ind(Y=1)]} \E_{X,Y \sim \tilde{p}_j}[\Ind(Y=1)] \leq 0 \\
    &\iff \frac{\E_{X,Y \sim p}[\Ind(Y=1,\hat{Y}=1)]}{\E_{X,Y \sim p}[\Ind(Y=1)]} -\frac{\E_{X,Y \sim \tilde{p}_j}[\Ind(\hat{Y}=1, Y = 1)]}{\E_{X,Y \sim \tilde{p}_j}[\Ind(Y=1)]} -\alpha \leq 0 
\end{align*}
\end{proof}
By similar proof, we also show in Lemma \ref{lem:dro_fpr_l1} that the constraint using $\tilde{g}_j^{\text{FPR}}(\theta)$ in Equation (\ref{eq:dro_fpr}) can also be written as an equivalent constraint in the form of Equation (\ref{eq:dro}), as
$$\tilde{g}_j^{\text{FPR}}(\theta) = \E_{X,Y\sim \tilde{p}_j}[h^{\text{FPR}}(\theta, X, Y)]$$
for some function $h^{\text{FPR}}: \Theta \times \mathcal{X} \times \mathcal{Y} \to \mathbb{R}$.

\begin{lemma}\label{lem:dro_fpr_l1}
 Denote $\hat{Y}$ as $\Ind(\phi(X;\theta) > 0)$. Let $h^{\text{FPR}}(\theta,X,Y)$ be given by

$$h^{\text{FPR}}(\theta,X,Y) = \frac{1}{2} \left(\Ind(\hat{Y}=1, Y = 0) -  \Ind(Y=0) \left(\alpha + \frac{\E_{X,Y \sim p}[\Ind(Y=0,\hat{Y}=1)]}{\E_{X,Y \sim p}[\Ind(Y=0)]}  \right)\right).$$

Then \begin{align*}
    \frac{ \E_{X,Y \sim \tilde{p}_j}[ \Ind(Y=0) \Ind(\hat{Y}=1)] }{ \E_{X,Y \sim \tilde{p}_j}[\Ind(Y=0)]} - \frac{\E_{X,Y \sim p}[\Ind(Y=0) \Ind(\hat{Y}=1)]}{\E_{X,Y \sim p}[\Ind(Y=0)]}   - \alpha \leq 0 \\
        \iff \E_{X,Y \sim \tilde{p}_j}[h^{\text{FPR}}(\theta, X, Y)] \leq 0.
\end{align*}
\end{lemma}
\begin{proof}
Substituting the given function $h^{\text{FPR}}(\theta, X, Y)$, and using the fact that \\ $\E_{X,Y \sim \tilde{p}_j}[\Ind(Y=0)] \geq 0$:
\begin{align*}
    &\E_{X,Y \sim \tilde{p}_j}[h^{\text{FPR}}(\theta, X, Y)] \leq 0 \\
    &\iff \E_{X,Y \sim \tilde{p}_j}\left[\frac{1}{2} \left(\Ind(\hat{Y}=1, Y = 0) -  \Ind(Y=0) \left(\alpha + \frac{\E_{X,Y \sim p}[\Ind(Y=0,\hat{Y}=1)]}{\E_{X,Y \sim p}[\Ind(Y=0)]}  \right)\right)\right] \leq 0 \\
    &\iff \E_{X,Y \sim \tilde{p}_j}[\Ind(\hat{Y}=1, Y = 0)] -\E_{X,Y \sim \tilde{p}_j}\left[\Ind(Y=0) \left(\alpha + \frac{\E_{X,Y \sim p}[\Ind(Y=0,\hat{Y}=1)]}{\E_{X,Y \sim p}[\Ind(Y=0)]}  \right)\right] \leq 0 \\
    &\iff \E_{X,Y \sim \tilde{p}_j}[\Ind(\hat{Y}=1, Y = 0)] -\alpha\E_{X,Y \sim \tilde{p}_j}[\Ind(Y=0)] \\
    & \quad\quad\quad - \frac{\E_{X,Y \sim p}[\Ind(Y=0,\hat{Y}=1)]}{\E_{X,Y \sim p}[\Ind(Y=0)]} \E_{X,Y \sim \tilde{p}_j}[\Ind(Y=0)] \leq 0 \\
    &\iff \frac{\E_{X,Y \sim \tilde{p}_j}[\Ind(\hat{Y}=1, Y = 0)]}{\E_{X,Y \sim \tilde{p}_j}[\Ind(Y=0)]} - \frac{\E_{X,Y \sim p}[\Ind(Y=0,\hat{Y}=1)]}{\E_{X,Y \sim p}[\Ind(Y=0)]} - \alpha \leq 0 
\end{align*}
\end{proof}

\subsection{DRO when $\hat{G}$ and $G$ have different dimensionalities}\label{app:dro-multi}
The soft assignments approach is naturally formulated to be able to handle $G \in \mathcal{G} = \{1,...,m\}$ and $\hat{G} \in \hat{\mathcal{G}} = \{1,...,\hat{m}\}$ when $\hat{m} \neq m$. 
The DRO approach can be extended to handle this case by generalizing Lemma \ref{lem:tv_bound} 
to $TV(p_j, \hat{p}_i) \leq P(\hat{G} \neq i | G = j), j \in \mathcal{G}, i \in \hat{\mathcal{G}}$, and generalizing the DRO formulation to have the true group distribution $p_j$ bounded in a TV distance ball centered at $\hat{p}_i$. Empirically comparing this generalized DRO approach to the soft group assignments approach when $\hat{m} \neq m$ is an interesting avenue of future work.


\section{Further details for soft group assignments approach}
\label{app:soft}
Here we provide additional technical details regarding the soft group assignments approach introduced in Section \ref{eq:softweights}.

\subsection{Derivation for $\E[h(\theta, X, Y) | G = j]$}\label{app:tower_property}
Here we show $\E[h(\theta, X, Y) | G = j] = \frac{\E[h(\theta, X, Y) P( G = j| \hat{Y}, Y, \hat{G})]}{P(G=j)} $, assuming that $h(\theta, X, Y)$ depends on $X$ through $\hat{Y}$, i.e. $\hat{Y} = \Ind(\phi(\theta, X) > 0)$. Using the tower property and the definition of conditional expectation:
\begin{align}
    \begin{split}
    \E[h(\theta, X, Y) | G = j] &= \frac{\E[h(\theta, X, Y) \Ind( G = j)]}{P(G = j)} \\
    &= \frac{\E[\E[h(\theta, X, Y) \Ind( G = j)| \hat{Y}, Y, \hat{G}]]}{P(G=j)}  \\
    &= \frac{\E[h(\theta, X, Y) \E[\Ind( G = j)| \hat{Y}, Y, \hat{G}]]}{P(G=j)}  \\
    &= \frac{\E[h(\theta, X, Y) P( G = j| \hat{Y}, Y, \hat{G})]}{P(G=j)} 
    \end{split} \label{eq:kallus_tower_property}
\end{align}

\subsection{Equalizing TPRs and FPRs using soft group assignments}\label{app:tpr_sa} 
In the two case studies in Section \ref{sec:experiments}, we enforce \textit{equality of opportunity} and \textit{equalized odds} \cite{Hardt:2016} by equalizing true positive rates (TPRs) and/or false positive rates (FPRs) within some slack $\alpha$. In this section, we describe in detail the implementation of the constraints for equalizing TPRs and FPRs under the soft group assignments approach.

To equalize TPRs with slack $\alpha$ under the soft group assignments approach, we set 
\begin{equation}\label{eq:sw_tpr_kallus}
    g_j^{\text{TPR}}(\theta, w) = \frac{\E[\Ind(Y=1) \Ind(\hat{Y}=1)]}{\E[\Ind(Y=1)]} - \frac{ \E[ \Ind(Y=1) \Ind(\hat{Y}=1) w(j | \hat{Y}, Y, \hat{G} )] }{ \E[ \Ind(Y=1) w(j | \hat{Y}, Y, \hat{G}) ] } - \alpha.
\end{equation}
The first term corresponds to the TPR for the full population.
The second term estimates the TPR for group $j$ as done by \citet{Kallus:2020} in Equation (5) and Proposition 8. Setting $\alpha = 0$ exactly equalizes true positive rates.

To equalize FPRs with slack $\alpha$ under the soft group assignments approach, we set 
\begin{equation}\label{eq:sw_fpr_kallus}
    g_j^{\text{FPR}}(\theta, w) =  \frac{ \E[ \Ind(Y=0) \Ind(\hat{Y}=1) w(j | \hat{Y}, Y, \hat{G} )] }{ \E[ \Ind(Y=0) w(j | \hat{Y}, Y, \hat{G}) ] } - \frac{\E[\Ind(Y=0) \Ind(\hat{Y}=1)]}{\E[\Ind(Y=0)]} - \alpha.
\end{equation}

The first term estimates the FPR for group $j$ as done previously for the TPR. The second term corresponds to the FPR for the full population. Setting $\alpha = 0$ exactly equalizes false positive rates.

To equalize TPRs for Case Study 1, we apply $m$ constraints, $\left\{\max_{w \in \mathcal{W}(\theta)} g_j^{\text{TPR}}(\theta, w) \leq 0\right\} \; \forall j \in \mathcal{G}$. To equalize both TPRs and FPRs simultaneously for Case Study 2, we apply $2m$ constraints, $\left\{\max_{w \in \mathcal{W}(\theta)} g_j^{\text{TPR}}(\theta, w) \leq 0, \max_{w \in \mathcal{W}(\theta)} g_j^{\text{FPR}}(\theta, w) \leq 0\right\} \; \forall j \in \mathcal{G}$.

\subsubsection{$h(\theta, X, Y)$ for equalizing TPRs and FPRs}
 
Since the notation in Section \ref{sec:softweights} and in the rest of the paper uses generic functions $h$ to express the group-specific constraints, we show in Lemma \ref{lem:sw_tpr_l1} that the constraint using $g_j^{\text{TPR}}(\theta, w)$ in Equation (\ref{eq:sw_tpr_kallus}) can also be written as an equivalent constraint in the form of Equation (\ref{eq:w_constraint}), as
$$g_j^{\text{TPR}}(\theta, w) = \frac{\E[h^{\text{TPR}}(\theta, X, Y) w(j |\hat{Y}, Y, \hat{G})]}{P(G=j)}$$
for some function $h^{\text{TPR}}: \Theta \times \mathcal{X} \times \mathcal{Y} \to \mathbb{R}$.

\begin{lemma}\label{lem:sw_tpr_l1}
 Denote $\hat{Y}$ as $\Ind(\phi(X;\theta) > 0)$. Let $h^{\text{TPR}}(\theta,X,Y)$ be given by

$$h^{\text{TPR}}(\theta,X,Y) = \frac{1}{2} \left(-\Ind(\hat{Y}=1, Y = 1) -  \Ind(Y=1) \left(\alpha - \frac{\E[\Ind(Y=1,\hat{Y}=1)]}{\E[\Ind(Y=1)]}  \right)\right).$$

Then \begin{align*}
    \frac{\E[\Ind(Y=1) \Ind(\hat{Y}=1)]}{\E[\Ind(Y=1)]} - \frac{ \E[ \Ind(Y=1) \Ind(\hat{Y}=1) w(j | \hat{Y}, Y, \hat{G} )] }{ \E[ \Ind(Y=1)  w(j | \hat{Y}, Y, \hat{G}) ] } - \alpha \leq 0 \\
        \iff \frac{\E[h^{\text{TPR}}(\theta, X, Y) w(j |\hat{Y}, Y, \hat{G})]}{P(G=j)} \leq 0.
\end{align*}
for all $j \in \mathcal{G}, P(G=j) > 0$.
\end{lemma}
\begin{proof}
Substituting the given function $h^{\text{TPR}}(\theta, X, Y)$, and using the fact that $P(G=j) > 0$ and $\E[\Ind(Y=1)w(j | \hat{Y}, Y, \hat{G})] \geq 0$:
\begin{align*}
    &\frac{\E[h^{\text{TPR}}(\theta, X, Y) w(j | \hat{Y}, Y, \hat{G})]}{P(G=j)} \leq 0 \\
    &\iff \E[h^{\text{TPR}}(\theta, X, Y) w(j | \hat{Y}, Y, \hat{G})] \leq 0 \\
    &\iff \E\left[\frac{1}{2} \left(-\Ind(\hat{Y}=1, Y = 1) -  \Ind(Y=1) \left(\alpha - \frac{\E[\Ind(Y=1,\hat{Y}=1)]}{\E[\Ind(Y=1)]}  \right)\right) w(j | \hat{Y}, Y, \hat{G})\right] \leq 0 \\
    &\iff -\E[\Ind(\hat{Y}=1, Y = 1)w(j | \hat{Y}, Y, \hat{G})] \\
    &\quad\quad\quad -\E\left[\Ind(Y=1) \left(\alpha - \frac{\E[\Ind(Y=1,\hat{Y}=1)]}{\E[\Ind(Y=1)]}  \right) w(j | \hat{Y}, Y, \hat{G})\right] \leq 0 \\
    &\iff -\E[\Ind(\hat{Y}=1, Y = 1)w(j | \hat{Y}, Y, \hat{G})] -\alpha\E[\Ind(Y=1)w(j | \hat{Y}, Y, \hat{G})] \\
    & \quad\quad\quad + \frac{\E[\Ind(Y=1,\hat{Y}=1)]}{\E[\Ind(Y=1)]} \E[\Ind(Y=1) w(j | \hat{Y}, Y, \hat{G})] \leq 0 \\
    &\iff \frac{\E[\Ind(Y=1,\hat{Y}=1)]}{\E[\Ind(Y=1)]} -\frac{\E[\Ind(\hat{Y}=1, Y = 1)w(j | \hat{Y}, Y, \hat{G})]}{\E[\Ind(Y=1)w(j | \hat{Y}, Y, \hat{G})]} -\alpha \leq 0
\end{align*}
\end{proof}

By similar proof, we also show in Lemma \ref{lem:sw_fpr_l1} that the constraint using $g_j^{\text{FPR}}(\theta, w)$ in Equation (\ref{eq:sw_fpr_kallus}) can also be written as an equivalent constraint in the form of Equation (\ref{eq:w_constraint}), as
$$g_j^{\text{FPR}}(\theta, w) = \frac{\E[h^{\text{FPR}}(\theta, X, Y) w(j |\hat{Y}, Y, \hat{G})]}{P(G=j)}$$
for some function $h^{\text{FPR}}: \Theta \times \mathcal{X} \times \mathcal{Y} \to \mathbb{R}$.

\begin{lemma}\label{lem:sw_fpr_l1}
 Denote $\hat{Y}$ as $\Ind(\phi(X;\theta) > 0)$. Let $h^{\text{FPR}}(\theta,X,Y)$ be given by

$$h^{\text{FPR}}(\theta,X,Y) = \frac{1}{2} \left(\Ind(\hat{Y}=1, Y = 0) -  \Ind(Y=0) \left(\alpha + \frac{\E[\Ind(Y=0,\hat{Y}=1)]}{\E[\Ind(Y=0)]}  \right)\right).$$

Then \begin{align*}
    \frac{ \E[ \Ind(Y=0) \Ind(\hat{Y}=1) w(j | \hat{Y}, Y, \hat{G} )] }{ \E[ \Ind(Y=0) w(j | \hat{Y}, Y, \hat{G}) ] } - \frac{\E[\Ind(Y=0) \Ind(\hat{Y}=1)]}{\E[\Ind(Y=0)]} - \alpha \leq 0 \\
        \iff \frac{\E[h^{\text{FPR}}(\theta, X, Y) w(j |\hat{Y}, Y, \hat{G})]}{P(G=j)} \leq 0.
\end{align*}
for all $j \in \mathcal{G}, P(G=j) > 0$.
\end{lemma}
\begin{proof}
Substituting the given function $h^{\text{FPR}}(\theta, X, Y)$, and using the fact that $P(G=j) > 0$ and $\E[\Ind(Y=0)w(j | \hat{Y}, Y, \hat{G})] \geq 0$:
\begin{align*}
    &\frac{\E[h^{\text{FPR}}(\theta, X, Y) w(j | \hat{Y}, Y, \hat{G})]}{P(G=j)} \leq 0 \\
    &\iff \E[h^{\text{FPR}}(\theta, X, Y) w(j | \hat{Y}, Y, \hat{G})] \leq 0 \\
    &\iff \E\left[\frac{1}{2} \left(\Ind(\hat{Y}=1, Y = 0) -  \Ind(Y=0) \left(\alpha + \frac{\E[\Ind(Y=0,\hat{Y}=1)]}{\E[\Ind(Y=0)]}  \right)\right) w(j | \hat{Y}, Y, \hat{G})\right] \leq 0 \\
    &\iff \E[\Ind(\hat{Y}=1, Y = 0)w(j | \hat{Y}, Y, \hat{G})] \\
    &\quad\quad\quad -\E\left[\Ind(Y=0) \left(\alpha + \frac{\E[\Ind(Y=0,\hat{Y}=1)]}{\E[\Ind(Y=0)]}  \right) w(j | \hat{Y}, Y, \hat{G})\right] \leq 0 \\
    &\iff \E[\Ind(\hat{Y}=1, Y = 0)w(j | \hat{Y}, Y, \hat{G})] -\alpha\E[\Ind(Y=0)w(j | \hat{Y}, Y, \hat{G})] \\
    & \quad\quad\quad - \frac{\E[\Ind(Y=0,\hat{Y}=1)]}{\E[\Ind(Y=0)]} \E[\Ind(Y=0) w(j | \hat{Y}, Y, \hat{G})] \leq 0 \\
    &\iff \frac{\E[\Ind(\hat{Y}=1, Y = 0)w(j | \hat{Y}, Y, \hat{G})]}{\E[\Ind(Y=0)w(j | \hat{Y}, Y, \hat{G})]} - \frac{\E[\Ind(Y=0,\hat{Y}=1)]}{\E[\Ind(Y=0)]} -\alpha \leq 0
\end{align*}
\end{proof}

\section{Optimality and feasibility for the \textit{Ideal} algorithm}
\label{app:ideal-alg-opt-feas}
\allowdisplaybreaks

\subsection{Optimality and feasibility guarantees}
We provide optimality and feasibility guarantees for Algorithm \ref{algo:ideal} and optimality guarantees for Algorithm \ref{algo:best-response-theta}.
\begin{theorem}[\textbf{Optimality and Feasibility for Algorithm \ref{algo:ideal}}]
Let $\theta^{*} \in \Theta$ be such that it satisfies the constraints $\displaystyle\max_{w\in \cW(\theta)}\, g_j(\theta^{*}, w) \,\leq\, 0,~\forall j \in \mathcal{G}$ and $f_0(\theta^{*}) \leq f(\theta)$ for every $\theta \in \Theta$ that satisfies the same constraints.
Let $0 \leq f_0(\theta) \leq B, \forall \theta \in \Theta$.
Let the space of Lagrange multipliers be defined as $\Lambda = \{\lambda \in \R_{+}^{m}\,|\, \|\lambda\|_{1} \leq R\},$ for  $R>0$.
Let $B_\lambda \,\geq\, \max_{t}\|\nabla_{\lambda}\mathcal{L}(\theta^{(t)},  \lambda^{(t)})\|_2$.  Let $\bar{\theta}$ be the stochastic classifier returned by Algorithm \ref{algo:ideal} when run for $T$ iterations, with the radius of the Lagrange multipliers $R = T^{1/4}$ and learning rate $\eta_\lambda = \frac{R}{B_{\lambda}\sqrt{T}}$ 
Then:
\[
\mathbf{E}_{\theta\sim \bar{\theta}}\left[
f(\theta)
\right] \,\leq\, f(\theta^{*}) \,+\, {\mathcal{O}}\left(\frac{1}{T^{1/4}}\right)\,+\, \rho
\]
and
\[
\mathbf{E}_{\theta\sim \bar{\theta}}\left[
\max_{w \in \cW(\theta)}\, g_j(\theta, w)
\right] \,\leq\,  {\mathcal{O}}\left(\frac{1}{T^{1/4}}\right)
+
\rho'
\]
\label{thm:ideal}
\end{theorem}
Thus for any given $\epsilon > 0$, by solving Steps 2 and 4 of Algorithm \ref{algo:ideal} to sufficiently small errors $\rho, \rho'$, and by running the algorithm for a sufficiently large number of steps $T$, we can guarantee that the returned stochastic model is $\epsilon$-optimal and $\epsilon$-feasible.
\begin{proof}
Let $\bar{\lambda} = \frac{1}{T}\sum_{t=1}^T \lambda^{(t)}$.
We will interpret the minimax problem in \eqref{eq:softweights_lagrangian} as a zero-sum between the $\theta$-player who optimizes $\cL$ over $\theta$, and the $\lambda$-player who optimizes  $\cL$ over $\lambda$.
We first bound the average regret incurred by the players over $T$ steps. The best response computation in Step 2 of Algorithm \ref{algo:ideal} gives us:
\begin{eqnarray}
\frac{1}{T}\sum_{t=1}^T\E_{\theta \sim \hat{\theta}^{(t)}}\left[\cL(\theta, \lambda^{(t)})\right] 
&\leq&
\frac{1}{T}\sum_{t=1}^T \min_{\theta \in \Theta}\, \cL(\theta, \lambda^{(t)}) \,+\, 
\epsilon
\nonumber\\
&\leq&
\min_{\theta \in \Theta}\, \frac{1}{T}\sum_{t=1}^T  \cL(\theta, \lambda^{(t)}) \,+\, 
\rho
\nonumber\\
&=&
\min_{\theta \in \Theta}\, \cL(\theta, \bar{\lambda}) \,+\, 
\rho
\nonumber\\
&\leq&
\min_{\theta \in \Theta}\,\max_{\lambda \in \Lambda}\, \cL(\theta, \lambda) \,+\, 
\rho
\nonumber\\
&\leq& f(\theta^{*}) +
\rho.
\label{eq:intermediate-1}
\end{eqnarray}
We then apply standard  gradient ascent analysis  for the projected gradient updates to $\lambda$ in Step 4 of the algorithm, and get:
\begin{eqnarray*}
\max_{\lambda \in \Lambda}\frac{1}{T}\sum_{t=1}^T\sum_{j=1}^m \lambda_j \delta^{(t)}_j
\,\geq\, \frac{1}{T}\sum_{t=1}^T\sum_{j=1}^m \lambda^{(t)}_j \delta^{(t)}_j \,-\, \cO\left(\frac{R}{\sqrt{T}}\right).
\end{eqnarray*}
We then plug the upper and lower bounds for the gradient estimates $\delta^{(t)}_j$'s from Step 3 of the Algorithm \ref{algo:ideal} into the above inequality:
\begin{align*}
&\max_{\lambda \in \Lambda}\frac{1}{T}\sum_{t=1}^T\sum_{j=1}^m \lambda_j \left(\E_{\theta \sim \hat{\theta}^{(t)}}\left[\max_{w \in \cW(\theta)} g_j(\theta, w)\right] \,+\, \rho'\right)\\
&\geq \frac{1}{T}\sum_{t=1}^T\sum_{j=1}^m \lambda^{(t)}_j \E_{\theta \sim \hat{\theta}^{(t)}}\left[\max_{w \in \cW(\theta)} g_j(\theta, w)\right] \,-\, \cO\left(\frac{R}{\sqrt{T}}\right).
\end{align*}

which further gives us:
\begin{align*}
&\max_{\lambda \in \Lambda}\left\{\sum_{j=1}^m \lambda_j \E_{\theta \sim \hat{\theta}^{(t)}}\left[\max_{w \in \cW(\theta)} g_j(\theta, w)\right] \,+\, \|\lambda\|_1\rho'\right\} \\
&\geq \sum_{j=1}^m \lambda^{(t)}_j \E_{\theta \sim \hat{\theta}^{(t)}}\left[\max_{w \in \cW(\theta)} g_j(\theta, w)\right] \,-\, \cO\left(\frac{R}{\sqrt{T}}\right).
\end{align*}
Adding $\frac{1}{T}\sum_{t=1}^T \E_{\theta \sim \hat{\theta}^{(t)}}\left[f(\theta)\right]$ to both sides of the above inequality, we finally get:
\begin{equation}
\frac{1}{T}\sum_{t=1}^T\E_{\theta \sim \hat{\theta}^{(t)}}\left[\cL(\theta, \lambda^{(t)})\right] 
\,\geq\,
\max_{\lambda \in \Lambda}
\left\{
\frac{1}{T}\sum_{t=1}^T\E_{\theta \sim \hat{\theta}^{(t)}}\left[\cL(\theta, \lambda)\right] 
\,+\,
\|\lambda\|_1\rho'
\right\}
\,-\, \cO\left(\frac{R}{\sqrt{T}}\right).
\label{eq:intermediate-2}
\end{equation}
\textbf{Optimality.} Now, substituting $\lambda = \mathbf{0}$ in \eqref{eq:intermediate-2} and combining with \eqref{eq:intermediate-1} completes the proof of the optimality guarantee:
\[
\E_{\theta \sim \bar{\theta}}\left[f(\theta)\right]
\,\leq\,
f_0(\theta^{*}) \,+\, \cO\left(\frac{R}{\sqrt{T}}\right) \,+\, 
\rho
\]
\textbf{Feasibility.} 
To show feasibility, we fix a constraint index $j \in \mathcal{G}$. Now substituting $\lambda_j = R$ and $\lambda_{j'} = 0, \forall j' \ne j$ in \eqref{eq:intermediate-2} and combining with \eqref{eq:intermediate-1} gives us:
\[
\frac{1}{T}\sum_{t=1}^T\E_{\theta \sim \hat{\theta}^{(t)}}\left[f(\theta) + R\max_{w \in \cW(\theta)}g_j(\theta, w)\right]
~\leq~
f(\theta^{*}) \,+\, \cO\left(\frac{R}{\sqrt{T}}\right) \,+\, \rho \,+\, 
R\rho'.
\]
which can be re-written as:
\begin{eqnarray*}
\E_{\theta \sim \bar{\theta}}\left[\max_{w \in \cW(\theta)} g_j(\theta, w)\right]
&\leq&
\frac{f(\theta^{*}) \,-\, \E_{\theta \sim \bar{\theta}}\left[f(\theta)\right]}{R} \,+\, \cO\left(\frac{1}{\sqrt{T}}\right) \,+\, \frac{\rho}{R} \,+\, \rho'.\\
&\leq&
\frac{B}{R} \,+\, \cO\left(\frac{1}{\sqrt{T}}\right) \,+\, \frac{\rho}{R} \,+\, \rho',
\end{eqnarray*}
which is our feasibility guarantee.
Setting $R = \cO(T^{1/4})$ then completes the proof.
\end{proof}

\subsection{Best Response over $\theta$}
We next describe our procedure for computing a best response over $\theta$ in Step 2 of Algorithm \ref{algo:ideal}. We will consider a slightly relaxed version of the  best response problem where the equality constraints in $\cW(\theta)$ are replaced with closely-approximating inequality constraints.

Recall that the constraint set $\cW(\theta)$ contains two sets of constraints \eqref{eq:W}, the total probability constraints that depend on $\theta$, and the simplex constraints that do not depend on $\theta$. So to decouple these constraint sets from $\theta$, we introduce Lagrange multipliers $\mu$ for the total probability constraints to make them a part of the objective, and obtain a nested \textit{minimax} problem over $\theta, \mu$, and $w$, where $w$ is constrained to satisfy the simplex constraints alone. We then jointly minimize the inner Lagrangian over $\theta$ and $\mu$, and perform gradient ascent updates on $w$ with projections onto the simplex constraints.
%
%
The joint-minimization over $\theta$ and $\mu$ is not necessarily convex and is solved using a minimization oracle.

\begin{algorithm}[t]
\caption{Best response on $\theta$ of Algorithm \ref{algo:ideal}}
\label{algo:best-response-theta}
\begin{algorithmic}[1]
\REQUIRE{$\lambda'$, learning rate $\eta_{\bw} > 0$, estimates of $P( G = j | \hat{G} = k)$ to specify constraints $r_{g,\hat{g}}$'s, $\kappa$}
\FOR{$q = 1, \ldots, Q$}
\STATE \textit{Best response on ($\theta$, $\bmu$)}: use an oracle to find  find ${\theta}^{(q)} \in \Theta$ and $\bmu^{(q)} \in \mathcal{M}^{m}$ such that:
        $$\ell({\theta}^{(q)},  {\bmu}^{(q)}, \bw^{(q)}; \lambda') \,\leq\, \min_{\theta \in \Theta,\, \bmu \in \mathcal{M}^{m}} \ell({\theta},  {\bmu}, {\bw}^{(q)}; \lambda') + \kappa,$$
    for a small slack $\kappa > 0$.
\STATE \textit{Ascent step on $\bw$}: 
$$w_j^{(q+1)} \gets \Pi_{\mathcal{W}_\Delta}\left(w_{j}^{(q)} + \eta_{\bw} \nabla_{w_j} \ell({\theta}^{(q)}, {\bmu}^{(q)},  {\bw}^{(q)}; \lambda')\right),$$
    where $\nabla_{w_j}\ell(\cdot)$ is a sub-gradient of $\ell$ w.r.t.\ $w_j$. 
\ENDFOR
\STATE \textbf{return }{A uniform distribution $\hat{\theta}$ over $\theta^{(1)}, \ldots, \theta^{(Q)}$}
\end{algorithmic}
\end{algorithm}

We begin by writing out the best-response problem over $\theta$ for a fixed $\lambda'$:
\begin{equation}
  \min_{\theta \in \Theta}\,\cL(\theta, \lambda') ~=~
  \min_{\theta\in\Theta} 
  f(\theta) + \sum_{j=1}^m \lambda'_j \max_{w_j \in \mathcal{W}(\theta)} g_j(\theta, w_j),
  \label{eq:nested-minmax-intermediate}
\end{equation}
where we use $w_j$ to denote the maximizer over $\cW(\theta)$ for constraint $g_j$ explicitly. 
We separate out the the simplex constraints
in $\cW(\theta)$  \eqref{eq:W} and denote them by:
\[
\cW_{\Delta} = \bigg\{
   w \in \R_{+}^{\mathcal{G} \times \{0,1\}^2\times \hat{\mathcal{G}}}
   \,\bigg|\,
  \sum_{j=1}^m w(j \mid \hat{y}, y, k) = 1,
  ~
  \forall k \in \hat{\mathcal{G}}, y, \hat{y} \in \{0,1\}
  \bigg\},
\]
where we represent each $w$ as a vector of values $w(i|\hat{y}, y, k)$ for each $j \in \mathcal{G}, \hat{y} \in \{0,1\},  y \in \{0,1\},$ and $k \in \hat{\mathcal{G}}$. 
We then relax the total probability constraints in $\cW(\theta)$ into a set of inequality constraints:
 \begin{eqnarray*}
  \label{eq:w_ltp_relaxed}
  P(G = j | \hat{G} = k) 
  \,-\, \sum_{\hat{y}, y \in \{0,1\}} w(j \mid \hat{y}, y, k) P(\hat{Y}(\theta) = \hat{y}, Y = y| \hat{G} = k) \,-\, \tau &\leq& 0\\
    \sum_{\hat{y}, y \in \{0,1\}} w(j \mid \hat{y}, y, k) P(\hat{Y}(\theta) = \hat{y}, Y = y| \hat{G} = k)  \,-\, P(G = j | \hat{G} = k) \,-\, \tau  &\leq& 0
\end{eqnarray*}
 for some small $\tau > 0$. We have a total of $U = 2\times m \times \hat{m}$ relaxed inequality constraints, and will denote each of them as
 $r_{u}(\theta, w) \leq 0,$ with index $u$ running from $1$ to $U$. Note that each $r_{u}(\theta, w)$ is linear in $w$.
 
 Introducing Lagrange multipliers $\mu$ for the relaxed total probability constraints, the optimization problem in \eqref{eq:nested-minmax-intermediate} can be re-written equivalently as:
\begin{eqnarray*}
     \min_{\theta \in \Theta}\,
    f(\theta) \,+\,
    \sum_{j=1}^m \lambda'_j 
    \max_{w_j \in \cW_{\Delta}}
    \,
    \min_{\mu_j \in \mathcal{M}}    \bigg\{
    g_j(\theta, w_j)
    -\sum_{u=1}^U \mu_{j,u}\, r_{u}(\theta, w_j)
    \bigg\}
    ,
\end{eqnarray*}
where note that each $w_j$ is maximized over only the simplex constraints $\cW_{\Delta}$ which are independent of $\theta$, and $\mathcal{M} = \{\mu_j \in \mathbb{R}_{+}^{m\times \hat{m}} \,|\, \|\mu_j\|_1 \leq R'\}$, for some constant $R' > 0$. Because each $w_j$ and $\mu_j$ appears only in the  $j$-th term in the summation, we can pull out the max and min, and equivalently rewrite the above problem as:
\begin{eqnarray}
    \min_{\theta \in \Theta}\,\max_{\bw \in \mathcal{W}^m_{\Delta}}\,\min_{\bmu \in \mathcal{M}^{m}} \underbrace{f(\theta) + 
    \sum_{j=1}^m \lambda'_j 
    \bigg(
    \underbrace{
    g_j(\theta, w_j)
    -\sum_{u=1}^U \mu_{j,u}\, r_{u}(\theta, w_j)}_{\omega(\theta, \mu_j, w_j)}
    \bigg)}_{\ell(\theta, \bmu, \bw; \lambda')},
    \label{eq:minmax-inner}
\end{eqnarray}
where $\bw = (w_1, \ldots, w_m)$ and $\bmu = (\mu_{1}, \ldots, \mu_{m})$.
We then solve this nested minimax problem in Algorithm \ref{algo:best-response-theta} by using an minimization \textit{oracle} to perform a full optimization of $\ell$ over ($\theta$, $\mu$), and carrying out gradient ascent updates on $\ell$ over $w_j$.

We now proceed to show an optimality guarantee for Algorithm \ref{algo:best-response-theta}. 
\begin{theorem}[\textbf{Optimality Guarantee for Algorithm \ref{algo:best-response-theta}}]
Suppose for every $\theta \in \Theta$, there exists a
$\widetilde{w}_j \in \cW_\Delta$ such that $r_{u}(\theta, \widetilde{w}_j)\leq -\gamma,\,\forall u \in [U]$, for some $\gamma > 0$. Let $0 \leq g_j(\theta, w_j) \leq B',\, \forall \theta \in \Theta, w_j \in \cW_\Delta$. Let $B_\bw \,\geq\, \max_{q}\|\nabla_{\bw}\,\ell(\theta^{(q)}, \bmu^{(q)}, \bw^{(q)}; \lambda'))\|_2$.  Let $\hat{\theta}$ be the stochastic classifier returned by Algorithm \ref{algo:best-response-theta} when run for a given $\lambda'$ for $Q$ iterations, with the radius of the Lagrange multipliers $R' = B'/\gamma$ and learning rate $\eta_{\bw} = \frac{R'}{B_{\bw}\sqrt{T}}$. Then:
    $$
      \E_{\theta \sim \hat{\theta}}\left[\cL(\theta, \lambda')\right] \,\leq\, \min_{\theta \in \Theta}\, \cL(\theta, \lambda') \,+\, \cO\left(\frac{1}{\sqrt{Q}}\right) + \kappa.$$
\label{thm:best-response-theta}
\end{theorem}

Before proving Theorem \ref{thm:best-response-theta}, we will find it useful to state the following lemma.
\begin{lemma}[\textbf{Boundedness of Inner Lagrange Multipliers in \eqref{eq:minmax-inner}}]
Suppose for every $\theta \in \Theta$, there exists a
$\widetilde{w}_j \in \cW$ such that $r_{u}(\theta, \widetilde{w}_j)\leq -\gamma,\,\forall u \in [U]$, for some $\gamma > 0$. Let $0 \leq g_j(\theta, w_j) \leq B',\, \forall \theta \in \Theta, w_j \in \cW_\Delta$. Let $\mathcal{M} = \{\mu_j \in \mathbb{R}_{+}^{K} \,|\, \|\mu_j\|_1 \leq R'\}$ with
the radius of the Lagrange multipliers $R' = B'/\gamma$. Then we have for all $j \in \mathcal{G}$:
$$\max_{w_j \in \cW_\Delta}\min_{\mu_j \in \cM}\,\omega\big(\theta, \mu_j, {w}_j\big) = \underset{w_j  \in \cW_\Delta:\, r_{u}(\theta, w_j)\leq 0,\, \forall u}{\text{max}}\, g_j(\theta, w_j).$$
\label{lem:bounded-mu}
\end{lemma}
\begin{proof}
For a given $j \in \mathcal{G}$, let $w^{*}_j \in \underset{w_j \in \cW_\Delta:\, r_u(\theta, w_j)\leq 0,\, \forall u}{\text{argmax}}\, g_j(\theta, w_j)$. Then:
\begin{eqnarray}
g_j(\theta, w^{*}_j)
&=&
\max_{w_j \in \mathcal{W}_\Delta}
\min_{\mu_j \in \R_{+}^{K}}    
    \,
    \omega\big(\theta, \mu_j, w_j\big),
\label{eq:maxmin-fstar}
\end{eqnarray}
where note that $\mu_j$ is minimized over all non-negative values. Since the $\omega$ is linear in both $\mu_j$ and $w_j$, we can interchange the min and max:
\begin{eqnarray*}
g_j(\theta, w^{*}_j)
&=&
\min_{\mu_{j} \in \R_{+}^{K}}
\max_{w_j \in \mathcal{W}_\Delta}
    \,
    \omega\big(\theta, \mu_j, w_j\big).
\end{eqnarray*}
 We show below that the minimizer $\mu^\ast$ in the above problem is in fact bounded and present in $\cM$. 
\begin{eqnarray*}
g_j(\theta, w^{*}_j)
&=&
\max_{w_j \in \mathcal{W}}
    \,
    \omega\big(\theta, \mu^{*}_j, w_j\big)\\
&=& 
\max_{w_j \in \mathcal{W}}
    \,
    \bigg\{
    g_j(\theta, w_j)\,-\,
    \sum_{k=1}^K \mu^{*}_{j,k}\, r_k(\theta, w_j)
    \bigg\}\\
&\geq& 
g_j(\theta, \widetilde{w}_j)\,-\,
  \|\mu^{*}_j\|_1\,\max_{k\in[K]}\, r_k(\theta, \widetilde{w}_j)\\
&\geq&
g_j(\theta, w_j) \,+\, \|\mu^{*}_j\|_1\gamma ~\geq~ \|\mu^{*}_j\|_1\gamma.
\end{eqnarray*}
We further have:
\begin{equation}
\|\mu^{*}_j\|_1 ~\leq~ g_j(\theta, w_j)/\gamma ~\leq~ B'/\gamma.
\end{equation}
Thus the minimizer $\mu^{*}_j \in \cM$. So the minimization in  \eqref{eq:maxmin-fstar} can be performed over only $\cM$, which completes the proof of the lemma.
\end{proof}
Equipped with the above result, we are now ready to prove Theorem \ref{thm:best-response-theta}.
\begin{proof}[Proof of Theorem \ref{thm:best-response-theta}]
    Let $\bar{w}_j = \frac{1}{Q}\sum_{q=1}^Q w^{(q)}_j$. The best response on $\theta$ and $\mu$ gives us:
\begin{eqnarray}
\lefteqn{
    \frac{1}{Q}\sum_{q=1}^Q \Big(
        f(\theta^{(q)}) \,+\, \sum_{j=1}^m \lambda'_j\, \omega\big(\theta^{(q)}, \mu^{(q)}_j, w^{(q)}_j\big)\Big)}\nonumber\\
    &\leq&
    \frac{1}{Q}\sum_{q=1}^Q
    \min_{\theta \in \Theta,\, \bmu \in \cM^m}\Big(
        f(\theta) \,+\, \sum_{j=1}^m \lambda'_j\,
        \omega\big(\theta, \mu_j, w^{(q)}_j\big)\Big) \,+\, \kappa
        \nonumber\\
    &=&
    \frac{1}{Q}\sum_{q=1}^Q
  \Big(\min_{\theta \in \Theta}
        f(\theta) \,+\, \sum_{j=1}^m \lambda'_j
         \min_{\mu_j \in \cM}
        \omega\big(\theta, \mu_j, w^{(q)}_j\big)\Big) \,+\, \kappa
        ~~~~\text{($j$-th summation term depends on $\mu_j$ alone)}
        \nonumber\\
    &\leq&
  \min_{\theta \in \Theta} \frac{1}{Q}\sum_{q=1}^Q
  \Big(
        f(\theta) \,+\, \sum_{j=1}^m \lambda'_j
         \min_{\mu_j \in \cM}
        \omega\big(\theta, \mu_j, w^{(q)}_j\big)\Big) \,+\, \kappa
        \nonumber\\
    &\leq&
  \min_{\theta \in \Theta}\Big\{ f(\theta) \,+\, \sum_{j=1}^m \lambda'_j
         \min_{\mu_j \in \cM}\frac{1}{Q}\sum_{q=1}^Q
        \omega\big(\theta, \mu_j, w^{(q)}_j\big)\Big\} \,+\, \kappa
        \nonumber\\
    &=&
    \min_{\theta \in \Theta}\Big\{
        f(\theta) \,+\, \sum_{j=1}^m \lambda'_{j}
        \min_{\mu_j \in \cM}\,\omega\big(\theta, \mu_j, \bar{w}_j\big)\Big\} \,+\, \kappa
        \nonumber\\
    &\leq&
    \min_{\theta \in \Theta}\Big\{
        f(\theta) \,+\, \sum_{j=1}^m \lambda'_{j}
        \max_{w_j \in \cW}\min_{\mu_j \in \cM}\,\omega\big(\theta, \mu_j, w_j\big) \Big\} \,+\, \kappa
        ~~~~\text{(by linearity of $\omega$ in $w_j$)}
        \nonumber\\
    &=&
    \min_{\theta \in \Theta}\Big\{
        f(\theta) \,+\, \sum_{j=1}^m \lambda'_{j}
            \max_{w_j:\, r_u(\theta, w_j)\leq 0,\, \forall u} g_j(\theta, w_j)
         \Big\} \,+\, \kappa
         ~~~~\text{(from Lemma \ref{lem:bounded-mu})}
        \nonumber\\
    &=&
    \min_{\theta \in \Theta}\,
    \cL(\theta, \lambda')
    \,+\, \kappa.
    \label{eq:intermediate-3}
\end{eqnarray}
Applying standard gradient ascent analysis to the gradient ascent steps on $\mathbf{w}$ (using the fact that $\omega$ is linear in $\mathbf{w}$)
\begin{eqnarray}
\lefteqn{
    \frac{1}{Q}\sum_{q=1}^Q \Big(
        f(\theta^{(q)}) \,+\, \sum_{j=1}^m \lambda'_j\, \omega\big(\theta^{(q)}, \mu^{(q)}_j, w^{(q)}_j\big)\Big)
    }\nonumber\\
        &\geq& 
            \max_{\mathbf{w} \in \cW_\Delta^m}
            \frac{1}{Q}\sum_{q=1}^Q \Big(
            f(\theta^{(q)}) \,+\, \sum_{j=1}^m \lambda'_j\, \omega\big(\theta^{(q)}, \mu^{(q)}_j, w_j\big)\Big)
            \,-\, \cO\left(\frac{1}{\sqrt{Q}}\right)
            \nonumber\\
        &=&
            \frac{1}{Q}\sum_{q=1}^Q \Big(
            f(\theta^{(q)}) \,+\, \sum_{j=1}^m \lambda'_j\,
            \max_{w_j \in \cW_\Delta}
            \omega\big(\theta^{(q)}, \mu^{(q)}_j, w_j\big)\Big)
            - \cO\left(\frac{1}{\sqrt{Q}}\right)
            ~~\text{($j$-th summation term depends on $w_j$ alone)}
            \nonumber\\
        &\geq&
            \frac{1}{Q}\sum_{q=1}^Q \Big(
            f(\theta^{(q)}) \,+\, \sum_{j=1}^m \lambda'_j\,
            \max_{w_j \in \cW_\Delta} \min_{\mu_j \in \cM}
            \omega\big(\theta^{(q)}, \mu_j, w_j\big)\Big)
            \,-\, \cO\left(\frac{1}{\sqrt{Q}}\right)
            ~~~~\text{(by linearity of $\omega$ in $w_j$ and $\mu_j$)}
            \nonumber\\
        &=&
            \E_{\theta \sim \hat{\theta}}\left[ 
            f(\theta) \,+\, \sum_{j=1}^m \lambda'_j\,
            \max_{w_j \in \cW_\Delta} \min_{\mu_j \in \cM}
            \omega\big(\theta, \mu_j, w_j\big)\right]
            \,-\, \cO\left(\frac{1}{\sqrt{Q}}\right)
            \nonumber\\
        &=&
            \E_{\theta \sim \hat{\theta}}\left[ 
            f(\theta^{(q)}) \,+\, \sum_{j=1}^m \lambda'_j\,
            \max_{w_j \in \cW_\Delta:\, r_u(\theta, w_j)\leq 0,\, \forall u} g_j(\theta, w_j)\right]
            \,-\, \cO\left(\frac{1}{\sqrt{Q}}\right)
            ~~~~\text{(from Lemma \ref{lem:bounded-mu})}
            \nonumber\\
        &=&
            \E_{\theta \sim \hat{\theta}}\left[ \cL(\theta, \lambda')\right]
            \,-\, \cO\left(\frac{1}{\sqrt{Q}}\right).
        \label{eq:intermediate-4}
\end{eqnarray}
Combining \eqref{eq:intermediate-3} and \eqref{eq:intermediate-4} completes the proof.
\end{proof}


\begin{figure}[!ht]
\vspace{-8pt}
\begin{algorithm}[H]
\caption{\textit{Practical} Algorithm}
\label{algo:heuristic}
\begin{algorithmic}[1]
\REQUIRE learning rates $\eta_\theta > 0$, $\eta_\lambda > 0$, estimates of \\ $P( G = j | \hat{G} = k)$ to specify $\mathcal{W}(\theta)$
\FOR{$t = 1, \ldots, T$}
\STATE Solve for $w$ given $\theta$ using linear programming or a gradient method:\\
$w^{(t)} \gets \max_{w \in \mathcal{W}(\theta^{(t)})} \sum_{j=1}^m \lambda_j^{(t)} g_j( \theta^{(t)}, w )$
\vskip 0.05in
%
%
\STATE \textit{Descent step on $\theta$:} \\

 $ \theta^{(t+1)} \gets \theta^{(t)} - \eta_\theta \delta_\theta^{(t)}$, where \\
$\delta_\theta^{(t)} = \nabla_\theta \left(f_0(\theta^{(t)}) + \sum_{j=1}^m \lambda_j^{(t)} g_j\left(\theta^{(t)}, w^{(t+1)}\right)\right)$
\vskip 0.05in
%
%
\STATE \textit{Ascent step on $\lambda$}:\\

$\tilde{\lambda}_j^{(t+1)} \gets \lambda^{(t)}_j + \eta_{\lambda} g_j\left(\theta^{(t+1)}, w^{(t+1)}\right) \;\; \forall j \in \mathcal{G}$\\
$\lambda^{(t+1)}  \gets \Pi_{\Lambda}(\tilde{\lambda}^{(t+1)}),$\\
%
%
\ENDFOR
\STATE \textbf{return} {$\theta^{(t^{*})}$ where $t^{*}$ denotes the \textit{best} iterate that satisfies the constraints in (\ref{eq:softweights}) with the lowest objective.}
\end{algorithmic}
\end{algorithm}
\vspace{-20pt}
\end{figure}

\section{Discussion on the \textit{Practical} algorithm}
\label{app:practical}
Here we provide the details of the \textit{practical} Algorithm \ref{algo:heuristic} to solve problem (\ref{eq:softweights_lagrangian}). We also further discuss how we arrive at Algorithm \ref{algo:heuristic}. Recall that in the minimax problem in \eqref{eq:softweights_lagrangian}, restated below, each of the $m$ constraints contain a max over $w$:
\begin{eqnarray*}
  \min_{\theta\in\Theta} \max_{\lambda \in \Lambda} 
  f(\theta) + \sum_{j=1}^m \lambda_j \max_{w \in \mathcal{W}(\theta)} g_j(\theta, w).
\end{eqnarray*}
We show below that this is equivalent to a minimax problem where the sum over $j$ and max over $w$ are swapped:
\begin{lemma}
The minimax problem in \eqref{eq:softweights_lagrangian} is equivalent to:
\begin{equation}
    \min_{\theta\in\Theta}\max_{\lambda\in\Lambda}\max_{w \in \mathcal{W}(\theta)}\,
  f(\theta) +  \sum_{j=1}^m
  \lambda_j g_j(\theta, w).
  \label{eq:softweights_max_outside}
\end{equation}
\end{lemma}
\begin{proof}
Recall that the space of Lagrange multipliers $\Lambda = \{\lambda \in \R_{+}^{m}\,|\, \|\lambda\|_{1} \leq R\},$ for  $R>0$. So the above maximization over $\Lambda$ can be re-written in terms of a maximization over the $m$-dimensional simplex $\Delta_m$ and a scalar $\beta \in [0, R]$:
\begin{eqnarray*}
\lefteqn{
  \min_{\theta\in\Theta} \max_{\beta \in [0,R],\,\nu \in \Delta_m} 
  f(\theta) + \beta\sum_{j=1}^m
  \nu_j \max_{w \in \mathcal{W}(\theta)} g_j(\theta, w)}\\
&=&
  \min_{\theta\in\Theta} \max_{\beta \in [0,R]} 
  f(\theta) + \beta\max_{\nu \in \Delta_m} \sum_{j=1}^m
  \nu_j \max_{w \in \mathcal{W}(\theta)} g_j(\theta, w)\\
 &=&
  \min_{\theta\in\Theta} \max_{\beta \in [0,R]} 
  f(\theta) + \beta\max_{j\in \mathcal{G}} \max_{w \in \mathcal{W}(\theta)} g_j(\theta, w)\\
  &=&
  \min_{\theta\in\Theta} \max_{\beta \in [0,R]} 
  f(\theta) + \beta \max_{w \in \mathcal{W}(\theta)}\max_{j\in \mathcal{G}} g_j(\theta, w)\\
 &=&
  \min_{\theta\in\Theta} \max_{\beta \in [0,R]} 
  f(\theta) + \beta \max_{w \in \mathcal{W}(\theta)}\max_{\nu \in \Delta_m} \sum_{j=1}^m
  \nu_j g_j(\theta, w)\\
&=&
  \min_{\theta\in\Theta}
  f(\theta) + \max_{\beta\in[0,R],\,\nu \in \Delta_m}\max_{w \in \mathcal{W}(\theta)} \sum_{j=1}^m
  \beta\nu_j g_j(\theta, w)\\
&=&
  \min_{\theta\in\Theta}
  f(\theta) + \max_{\lambda\in\Lambda}\max_{w \in \mathcal{W}(\theta)} \sum_{j=1}^m
  \lambda_j g_j(\theta, w),
\end{eqnarray*}
which completes the proof.
\end{proof}
The practical algorithm outlined in Algorithm \ref{algo:heuristic} seeks to solve the re-written minimax problem in \eqref{eq:softweights_max_outside}, and is similar in structure to the ideal algorithm in Algorithm \ref{algo:ideal}, in that it has two high-level steps: an approximate best response over $\theta$ and gradient ascent updates on $\lambda$. However, the algorithm works with deterministic classifiers $\theta^{(t)}$, and uses a simple heuristic to approximate the best response step. Specifically, for the best response step, the algorithm finds the maximizer of the Lagrangian over $w$ for a fixed $\theta^{(t)}$ by e.g. using linear programming:
$$w^{(t)} \gets \max_{w \in \mathcal{W}(\theta^{(t)})} \sum_{j=1}^m \lambda_j^{(t)} g_j( \theta^{(t)}, w ),$$ 
uses the maximizer $w^{(t)}$ to approximate the gradient of the Lagrangian at $\theta^{(t)}$:
$$
\delta_\theta^{(t)} = \nabla_\theta \Big(f_0(\theta^{(t)}) + \sum_{j=1}^m \lambda_j^{(t)} f_j\left(\theta^{(t)}, w^{(t+1)}\right)\Big)
$$
and performs a single gradient update on $\theta$:
 $$ \theta^{(t+1)} \gets \theta^{(t)} - \eta_\theta \delta_\theta^{(t)}.
 $$
The gradient ascent step on $\lambda$ is the same as the ideal algorithm, except that it is simpler to implement as the iterates $\theta^{(t)}$ are deterministic:
$$\tilde{\lambda}_j^{(t+1)} \gets \lambda^{(t)}_j + \eta_{\lambda} f_j\Big(\theta^{(t+1)}, w^{(t+1)}\Big) \;\; \forall j \in \mathcal{G};$$
$$\lambda^{(t+1)}  \gets \Pi_{\Lambda}(\tilde{\lambda}^{(t+1)}).$$

\section{Additional experiment details and results}\label{app:experiments}
We provide more details on the experimental setup as well as further results.
\subsection{Additional experimental setup details}\label{app:experiment_details}

This section contains further details on the experimental setup, including the datasets used and hyperparameters tuned.  All categorical features in each dataset were binarized into one-hot vectors. All numerical features were bucketized into 4 quantiles, and further binarized into one-hot vectors. All code that we used for pre-processing the datasets from their publicly-downloadable versions can be found at \url{https://github.com/wenshuoguo/robust-fairness-code}. 
  
For the na{\"i}ve approach, we solve the constrained optimization problem (\ref{eq:naiveproxy}) with respect to the noisy groups $\hat{G}$. For comparison, we also report the results of the unconstrained optimization problem and the constrained optimization problem (\ref{eq:orig_short}) when the true groups $G$ are known. For the DRO problem (\ref{eq:dro}), we estimate the bound $\gamma_j = P(\hat{G} \neq G | G = j)$ in each case study. For the soft group assignments approach, we implement the \textit{practical} algorithm (Algorithm \ref{algo:heuristic}). 

In the experiments, we replace all expectations in the objective and constraints with finite-sample empirical versions. So that the constraints will be convex and differentiable, we replace all indicator functions with hinge upper bounds, as in \citet{Davenport:2010} and \citet{Eban:2017}. We use a linear model: $\phi(X;\theta) = \theta^T X$. The noisy protected groups $\hat{G}$ are included as a feature in the model, demonstrating that conditional independence between $\hat{G}$ and the model $\phi(X;\theta)$ is not required here, unlike some prior work~\citep{Awasthi:2020}. Aside from being used to estimate the noise model $P(G = k | \hat{G} = j)$ for the soft group assignments approach\footnote{If $P(G = k | \hat{G} = j)$ is estimated from an auxiliary dataset with a different distribution than test, this could lead to generalization issues for satisfying the true group constraints on test. In our experiments, we lump those generalization issues in with any distributional differences between train and test.}, the true groups $G$ are never used in the training or validation process. 

Each dataset was split into train/validation/test sets with proportions 0.6/0.2/0.2. 
For each algorithm, we chose the \textit{best} iterate $\theta^{(t^{*})}$ out of $T$ iterates on the train set, where we define \textit{best} as the iterate that achieves the lowest objective value while satisfying all constraints. We select the hyperparameters that achieve the best performance on the validation set (details in Appendix \ref{app:experiments}).
We repeat this procedure for ten random train/validation/test splits and record the mean and standard errors for all metrics\footnote{When we report the ``maximum'' constraint violation, we use the mean and standard error of the constraint violation for the group $j$ with the maximum mean constraint violation.}. 


 \subsubsection{Adult dataset}
 
 For the first case study, we used the Adult dataset from UCI~\citep{Dua:2019}, which includes 48,842 examples. The features used were \textit{age}, \textit{workclass}, \textit{fnlwgt}, \textit{education}, \textit{education\_num}, \textit{marital\_status}, \textit{occupation}, \textit{relationship}, \textit{race}, \textit{gender}, \textit{capital\_gain}, \textit{capital\_loss}, \textit{hours\_per\_week}, and \textit{native\_country}. Detailed descriptions of what these features represent are provided by UCI \citep{Dua:2019}. The label was whether or not \textit{income\_bracket} was above \$50,000. 
The true protected groups were given by the \textit{race} feature, and we combined all examples with race other than ``white'' or ``black'' into a group of race ``other.'' When training with the noisy group labels, we did \textit{not} include the true \textit{race} as a feature in the model, but included the noisy race labels as a feature in the model instead. We set $\alpha = 0.05$ as the constraint slack.

The constraint violation that we report in Figure \ref{fig:adult_constraints} is taken over a test dataset with $n$ examples $(X_1,Y_1,G_1),...,(X_n,Y_n,G_n)$, and is given by:
$$ \max_{j \in \mathcal{G}} \quad \frac{\sum_{i=1}^n \Ind(\hat{Y}(\theta)_i=1, Y_i=1)}{\sum_{i=1}^n\Ind(Y_i=1)} - \frac{\sum_{i=1}^n \Ind(\hat{Y}(\theta)_i=1, Y_i=1, G_i=j)}{\sum_{i=1}^n\Ind(Y_i=1, G_i=j)} - \alpha,$$
where $\hat{Y}(\theta)_i = \Ind(\phi(\theta; X_i) > 0)$.

Section \ref{app:tpr_sa} shows how we specifically enforce equality of opportunity using the soft assignments approach, and Section \ref{app:tpr_dro} shows how we enforce equality of opportunity using DRO.

 \subsubsection{Credit dataset}
 For the second case study, we used default of credit card clients dataset from UCI ~\cite{Dua:2019} collected by a company in Taiwan \cite{Yeh:2009}, which contains 30000 examples and 24 features. The features used were \textit{amount\_of\_the\_given\_credit}, \textit{gender}, \textit{education}, \textit{education}, \textit{marital\_status},  \textit{age}, \textit{history\_of\_past\_payment}, \textit{amount\_of\_bill\_statement}, \textit{amount\_of\_previous\_payment}. Detailed descriptions of what these features represent are provided by UCI \citep{Dua:2019}. The label was whether or not \textit{default} was true. 
The true protected groups were given by the \textit{education} feature, and we combined all examples with education level other than ``graduate school'' or ``university'' into a group of education level ``high school and others''. When training with the noisy group labels, we did \textit{not} include the true \textit{education} as a feature in the model, but included the noisy education level labels as a feature in the model instead. We set $\alpha = 0.03$ as the constraint slack.

The constraint violation that we report in Figure \ref{fig:adult_constraints} is taken over a test dataset with $n$ examples $(X_1,Y_1,G_1),...,(X_n,Y_n,G_n)$, and is given by:
$$ \max_{j \in \mathcal{G}} \quad \max( \Delta_j^{\text{TPR}}, \Delta_j^{\text{FPR}})$$
where 
$$\Delta_j^{\text{TPR}} = \frac{\sum_{i=1}^n \Ind(\hat{Y}(\theta)_i=1, Y_i=1)}{\sum_{i=1}^n\Ind(Y_i=1)} - \frac{\sum_{i=1}^n \Ind(\hat{Y}(\theta)_i=1, Y_i=1, G_i=j)}{\sum_{i=1}^n\Ind(Y_i=1, G_i=j)} - \alpha$$
and 
$$\Delta_j^{\text{FPR}} =  \frac{\sum_{i=1}^n \Ind(\hat{Y}(\theta)_i=1, Y_i=0, G_i=j)}{\sum_{i=1}^n\Ind(Y_i=0, G_i=j)} - \frac{\sum_{i=1}^n \Ind(\hat{Y}(\theta)_i=1, Y_i=0)}{\sum_{i=1}^n\Ind(Y_i=0)} - \alpha$$
and $\hat{Y}(\theta)_i = \Ind(\phi(\theta; X_i) > 0)$.

Section \ref{app:tpr_sa} shows how we specifically enforce equalized odds using the soft assignments approach, and Section \ref{app:tpr_dro} shows how we enforce equalized odds using DRO.

\subsubsection{Optimization code}
For all case studies, we performed experiments comparing the na{\"i}ve approach, the DRO approach (Section \ref{sec:dro}) and the soft group assignments approach (Section \ref{sec:softweights}). We also compared these to the baselines of optimizing without constraints and optimizing with constraints with respect to the true groups.
All optimization code was written in Python and TensorFlow \footnote{Abadi, M. et al. TensorFlow: Large-scale machine learning on heterogeneous systems,
2015. tensorflow.org.}. All gradient steps were implemented using TensorFlow's Adam optimizer \footnote{\url{https://www.tensorflow.org/api_docs/python/tf/compat/v1/train/AdamOptimizer}}, though all experiments can also be reproduced using simple gradient descent without momentum. We computed full gradients over all datasets, but minibatching can also be used for very large datasets. Implementations for all approaches are included in the attached code. Training time was less than 10 minutes per model.

\begin{table}[!ht]
\caption{Hyperparameters tuned for each approach}
\label{table: hparams}
\vskip 0.15in
\begin{center}
\begin{small}
\begin{sc}
\begin{tabular}{llll}
\toprule
Hparam & Values tried & Relevant approaches & Description \\
\midrule
$\eta_\theta$ & \{0.001,0.01,0.1\} & all approaches & learning rate for $\theta$ \\
$\eta_\lambda$ & \{0.25,0.5,1.0,2.0\} & all except unconstrained& learning rate for $\lambda$\\
$\eta_{\tilde{p}_j}$    & \{0.001, 0.01, 0.1\} & DRO & learning rate for $\tilde{p}_j$ \\
$\eta_w$    & \{0.001, 0.01, 0.1\} & soft assignments & learning rate using \\
&&&gradient methods for $w$ \\
\bottomrule
\end{tabular}
\end{sc}
\end{small}
\end{center}
\vskip -0.1in
\end{table}

\subsubsection{Hyperparameters}
The hyperparameters for each approach were chosen to achieve the best performance on the validation set on average over 10 random train/validation/test splits, where ``best'' is defined as the set of hyperparameters that achieved the lowest error rate while satisfying all constraints relevant to the approach. The final hyperparameter values selected for each method were neither the largest nor smallest of all values tried. A list of all hyperparameters tuned and the values tried is given in Table \ref{table: hparams}. 

For the na{\"i}ve approach, the constraints used when selecting the hyperparameter values on the validation set were the constraints with respect to the noisy group labels given in Equation (\ref{eq:naiveproxy}). For the DRO approach and the soft group assignments approach, the respective robust constraints were used when selecting hyperparameter values on the validation set. Specifically, for the DRO approach, the constraints used were those defined in Equation (\ref{eq:dro}), and for the soft group assignments approach, the constraints used were those defined in Equation (\ref{eq:softweights}). For the unconstrained baseline, no constraints were taken into account when selecting the best hyperparameter values. For the baseline constrained with access to the true group labels, the true group constraints were used when selecting the best hyperparameter values.

Hinge relaxations of all constraints were used during training to achieve convexity. Since the hinge relaxation is an upper bound on the real constraints, the hinge-relaxed constraints may require some additional slack to maintain feasibility. This positive slack $\beta$ was added to the original slack $\alpha$ when training with the hinge-relaxed constraints, and the amount of slack $\beta$ was chosen so that the relevant hinge-relaxed constraints were satisfied on the training set.

All approaches ran for 750 iterations over the full dataset.
 
\subsection{Additional experiment results}\label{app:experiment_results}

This section provides additional experiment results. All results reported here and in the main paper are on the test set (averaged over 10 random train/validation/test splits).

\subsubsection{Case study 1 (Adult)}
This section provides additional experiment results for case study 1 on the Adult dataset. 

Figure \ref{fig:adult_proxy_constraint_volations} that the na{\"i}ve approach, DRO approach, and soft assignments approaches all satisfied the fairness constraints for the noisy groups on the test set. 

Figure \ref{fig:adult_robust_constraint_volations} confirms that the DRO approach and the soft assignments approaches both managed to satisfy their respective robust constraints on the test set on average. For the DRO approach, the constraints measured in Figure \ref{fig:adult_robust_constraint_volations} come from Equation (\ref{eq:dro}), and for the soft assignments approach, the constraints measured in Figure \ref{fig:adult_robust_constraint_volations} come from Equation (\ref{eq:softweights}). We provide the exact error rate values and maximum violations on the true groups for the Adult dataset in Table \ref{table:adult_exact_vals}.

\begin{figure*}[!ht]
\vskip 0.2in
\begin{center}
\begin{tabular}{ccc}
\includegraphics[width=0.3\textwidth]{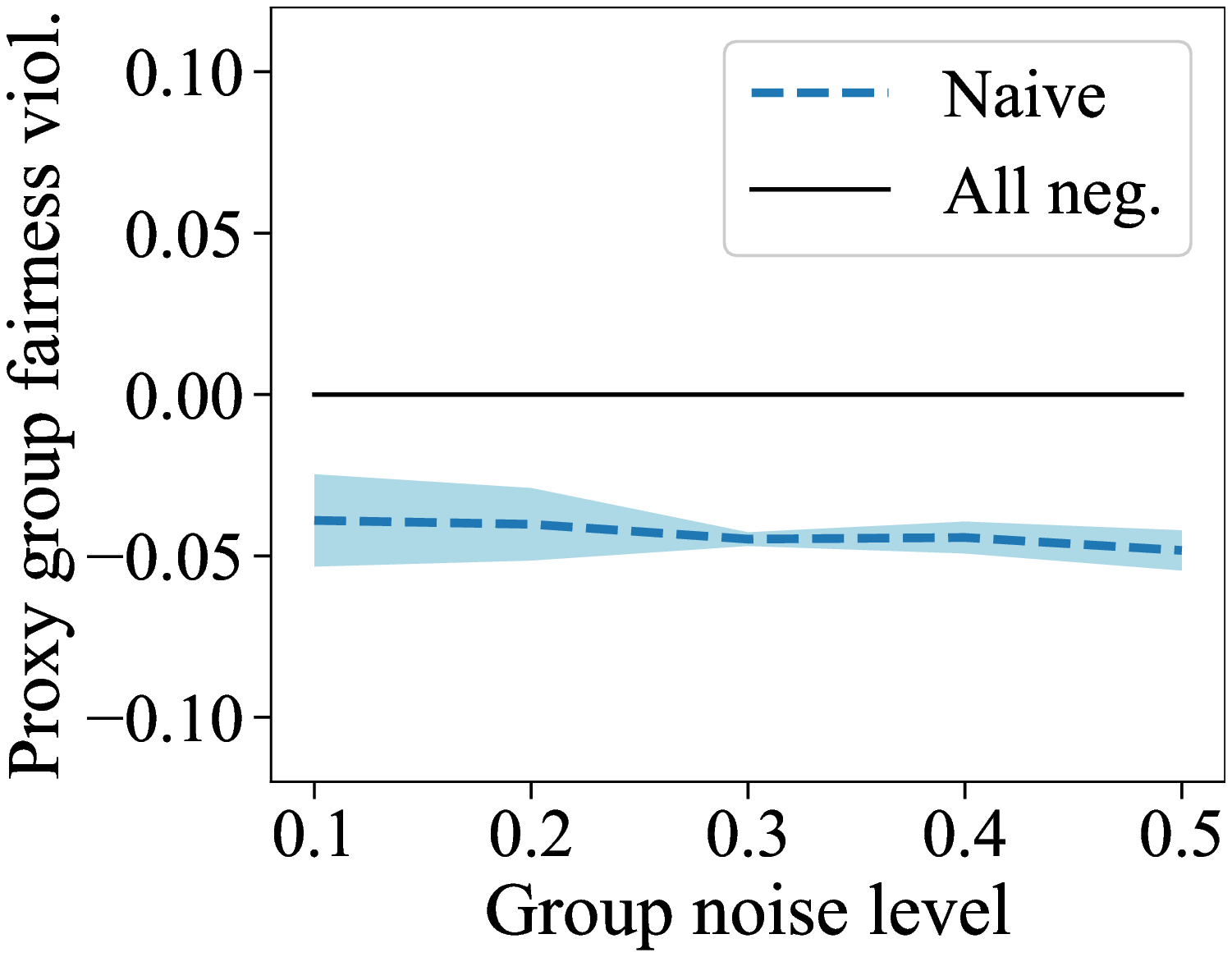} & \includegraphics[width=0.3\textwidth]{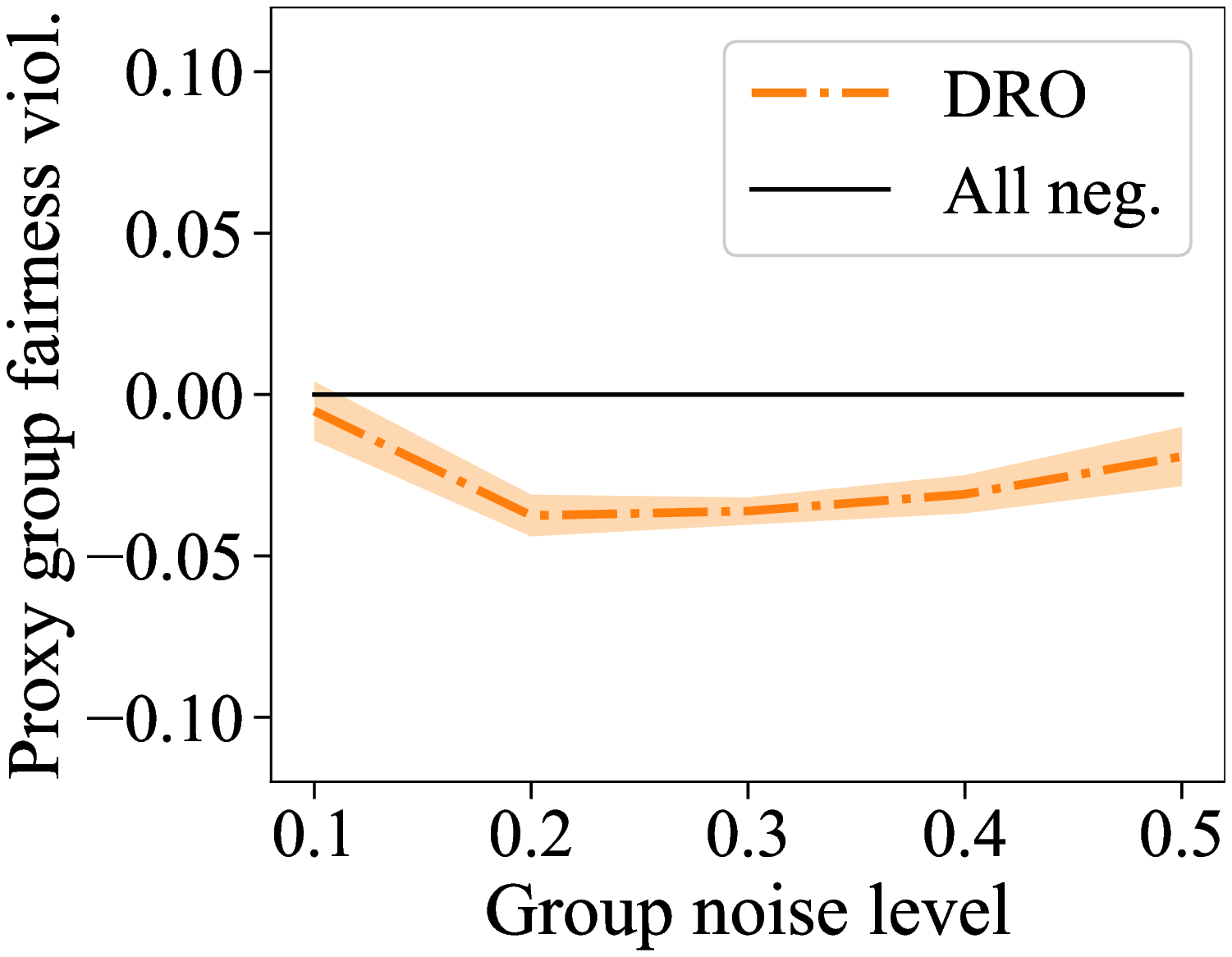} & \includegraphics[width=0.3\textwidth]{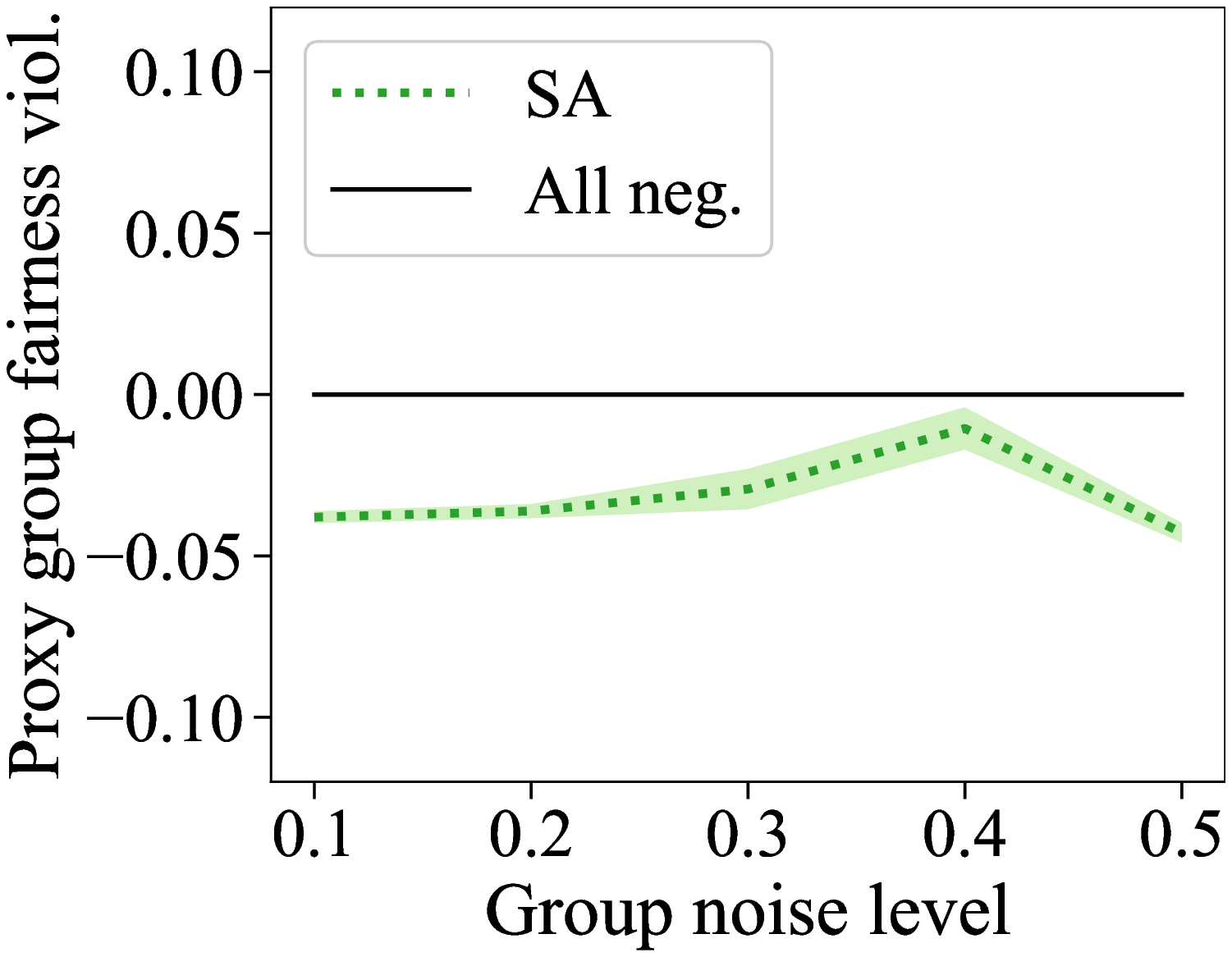} \\
\end{tabular}
\caption{Maximum fairness constraint violations with respect to the noisy groups $\hat{G}$ on the test set for different group noise levels $\gamma$ on the Adult dataset. For each noise level, we plot the mean and standard error over 10 random train/val/test splits. The black solid line illustrates a maximum constraint violation of 0. While the na{\"i}ve approach (\textit{left}) has increasingly higher fairness constraints with respect to the true groups as the noise increases, it always manages to satisfy the constraints with respect to the noisy groups $\hat{G}$}
\label{fig:adult_proxy_constraint_volations} 
\end{center}
\vskip -0.2in
\end{figure*}

\begin{figure*}[!ht]
\vskip 0.2in
\begin{center}
\begin{tabular}{cc} \includegraphics[width=0.4\textwidth]{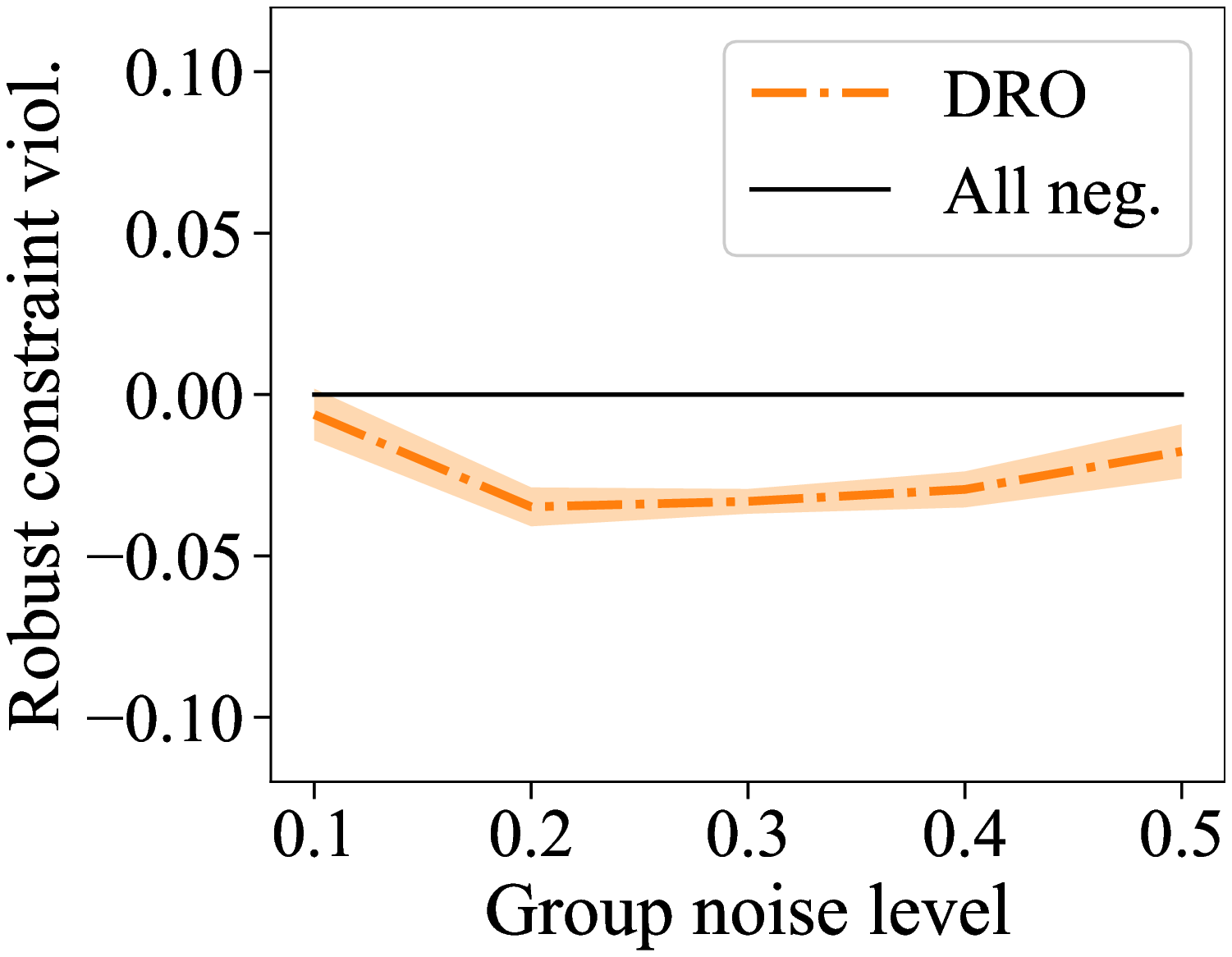} & \includegraphics[width=0.4\textwidth]{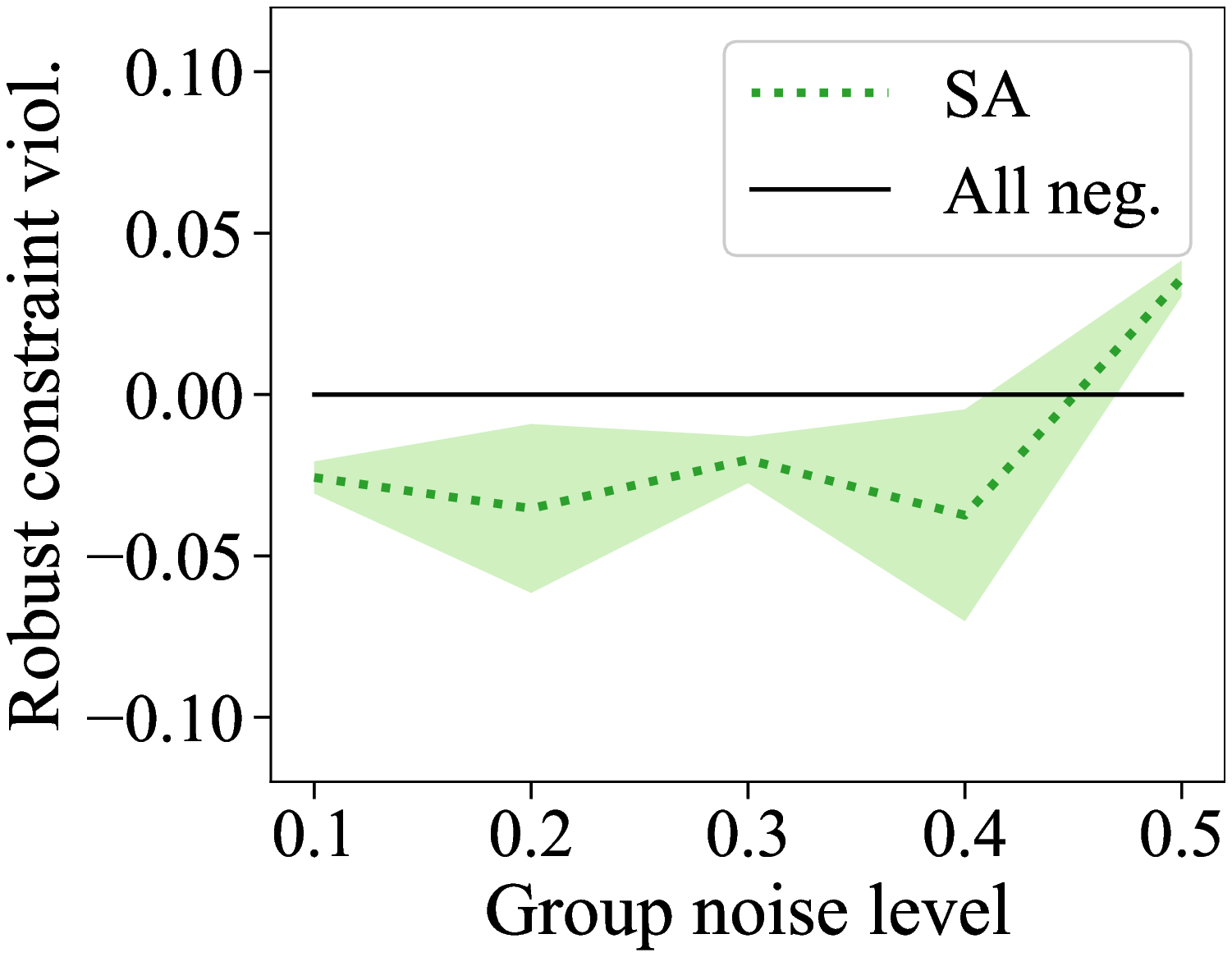} \\
\end{tabular}
\caption{Maximum robust constraint violations on the test set for different group noise levels $P(\hat{G} \neq G)$ on the Adult dataset. For each noise level, we plot the mean and standard error over 10 random train/val/test splits. The black dotted line illustrates a maximum constraint violation of 0. Both the DRO approach (\textit{left}) and the soft group assignments approach (\textit{right}) managed to satisfy their respective robust constraints on the test set on average for all noise levels.}
\label{fig:adult_robust_constraint_volations} 
\end{center}
\vskip -0.2in
\end{figure*}

\begin{table}[!ht]
\caption{Error rate and fairness constraint violations on the true groups for the Adult dataset (mean and standard error over 10 train/test/splits).}
\label{table:adult_exact_vals}
\begin{center}
\begin{small}
\begin{tabular}{{ p{0.80cm}|p{2cm}p{2cm}|p{2cm}p{2cm}|p{2cm}p{2cm}}}
\toprule 
&\multicolumn{2}{c|}{DRO} &\multicolumn{2}{c|}{Soft Assignments} \\
Noise & Error rate & Max $G$ Viol.  & Error rate & Max $G$ Viol.  \\

\midrule \midrule
0.1 & 0.152 $\pm$  0.001 & 0.002 $\pm$  0.019 & 0.148 $\pm$  0.001 & -0.048 $\pm$  0.002 \\

0.2 & 0.200 $\pm$  0.002 & -0.045 $\pm$  0.003 & 0.157 $\pm$ 0.003 & -0.048 $\pm$  0.002 \\

0.3 &0.216 $\pm$ 0.010 & -0.044 $\pm$  0.004& 0.158 $\pm$ 0.005 & 0.002 $\pm$  0.030 \\

0.4 & 0.209 $\pm$  0.006 & -0.019 $\pm$ 0.031 & 0.188 $\pm$  0.003 & -0.016 $\pm$  0.016 \\

0.5 & 0.219 $\pm$  0.012& -0.030 $\pm$  0.032 & 0.218 $\pm$  0.002 & 0.004 $\pm$  0.006 \\
\bottomrule
\end{tabular}
\end{small}
\end{center}
\vskip -5pt
\end{table}

\subsubsection{Case study 2 (Credit)}


This section provides additional experiment results for case study 2 on the Credit dataset. 

Figure \ref{fig:credit_constraints_tpr_fpr} shows the constraint violations with respect to the true groups on test separated into TPR violations and FPR violations. For all noise levels, there were higher TPR violations than FPR violations. However, this does not mean that the FPR constraint was meaningless -- the FPR constraint still ensured that the TPR constraints weren't satisfied by simply adding false positives. 

Figure \ref{fig:credit_proxy_constraint_volations} confirms that the na{\"i}ve approach, DRO approach, and soft assignments approaches all satisfied the fairness constraints for the noisy groups on the test set. 

Figure \ref{fig:credit_robust_constraint_volations} confirms that the DRO approach and the soft assignments approaches both managed to satisfy their respective robust constraints on the test set on average. For the DRO approach, the constraints measured in Figure \ref{fig:credit_robust_constraint_volations} come from Equation (\ref{eq:dro}), and for the soft assignments approach, the constraints measured in Figure \ref{fig:credit_robust_constraint_volations} come from Equation (\ref{eq:softweights}). 

We provide the exact error rate values and maximum violations on the true groups for the Credit dataset in Table \ref{table:credit_exact_vals}.

\begin{figure}[!ht]
\begin{center}
\centerline{\begin{tabular}{ccc}
 \includegraphics[width=0.31\textwidth]{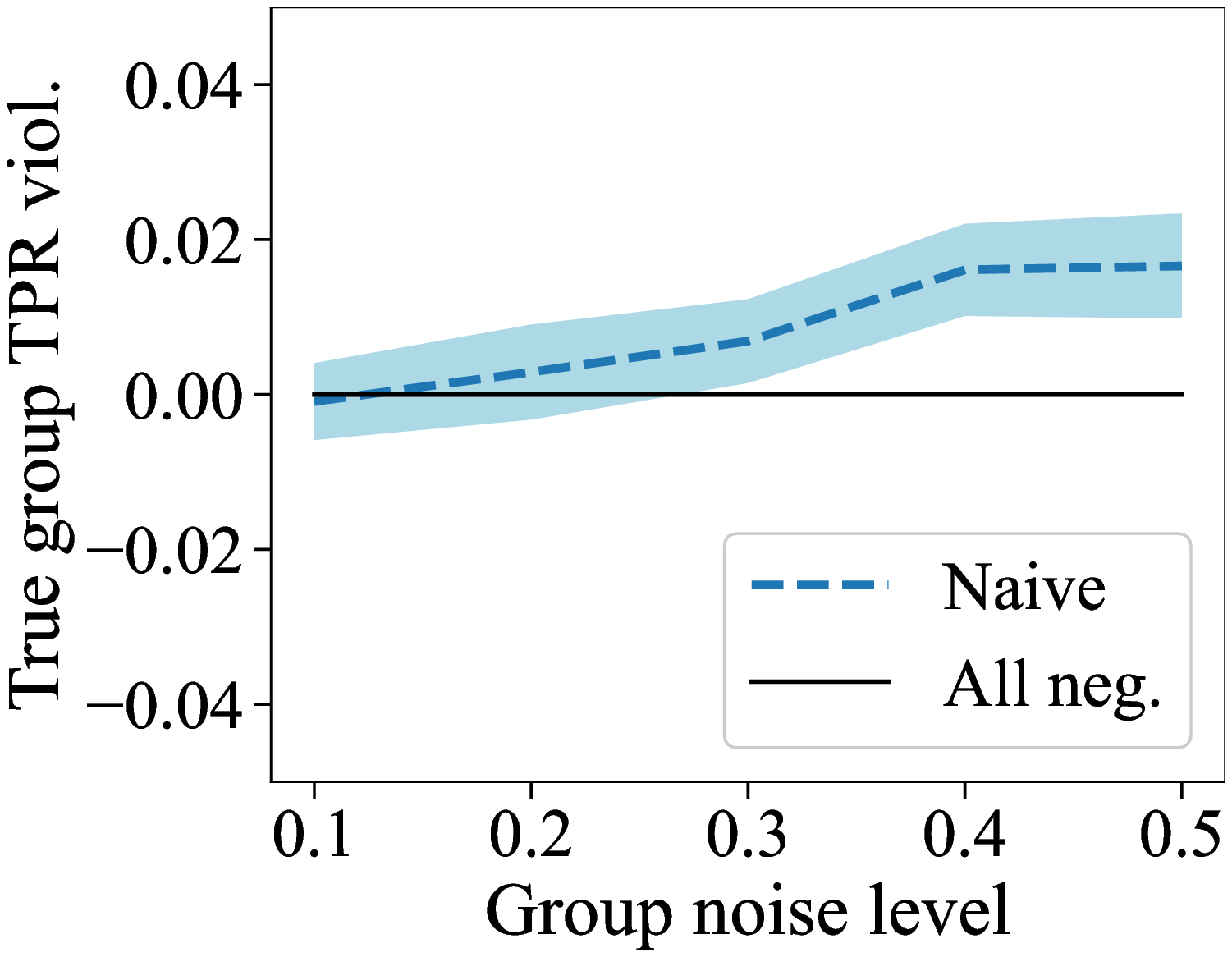} &
  \includegraphics[width=0.31\textwidth]{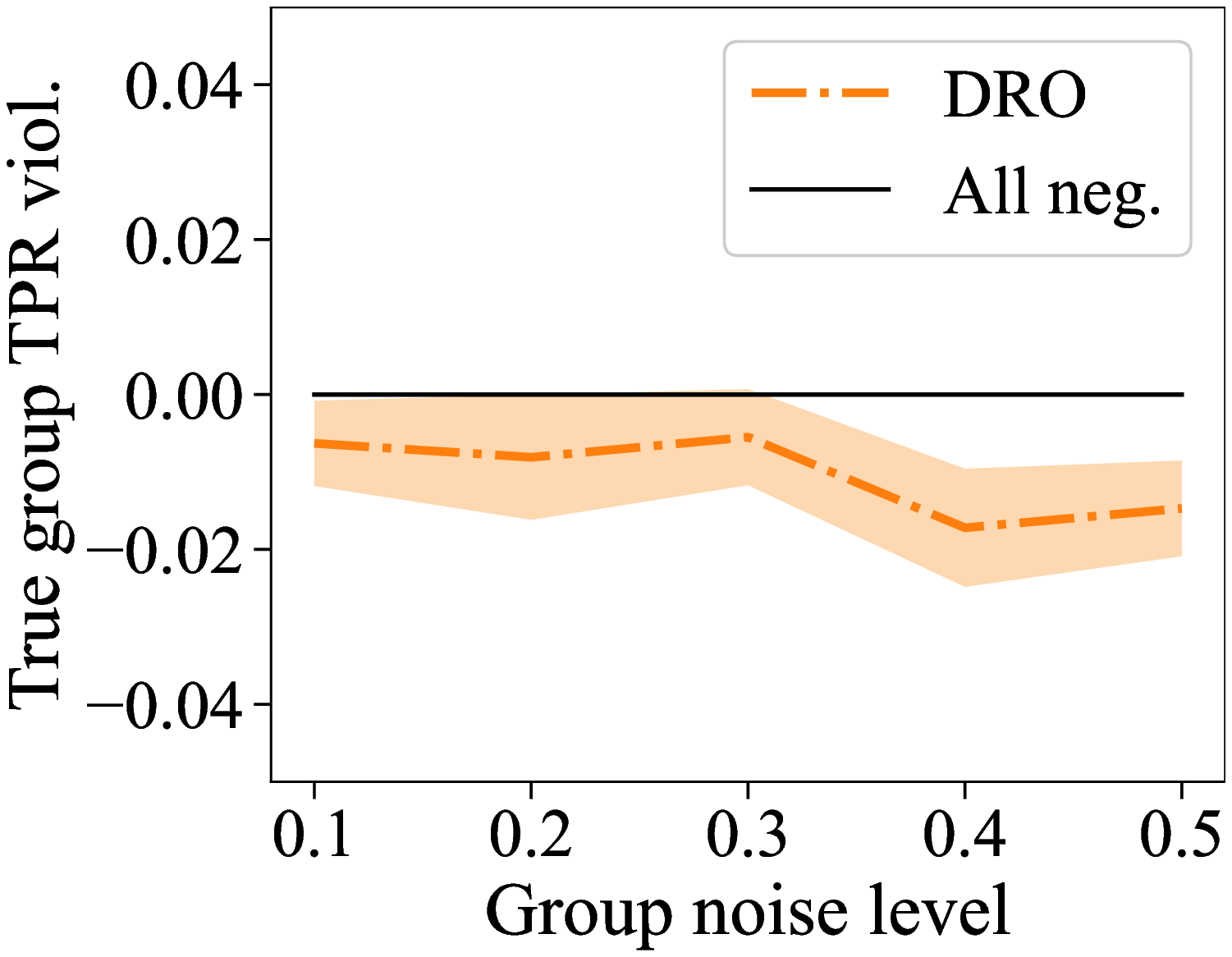} & \includegraphics[width=0.31\textwidth]{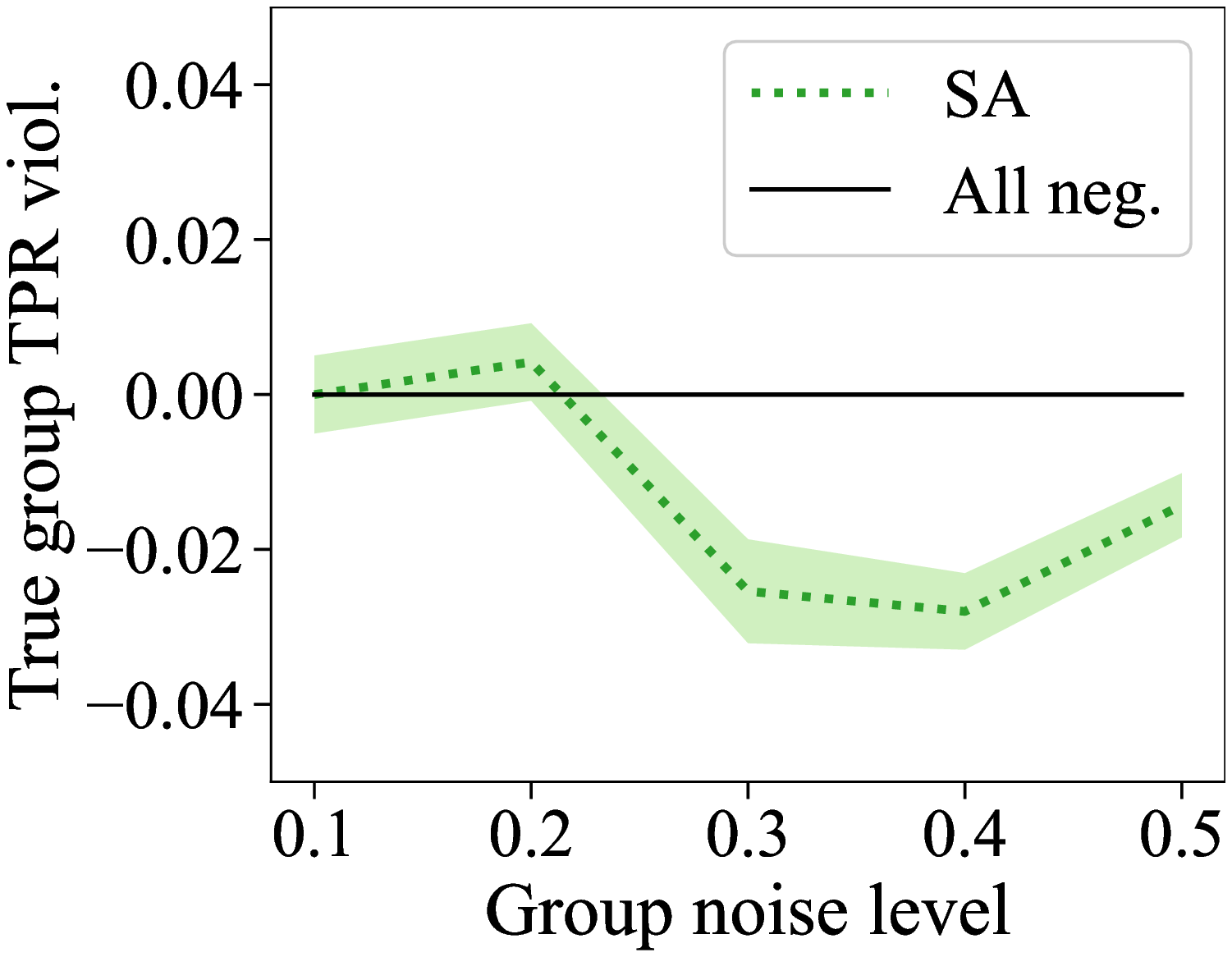} \\
  \includegraphics[width=0.31\textwidth]{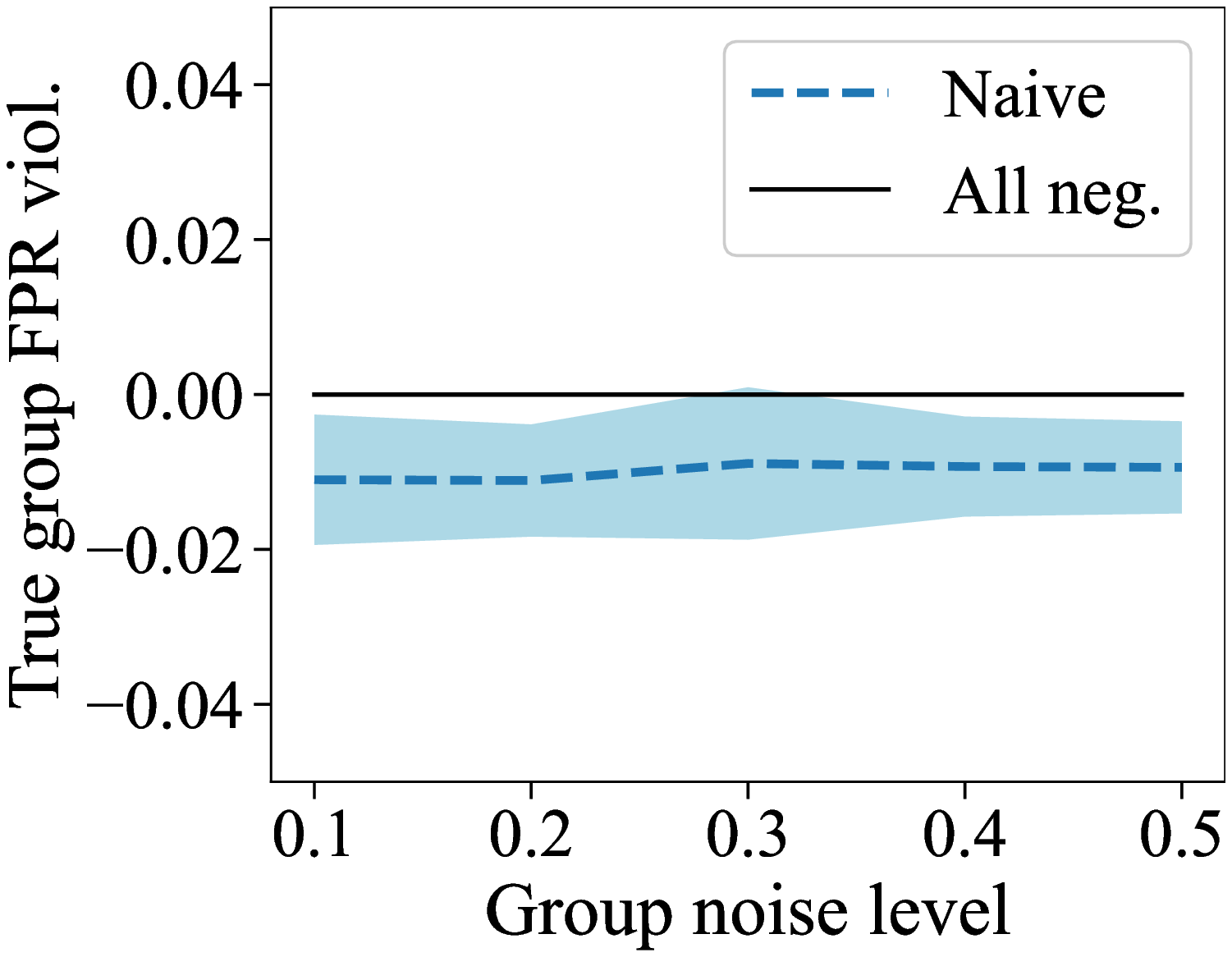} &
  \includegraphics[width=0.31\textwidth]{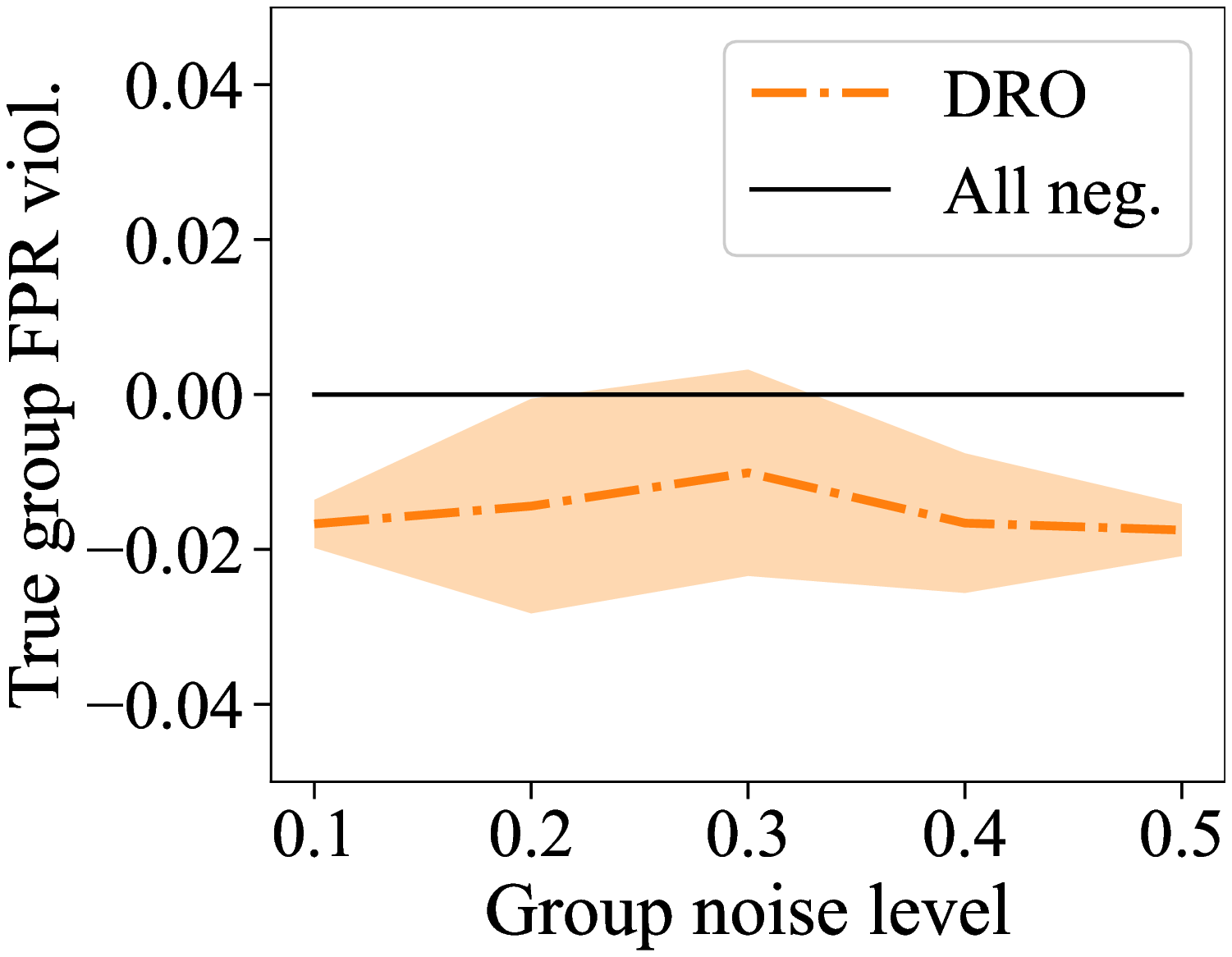} & \includegraphics[width=0.31\textwidth]{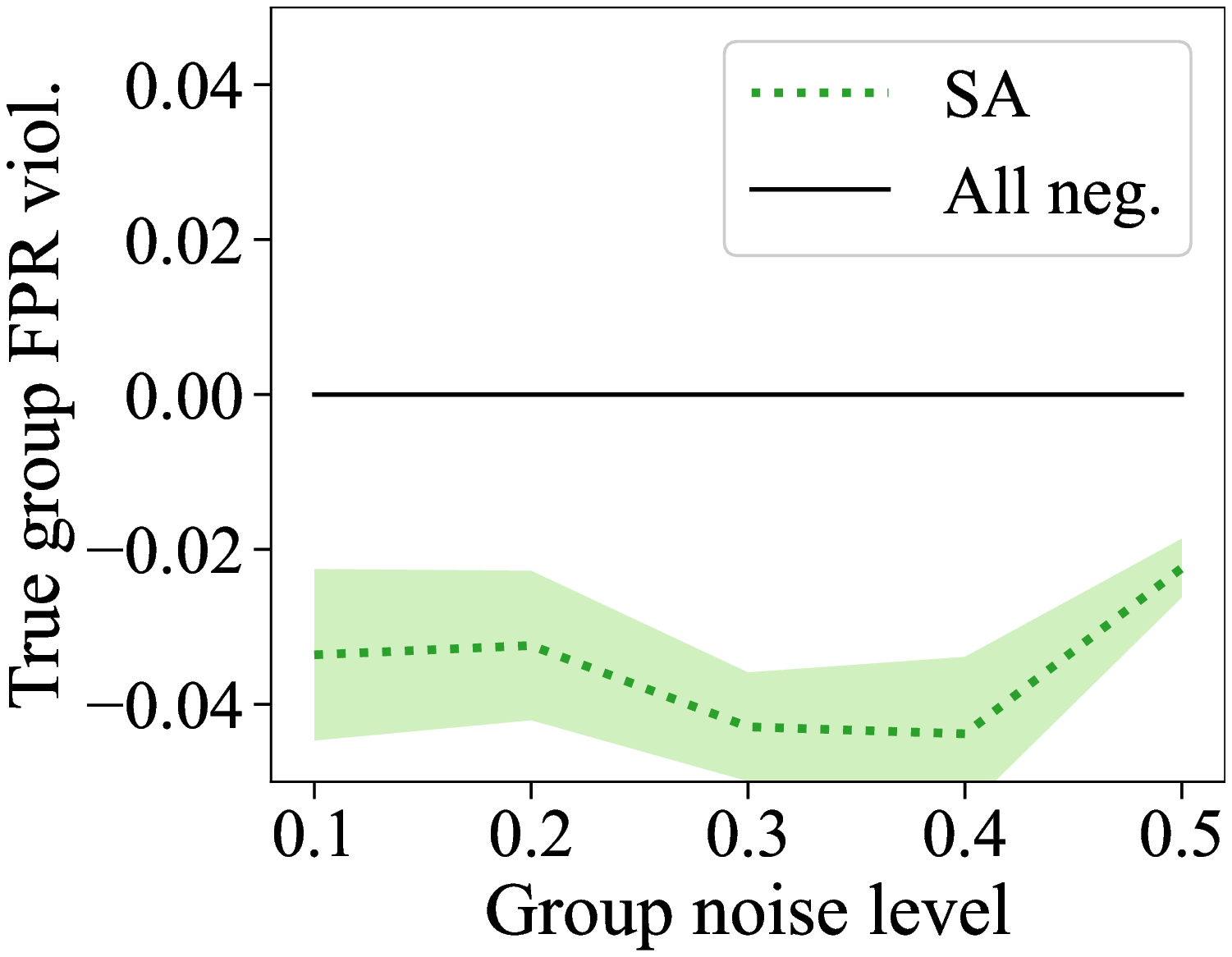} \\
\end{tabular}}
\caption{Case study 2 (Credit): Maximum true group TPR (top) and FPR (bottom) constraint violations for the Naive, DRO, and soft assignments (SA) approaches on test set for different group noise levels $\gamma$ on the Credit dataset (mean and standard error over 10 train/val/test splits). The black solid line represents the performance of the trivial ``all negatives'' classifier, which has constraint violations of 0. A negative violation indicates satisfaction of the fairness constraints on the true groups.}
\label{fig:credit_constraints_tpr_fpr} 
\end{center}
\end{figure}

\begin{figure*}[!ht]
\vskip 0.2in
\begin{center}
\begin{tabular}{ccc}
\includegraphics[width=0.3\textwidth]{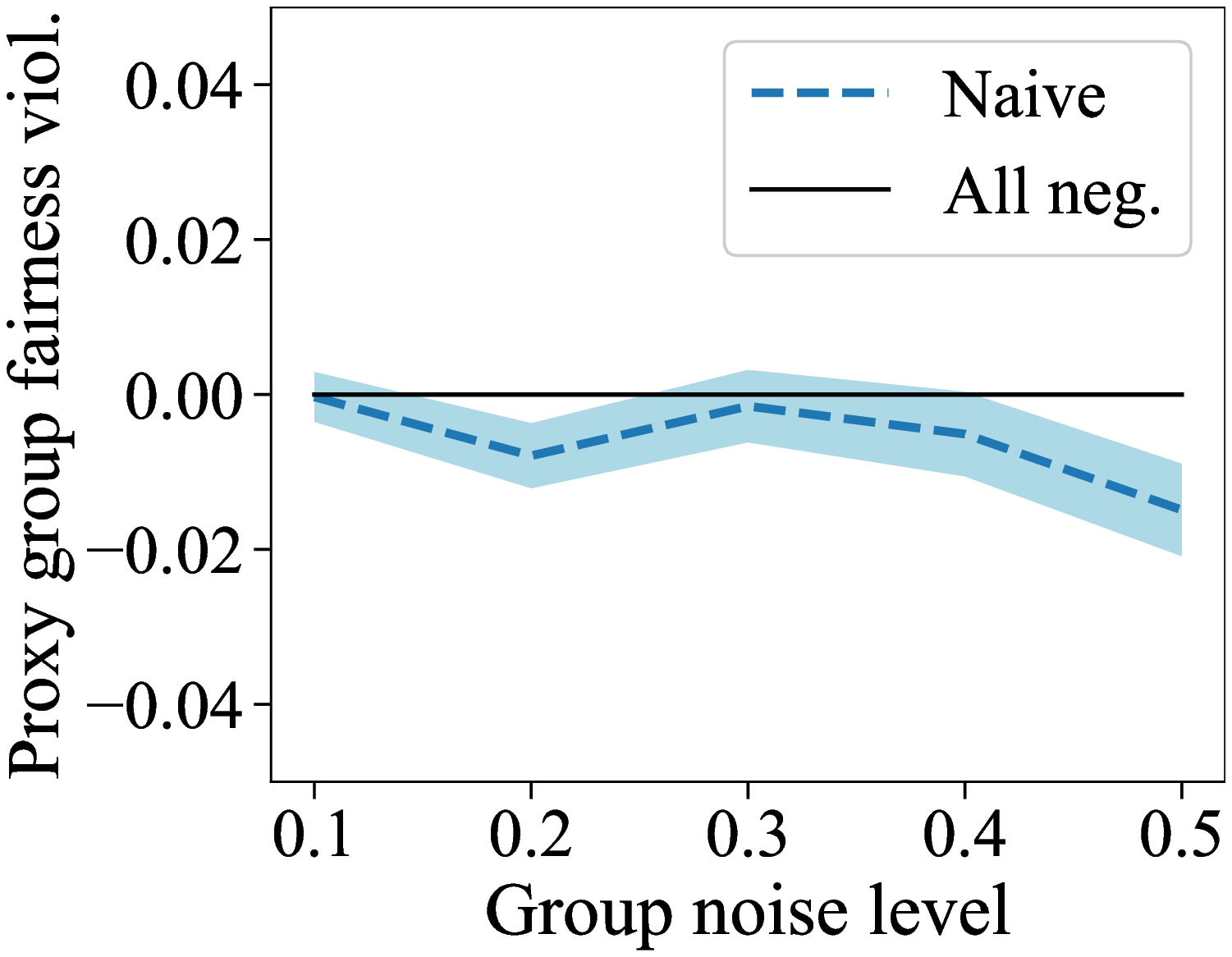} & \includegraphics[width=0.3\textwidth]{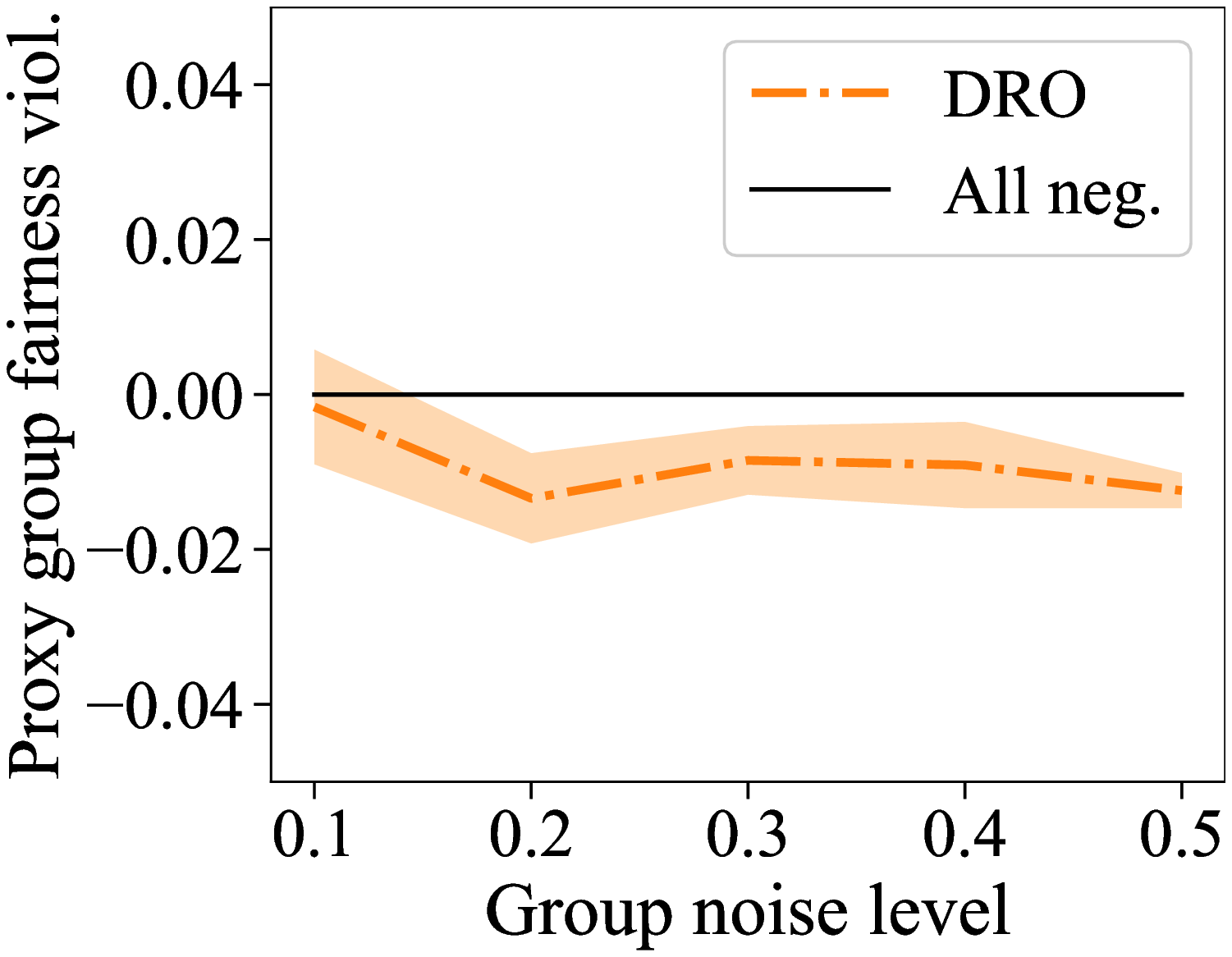} & \includegraphics[width=0.3\textwidth]{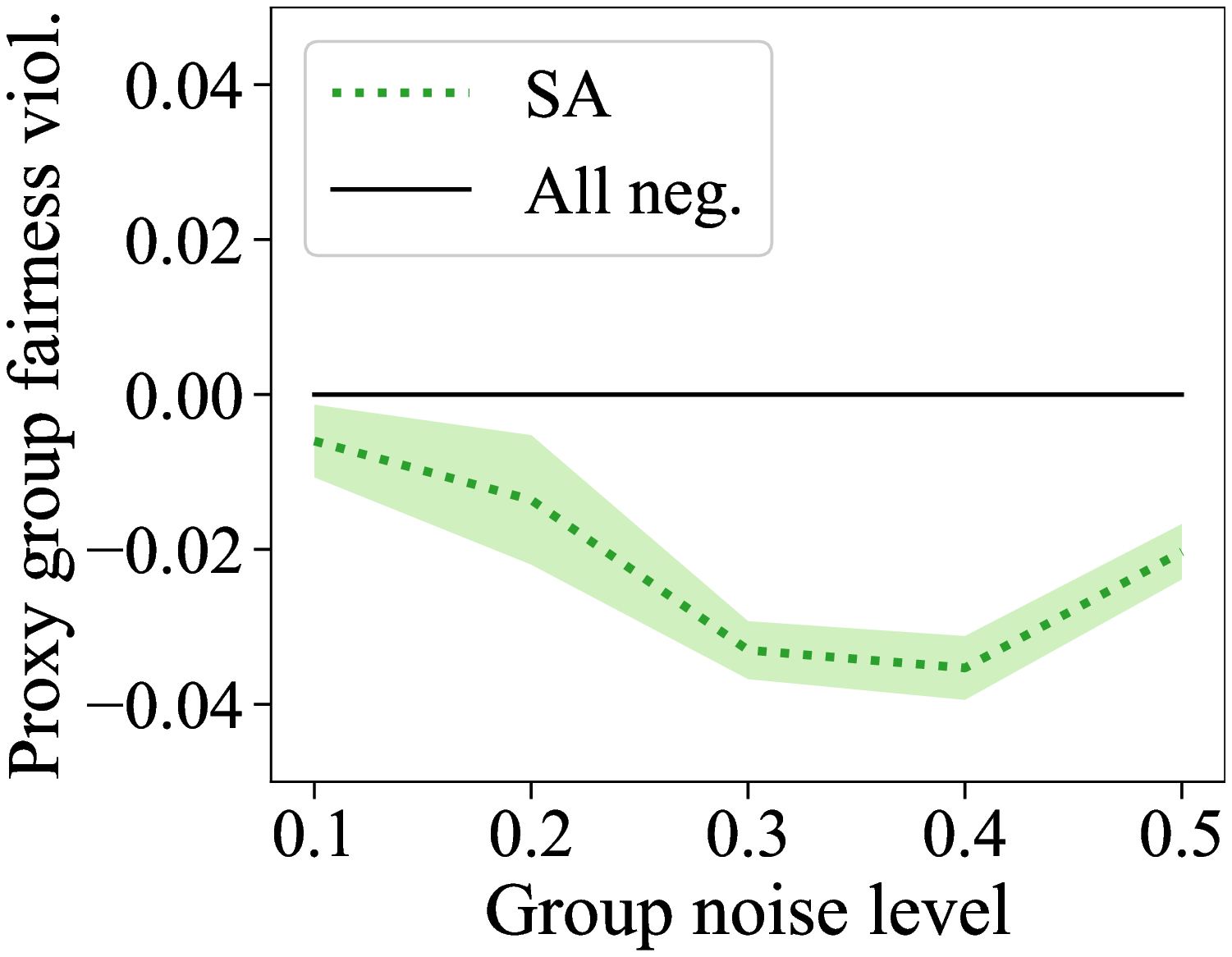} \\
\end{tabular}
\caption{Maximum fairness constraint violations with respect to the noisy groups $\hat{G}$ on the test set for different group noise levels $\gamma$ on the Credit dataset. For each noise level, we plot the mean and standard error over 10 random train/val/test splits. The black solid line illustrates a maximum constraint violation of 0. While the na{\"i}ve approach (\textit{left}) has increasingly higher fairness constraints with respect to the true groups as the noise increases, it always manages to satisfy the constraints with respect to the noisy groups $\hat{G}$}
\label{fig:credit_proxy_constraint_volations} 
\end{center}
\vskip -0.2in
\end{figure*}

\begin{figure*}[!ht]
\vskip 0.2in
\begin{center}
\begin{tabular}{cc} \includegraphics[width=0.4\textwidth]{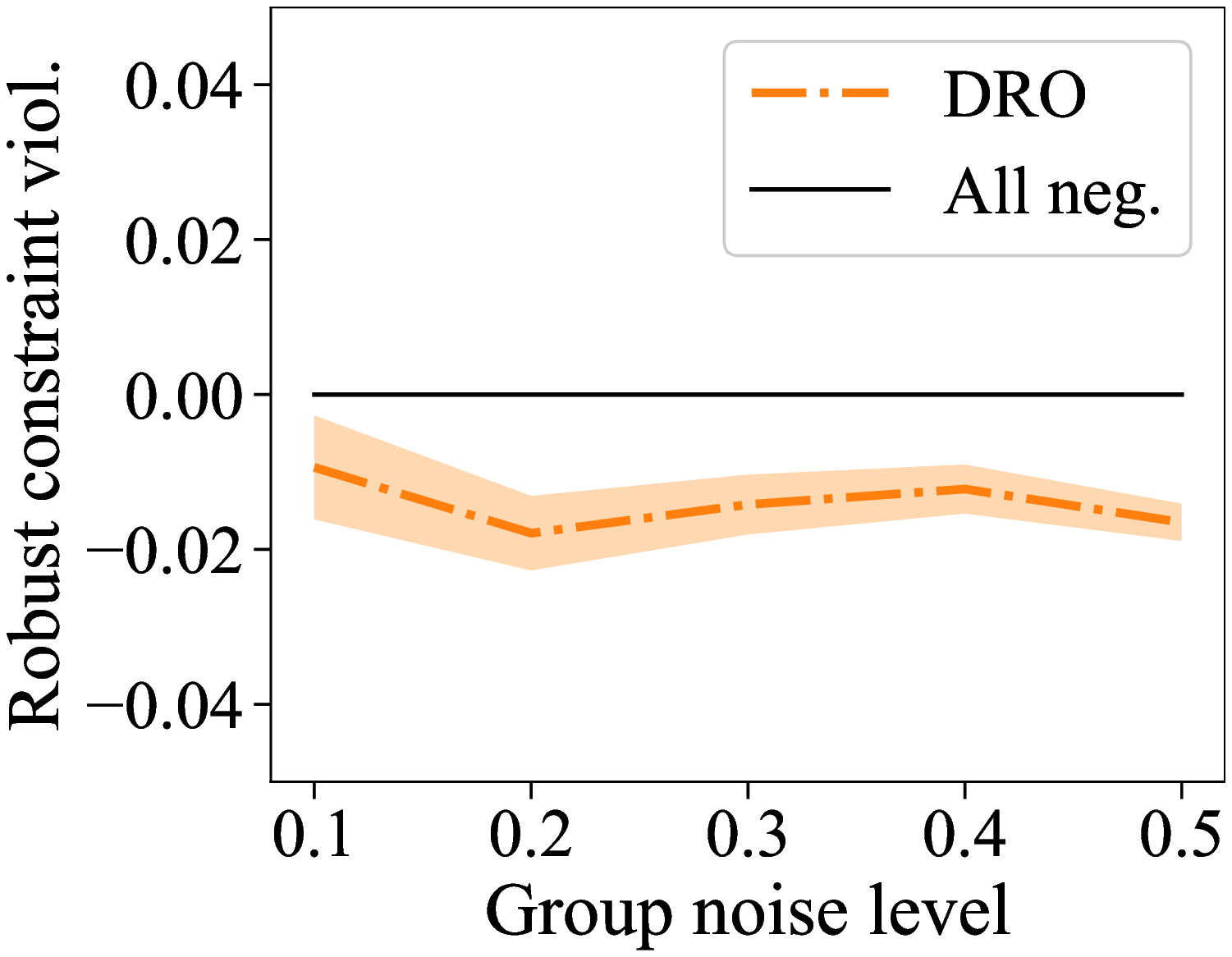} & \includegraphics[width=0.4\textwidth]{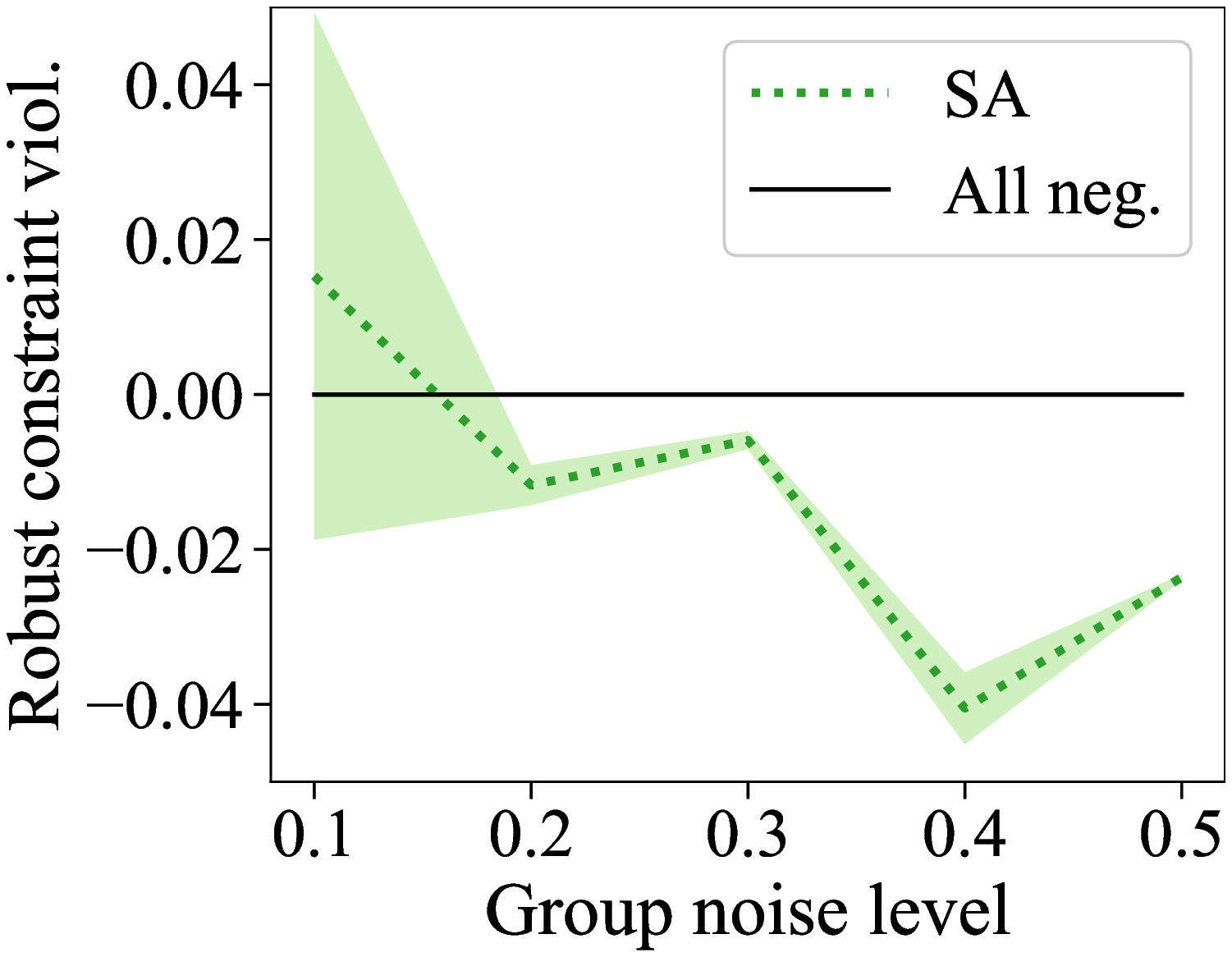} \\
\end{tabular}
\caption{Maximum robust constraint violations on the test set for different group noise levels $P(\hat{G} \neq G)$ on the Credit dataset. For each noise level, we plot the mean and standard error over 10 random train/val/test splits. The black dotted line illustrates a maximum constraint violation of 0. Both the DRO approach (\textit{left}) and the soft group assignments approach (\textit{right}) managed to satisfy their respective robust constraints on the test set on average for all noise levels.}
\label{fig:credit_robust_constraint_volations} 
\end{center}
\vskip -0.2in
\end{figure*}

\begin{table}[!ht]
\caption{Error rate and fairness constraint violations on the true groups for the Credit dataset (mean and standard error over 10 train/test/splits).}
\label{table:credit_exact_vals}
\begin{center}
\begin{small}
\begin{tabular}{{ p{0.80cm}|p{2cm}p{2cm}|p{2cm}p{2cm}|p{2cm}p{2cm}}}
\toprule 
&\multicolumn{2}{c|}{DRO} &\multicolumn{2}{c|}{Soft Assignments} \\
Noise & Error rate & Max $G$ Viol.  & Error rate & Max $G$ Viol.  \\

\midrule \midrule
0.1 & 0.206 $\pm$  0.003 & -0.006 $\pm$  0.006 & 0.182 $\pm$  0.002 & 0.000 $\pm$  0.005 \\

0.2 & 0.209 $\pm$  0.002 & -0.008 $\pm$  0.008 & 0.182 $\pm$ 0.001 & 0.004  $\pm$  0.005 \\

0.3 &0.212 $\pm$ 0.002 & -0.006 $\pm$  0.006 & 0.198 $\pm$ 0.001 & -0.025 $\pm$ 0.007 \\

0.4 & 0.210 $\pm$  0.002 & -0.017 $\pm$ 0.008 & 0.213 $\pm$  0.001 & -0.028 $\pm$  0.005 \\

0.5 & 0.211 $\pm$  0.003& -0.015 $\pm$ 0.006 & 0.211 $\pm$ 0.001 & -0.014 $\pm$ 0.004 \\
\bottomrule
\end{tabular}
\end{small}
\end{center}
\end{table}

\end{document}